\documentclass[twoside,11pt]{article}

\usepackage{blindtext}

\usepackage[abbrvbib, preprint]{jmlr2e}

\usepackage{lastpage}
\jmlrheading{}{}{1-\pageref{LastPage}}{XX; Revised XX/XX}{XX/XX}{}{Dorian Baudry and Kazuya Suzuki and Junya Honda}

\usepackage{comment}
\usepackage{macros}
\usepackage[algo2e, inoutnumbered, algoruled,vlined]{algorithm2e} 
\usepackage{algorithm}
\usepackage{subcaption}
\newcommand{\qed}{\hfill $\blacksquare$ \par}

\ShortHeadings{A General Recipe for the Analysis of Randomized Multi-Armed Bandit Algorithms}{Baudry and Suzuki and Honda}
\firstpageno{1}
\begin{document}

\title{A General Recipe for the Analysis of Randomized Multi-Armed Bandit Algorithms}

\author{\name Dorian Baudry \email dorian.baudry@stats.ox.ac.uk \\
       \addr Department of Statistics, University of Oxford\\
       Oxford, United Kingdom
       \AND
       \name Kazuya Suzuki\\ 
       \addr Kyoto University \\ Kyoto, Japan
       \AND
       \name Junya Honda
       \email honda@i.kyoto-u.ac.jp\\
       \addr Kyoto University\\ Kyoto, Japan
}

\editor{XXXX XXXX}

\maketitle

\begin{abstract} In this paper we propose a general methodology to derive regret bounds for randomized multi-armed bandit algorithms.
It consists in checking a set of sufficient conditions on the family of distributions and on the \emph{sampling probability} of each arm to prove a logarithmic regret.
As a direct application we revisit two 
bandit algorithms, Minimum Empirical Divergence (MED) and Thompson Sampling (TS), under various models for the distributions including single parameter exponential families, Gaussian distributions, bounded distributions, or distributions satisfying some conditions on their moments. In particular, we prove that MED is asymptotically optimal for all these models, but also provide a simple regret analysis for \rev{a variant of} some TS algorithms for which the optimality is already known.
We then further illustrate the interest of our approach, by analyzing a new Non-Parametric TS algorithm ($h$-NPTS), adapted to some families of unbounded reward distributions with a bounded \emph{$h$-moment}. This model can for instance capture some non-parametric families of distributions whose variance is upper bounded by a known constant.
\end{abstract}


\section{Introduction}\label{sec::introduction}

A Multi-Armed Bandit (MAB) is a problem in which a learner sequentially picks an action among $K$ alternatives, called arms, and collects a random reward. The rewards collected from an arm $k \in [K]$ are all drawn independently from a distribution $F_k$, of mean $\mu_k$. At each time step $t$ the learner chooses an arm $A_t$, adapting her strategy in order to maximize the expected sum of rewards. For a time horizon $T$ this is equivalent to minimizing the \emph{regret}, 
formally defined as 

\begin{equation}\cR_T(F_1,\dots, F_K) = \bE\left[\sum_{t=1}^{T} (\mu^\star - \mu_{A_t})\right] = \sum_{k=1}^K \Delta_k \bE\left[N_k(T)\right] \label{eq:basicregret} \;, \end{equation}
where $N_k(T)=\sum_{t=1}^T \ind(A_t=k)$ is the total number of selections of arm $k$ up to time $T$, and $\Delta_k$ is the \emph{sub-optimality gap} of arm $k$: $\Delta_k = \mu^\star - \mu_k$, for $\mu^\star=\max_{j\in\K} \mu_j$. Assuming that $F_1,\dots, F_K$ come from the same family of distributions $\cF$, \cite{LaiRobbins85} and \cite{burnetas96LB} proved (respectively for single-parametric and general families) that a uniformly efficient bandit algorithm\footnote{$\forall (F_1, \dots, F_K)\in \cF^K$: $\forall \alpha>0$,  $\bE[N_k(T)]=o(T^\alpha)$ for all $k$ satisfying $\Delta_k>0$.} satisfies the following lower bound for any sub-optimal arm $k$:
\begin{equation}
\liminf_{T\rightarrow \infty} \frac{\bE[N_k(T)]}{\log(T)} \!\geq\! \frac{1}{\kinf^\cF(F_k, \mu^\star)} \;,\;\;\; \kinf^\cF(F_k, \mu^\star)=\inf_{G \in \cF} \left\{\KL(F_k, G)\!:\! \bE_{X\sim G}[X]\!>\! \mu^\star \right\}\;,
\label{eq::asymp_opt_def} 
\end{equation} 
where $\KL(.,.)$ denotes the Kullback-Leibler divergence between two distributions.
We call an algorithm \emph{asymptotically optimal} if it admits a regret upper bound \emph{matching} this lower bound.
Furthermore, if for each $k \in [K]$ it holds that $\bE[N_k(T)] \leq  C_k\log(T)$, \rev{for some problem-dependent constant $C_k$}, we say that the algorithm achieves a \emph{logarithmic regret}.

\paragraph{Families of distributions} The lower bound presented in \eqref{eq::asymp_opt_def} depends on the characteristics of the family of distributions $\cF$, on which assumptions have to be made before designing bandit algorithms. 
For instance, \emph{Single Parameter Exponential Families} (SPEF) are a usual parametric model (see Definition~\ref{def::spef}).
They include several usual distributions such as Bernoulli, Poisson, or Gaussian distributions with known variance.
In some cases a multiparameter exponential family is also considered such as \textit{Gaussian distributions with unknown variances}.

In other cases, non-parametric assumptions may be more suitable. For instance,
one can consider the model where the rewards are supported in a \emph{known bounded range} $[b, B]$, or the model of $\sigma$-\emph{sub-Gaussian} distributions
(see for example Definition 5.2 in \citealp{BanditBook}),
the latter of which essentially assumes that their tails are no heavier than the
Gaussian distribution with variance $\sigma^2$.

When the distributions are \emph{heavy-tailed}, some \emph{moment condition} may be assumed (see e.g. \citealp{bubeck_heavy}): $\bE_{X\sim \nu}[|X|^{1+\epsilon}]\leq B$ for known $\epsilon>0, B >0$. More recently, \cite{agrawal2020optimal} introduced a more general assumption: $\bE_\nu[h(|X|)]\leq B$ for a known convex function $h$ satisfying $x=o(h(|x|))$. We name for convenience such assumption an (uncentered) \emph{$h$-moment condition}. We further consider in this paper its \emph{centered} version: $\bE_\nu[h(|X-\bE_\nu[X]|)]\leq B$, that we simply call a \emph{centered} $h$-moment condition.
Some interesting examples would be
(a variant of) the $\sigma$-sub-Gaussian model and a family of distributions with bounded variance. \rev{In Section~\ref{sec::NPTS} we propose a simple policy based on bootstrapping to tackle these families of distributions. Its analysis is enabled by the theoretical framework that we propose in this work.}

\subsection{Asymptotically optimal bandit algorithms}\label{subsec::literature} 

There is a vast literature on MABs, so the following review of related works is non-exhaustive. We refer for instance to \citet{BanditBook} for a broader survey. 

\rev{In this study, we mainly focus on \emph{instance-dependent} regret bounds. The lower bound established in Equation \eqref{eq::asymp_opt_def} provides the best achievable asymptotic rate for this kind of guarantees. For that reason, in this section we proposed a review of related works primarily focused on the main families of asymptotically optimal policies.}

The most celebrated approach in bandits is certainly \emph{optimism in face of uncertainty} \citep{agrawal95,auer2002finite}. The simple UCB1 policy \citep{auer2002finite} achieves logarithmic regret for bounded-support
distributions, while the more sophisticated KL-UCB principle provides optimal algorithms for SPEF \citep{KL_UCB}, bounded distributions \citep{KL_UCB,agrawal2021regret}, and uncentered $h$-moment conditions when $h(x)=x^{1+\epsilon}$ for $\epsilon>0$ \citep{agrawal2021regret}. To do that, these policies build on \emph{upper confidence bounds}, that are well-calibrated for the family of distributions considered. In the cases where asymptotic optimality is achieved (KL-UCB policies), these confidence bounds involve the $\kinf$ function presented in \eqref{eq::asymp_opt_def}.

\rev{A second celebrated class of policies, \emph{Thompson Sampling} (TS) dates back to the work of \citet{TS_1933}. TS algorithms gained renewed attention in the early 2010s, following several studies that demonstrated their strong theoretical and practical performance. Notably, \citet{chapelle11} established the strong empirical performance of TS, and \citet{TS_Emilie, TS12AG} first proved the asymptotic optimality for the Bernoulli case. Later, \cite{AgrawalG17} proposed a simplified version of the proof, inspiring subsequent instance-dependent analyses of TS and other randomized policies (Theorem 36.2 in \cite{BanditBook}, Theorem 1 in \cite{Giro}).} At its core, TS is a Bayesian algorithm that, at each time step, uses an appropriate conjugate prior/posterior to
sample a \emph{parameter} for each arm (typically, their expectation),
and choose the best arm in this sampled environment. \rev{From an instance-dependent perspective}, several TS algorithms \rev{based on non-informative priors} are optimal: with Jeffreys priors for SPEF \citep{korda13TSexp}, with well-tuned inverse-gamma priors for Gaussian distributions \citep{honda14}, and with Dirichlet prior/posteriors for bounded distributions \citep{RiouHonda20}.

A third family notably includes several optimal algorithms: \emph{Minimum Empirical Divergence} policies. MED \citep{honda11MED} and IMED \citep{Honda15IMED} are respectively the randomized and deterministic versions of this principle. The former is optimal for multinomial distributions \citep{honda11MED}, while the latter is optimal for semi-bounded distributions \citep{Honda15IMED} and SPEF \citep{pesquerel21}. More recently, \citep{jun2022, Qin23} analyzed variants of MED under the name \emph{Maillard Sampling} (MS), \rev{after re-discovering the algorithm from the work of \citet[Figure 1.8 in Section 3.4]{Maillard11}}
respectively for sub-Gaussian and bounded distributions. 
\rev{Prior to these works, MED/MS was also reinvented under the name \emph{SoftElim} by \cite{boutilier20}, which the authors later extended to contextual bandits in \citet{kveton2021metalearningbanditpoliciesgradient}. In these studies, the algorithms are analyzed through the lens of policy gradient methods, with the authors comparing the performance of this approach to other policies, such as Thompson Sampling.}

More recently, several works have focused on developing alternative nonparametric algorithms to these three approaches. \rev{First, \cite{PHE} and \cite{Giro} established a general framework for exploration based on \emph{bootstrapping}.} \rev{The NPTS algorithm \citep{RiouHonda20}, which is related to the Bayesian bootstrap (see \citealp{tiapkin2022}), can also be categorized within this line of work, in addition to being a TS-based policy.
This is also true for the variant of NPTS introduced in Section~\ref{sec::NPTS} of this paper, designed for distributions characterized by an $h$-moment condition.} Finally, it has been shown that some algorithms based on \emph{sub-sampling} \citep{BESA, chan2020multi, SDA, lbsda} are optimal for SPEF or Gaussian distributions without requiring prior knowledge of the distribution family. 

\rev{Among the various types of policies, \emph{randomized} algorithms --the focus of our study-- are of particular practical interest. We refer to Chapter 1 of \citet{russo18survey}, that provides an overview of industrial applications of Thompson Sampling in recent years. 
Randomizing the pulled arm at each iteration
also plays a crucial role in natural extensions of MAB where rewards are accrued concurrently (see Chapter 6.4 of \citealp{russo18survey}). For instance, a recent work by \citet{GAUTRON24} highlights the importance of randomization in a crop-management problem in agriculture, where the algorithm operates over a limited number of periods (years), but collects data from several independent locations during each period. In this case, randomizing actions at each period is essential to ensure empirical performance within the time constraints.}

\subsection{\rev{Other standard performance metrics}} 

While the focus of this work is on providing a framework for proving instance-dependent regret bounds for randomized bandit algorithms, we also present, for completeness, other performance metrics commonly considered in the literature.

In frequentist analyses, it is also standard to consider the \emph{worst-case} expected regret (or \emph{problem-independent} regret) alongside the instance-dependent regret. While the latter captures the first-order asymptotic behavior of the regret, the former also requires clean bounds on the second-order terms in order to obtain a result close to the optimal $\cO(\sqrt{KT})$ bound (where the $\cO$ notation only captures universal constants), for \emph{any} $K$-armed bandit problem with arms belonging to the family $\cF$. Following the notation from Eq.~\eqref{eq:basicregret}, the worst-case regret is simply defined by 
\[\cR_T = \sup_{(F_1,\dots, F_K) \in \cF^K} \cR_T(F_1,\dots, F_K) \;, \]
and thus reflects the worst possible scaling of the regret in terms of all possible problem parameters.
Hence, when an explicit upper bound on the instance-dependent regret is available for any possible instance, it is possible to derive a corresponding bound on the worst-case regret. Notably, it is well established that an upper bound of $\cR_T = \cO(\sqrt{KT\log(T)})$ is achieved by UCB in the sub-Gaussian case \citep{auer2002finite}, as well as by the Beta-Bernoulli TS algorithm \citep{AgrawalG17}. \cite{AgrawalG17} further obtained $\cO(\sqrt{KT\log(K)}$ for TS with Gaussian prior. \cite{jun2022, Qin23} proved that this bound also holds for Maillard Sampling for sub-Gaussian distributions (which is, as discussed above, MED with the Gaussian divergence). Furthermore, \cite{Qin23} proved that the Bernoulli version of MED/MS satisfies $\cR_T=\cO(\sqrt{\mu^\star(1-\mu^\star)KT\log(K)})$. 
In this viewpoint, the results presented in the next sections are not specially tailored for the problem-independent bound, but can still be used to derive such a bound for some cases as we will see in Section~\ref{subsec::pb_ind}. Moreover, to the best of our knowledge, worst-case bounds have only been achieved for sub-Gaussian (including bounded) families of distributions, while our goal is to cover more diverse (e.g. nonparametric/unbounded) families, for which it is already difficult to obtain logarithmic problem-dependent bounds.
For these reasons, in this work the primary focus is the problem-dependent analysis. Still, we discuss in Section~\ref{subsubsec::discussion} how the theoretical framework that we introduce in Section~\ref{sec::results} can be used to derive problem-independent regret bounds.

Another widespread metric to analyze bandit algorithms, particularly Bayesian policies like Thompson Sampling, is the \emph{Bayesian regret}. Unlike the worst-case regret, this metric considers the expected regret \emph{averaged} over instances drawn from a given prior $p_\cF$. Specifically, it is defined as 
\[\cR_T^\text{Bayes} = \bE_{(F_1,\dots, F_K) \sim (p_\cF)^K}\left[\cR_T(F_1,\dots, F_K)\right]\;.\]
In particular, \cite{Russo14} proposed a simple framework from which most later Bayesian analyses are derived. Unlike the frequentist bounds, Bayesian guarantees can account for the value of the prior (see, e.g., \citealp{Lu19}), which is especially useful in applications where the learner encounters multiple problem instances drawn from the same prior \citep{kveton21meta}.

In the remainder of the paper, we present simple regret analyses applied to the frequentist regret.
Specifically, for Thompson Sampling, we focus on instances of algorithms that use non-informative priors, which do not require prior knowledge of the problem.
For unbounded (e.g.~Gaussian) distributions, non-informative priors are often improper, and Bayesian regret is not defined.
Instead, these policies are directly comparable to frequentist ones and are known to achieve strong instance-dependent guarantees \citep{TS_Emilie, honda14, RiouHonda20}.

\subsection{Outline and contributions} After comparing the results obtained in the literature, a striking observation is that diverse policies (KL-UCB, TS, IMED) are all proved to be asymptotically optimal for almost the same families of distributions (SPEF and bounded distributions in particular). This raises two intriguing questions: 

\begin{center}\emph{What are the fundamental properties shared by these families of distributions? Given a family of distribution, are all these algorithms different variants of the same exploration strategy?} \end{center}

After detailing some notation and algorithms in Section~\ref{sec::prelim}, we propose in Section~\ref{sec::results} an answer to the first question, by exhibiting a common property (divided in four sub-properties) shared by all the families of distributions that we introduced. This set of properties, that we formalize in Assumption~\ref{ass::sufficient_prop}, 
allows us to derive a unified regret analysis for randomized algorithms, depending on upper and lower bounds on the \emph{sampling probability} of each arm. We then confirm in Section~\ref{sec::applications} that it easily captures MED and a 
variant of TS, that we call \TS, proving their asymptotic optimality. \rev{In \TS, posterior sampling is used within an algorithmic structure that allows to derive explicit upper and lower bounds on the sampling probability of each arm in the analysis. Furthermore, \TS{} offers the additional benefit of improved computational efficiency compared to TS in some nonparametric settings, where the cost of posterior sampling increases with time.} 
Our result not only covers some families of distributions $\cF$ where the optimality of MED has not been proved, but also leads to a simplified analysis of some known TS policies 
under a slight change of the algorithms \rev{(the \TS{} framework)}. We also prove that some nonparametric families of distributions defined by $h$-moment conditions satisfy the above sufficient property for the optimality, so MED is the first algorithm with the optimal regret bound for the centered case. This setting has a special technical difficulty since the space of the distributions is not compact nor convex unlike the uncentered case.

To demonstrate the strength of our generic framework, we then propose a new Thompson Sampling algorithm, named $h$-NPTS, for the nonparametric models with uncentered or centered $h$-moment conditions in Section~\ref{sec::NPTS}. We show that its sampling probability can be bounded in a form such that the analysis can be captured within the same framework as MED, which completes the picture that MED and \rev{TS$^\star$} can be interpreted as two variants of the same exploration strategy. 
Furthermore, $h$-NPTS is an efficient alternative to MED and the existing optimal algorithm \citep{agrawal2020optimal} in the uncentered case, as it does not require any optimization procedure. \rev{In Section~\ref{sec::experiments}, we present numerical experiments on synthetic problems that confirm our theoretical findings.}

\section{Preliminaries}\label{sec::prelim}

In this section we introduce some notation to describe randomized bandit algorithms. We then detail MED and \TS, the algorithms that we consider more specifically in this work.

\subsection{Notation and terminology} 

At each step $t$, a bandit algorithm $\pi$ chooses an arm $A_t$ based on
past observations $\cH_{t-1}=(A_1,X_1, \dots, A_{t-1}, X_{t-1})$, where we denote by $X_t$ the reward collected at time $t$. We define the \emph{sampling probability} of each arm $k$ at time $t$ by 
\[p_k^\pi(t) \coloneqq \bP_\pi(A_t = k | \cH_{t-1}) \;.\]
This probability depends on the empirical distribution of each arm $k$, that we denote by $F_k(t)$. We also sometimes use this notation for the empirical cumulative distribution function (cdf) with a slight abuse of notation. 
Throughout the paper, we denote by $\mu_k(t)$ the empirical mean of an arm $k$ at time $t$, \rev{formally defined by $\mu_k(t)=\frac{1}{N_k(t)}\sum_{s=1}^T X_s \ind(A_s=k)$}, and by $\mu^\star(t)=\max_{k \in [K]} \mu_k(t)$ the best empirical mean. We also use for convenience the notation $F_{k, n}$ and $\mu_{k,n}$, denoting respectively the empirical distribution and mean corresponding to the $n$ first observations collected from an arm $k$. \rev{We formally define these quantities as follows: given a sequence of observations collected under the policy $\pi$, we define by $(s_{k,i})_{i \in \N}$ the increasing sequence of (random) times where arm $k$ has been pulled. Hence, $s_{k,i}$ denotes the time of the $i$-th pull of arm $k$. Then, if $s_{k,n}$ exists we can define 
\[\mu_{k,n} = \frac{1}{n}\sum_{i=1}^n X_{s_{k,i}}\;, \quad \text{and} \quad F_{k,n}: x \in \R \mapsto \frac{1}{n}\sum_{i=1}^n \ind(X_{s_{k,i}} \leq x) \;. \]} 

\subsection{Presentation of the randomized policies under study}\label{subsec::presentation}

We now introduce in details the two policies that are the focus of this paper: Minimum Empirical Divergence (MED) and a slight variant of Thompson Sampling (TS) that we call \TS{}. \rev{We also present the main motivations for considering a modified version of TS.}

\subsubsection{Minimum Empirical Divergence}\label{subsec::MED} We consider a general version of the MED algorithm, proposed by \cite{honda11MED}. 
Given a function $D_\pi$ and a non-negative sequence $(a_n)_{n \in \N}$, this policy samples an arm $k$ with probability \begin{equation}\label{eq::MED} p_k^\pi(t) \propto \exp\left(-N_k(t)\frac{D_\pi(F_k(t), \mu^\star(t))}{1+a_{N_k(t)}}\right) \;.\end{equation} If for a family of distributions $\cF$ we choose $a_n=0$ for all $n \in \N$ and $D_\pi\coloneqq\kinf^\cF$, we exactly obtain the original MED algorithm proposed by \cite{honda11MED} for multinomial distributions.
The version of this policy \rev{known as Maillard Sampling (MS) \citep{Maillard11, boutilier20, jun2022}, assuming sub-gaussian distributions,} follows the same definition but with $D_\pi(F_k(t), \mu_\star(t)) = \frac{(\mu^\star(t)-\mu_k(t))^2}{2}$.
In the present paper we consider general functions $D_\pi$, but will try to design them as close as possible to $\kinf^\cF$. We keep the name MED for simplicity, but will carefully specify $D_\pi$ in each context. Furthermore, the tuning of $(a_n)$ will be motivated by our analysis, and will be non-zero only for a few specific cases of unbounded distributions. We detail the implementation of the generic MED in Algorithm~\ref{alg::MED} below. \rev{For a finite set $\cS=(s_1, \dots, s_M)$ and a probability vector $p_\cS=(p_1,\dots, p_M)$, we denote by $\text{Mult}(\cS, p_\cS)$ the multinomial distribution that draws an item at random from $\cS$ with probability $p_\cS$.}

\begin{algorithm}[htbp]
	\SetKwInput{KwData}{Input}
	\SetKwComment{Comment}{\color{purple} $\triangleright$\ }{}
	\SetKwComment{Titleblock}{// }{}
	\KwData{Function $D_\pi$, sequence $(a_n)_{n \in \N}$}
\For{$t\leq K$}{
For $k=t$: draw arm $k$, collect reward $X_k$, set $\cX_k=\{X_k\}$ and $N_k=1$ \\ 
Initialize $\wh \mu_k = X_k$, $\wh F_k: x \in \R \mapsto\ind(X_k \leq x)$. \Comment*[r]{\color{purple} \small Initialize emp.$\!$ means and cdf}
}
\For{$t> K$}{
            Set $\wh \mu^\star = \max_{k \in[K]} \wh \mu_k$ \Comment*[r]{\color{purple} \small Use best empirical mean as reference}
		\For{$k \in \K$}{
                Compute $\wh P_k = \exp\left(-N_k \frac{D_\pi(\wh F_k, \wh \mu^\star)}{1+a_{N_k}}\right)$ \Comment*[r]{\color{purple} \small Compute the relative weights}
		}
		Draw $k_t \sim \text{Mult}\left([K], \frac{\wh P_1}{\sum_{j=1}^K \wh P_j}, \dots, \frac{\wh P_K}{\sum_{j=1}^K \wh P_j} \right)$ \Comment*[r]{\color{purple} \small draw an arm with prob. $\!\!\!\!$ $(p_k(t))_{k \n [K]}$}
  \vspace{1mm}
		Collect the reward $X_t$ from arm $k_t$, add $X_t$ to $\cX_{k_t}$, set $N_{k_t}=N_{k_t}+1$. \\ 
  Update $\wh F_{k_t}: x \in \R \mapsto \frac{1}{N_{k_t}}\sum_{r \in \cX_{k_t}}\ind(r\leq x)$ and $\wh \mu_{k_t} = \frac{1}{N_{k_t}}((N_{k_t}-1) \wh \mu_{k_t} + X_t)$.
	}  
	\SetKwInput{KwResult}{Return}
	\caption{\rev{Minimum Empirical Divergence (MED)}}
	\label{alg::MED}
\end{algorithm}



\subsubsection{Thompson Sampling}\label{subsec::presentation_TS} 
\rev{Following Algorithm 4.2 in \citet{russo18survey}, an iteration of a generic Thompson Sampling algorithm is described as follows: first, a model is sampled using posterior sampling. Then, the algorithm selects the best arm based on the sampled model. In the context of $K$-armed bandits, this results in an \emph{index policy}, where the posterior sampling step provides a sampled expectation $\widetilde{\mu}_k(t)$ for each arm $k \in [K]$, and the arm with the highest sampled expectation, $A_t = \arg\max_{k \in [K]} \widetilde{\mu}_k(t)$, is pulled. In this framework, upper and lower bounding the sampling probability of each arm is challenging because it depends on 
the distribution functions of all $(\wt \mu_k(t))_{k \in [K]}$, since $p_k^\text{TS}(t)=\bP(\wt \mu_k(t)\geq \max_{j\neq k}\wt \mu_j(t))$. To address this, we propose an alternative framework, which we refer to as \TS, to simplify the analysis of the arms' sampling probabilities.}

Assuming that the learner is provided a TS sampler that can return a sampled mean $\widetilde \mu_k(t)$ for each arm given 
their history\rev{, which is the same pre-requisite as for standard TS, an iteration of} the proposed variant \TS{} performs the following steps:
\begin{enumerate} 
\item \rev{Use the best empirical average as a reference threshold $\mu^\star(t)=\max_{k \in [K]} \mu_k(t)$.}
	\item \rev{Sample a parameter \emph{only} for arms that are empirically sub-optimal: for $k \in[K]$, request a sampled mean $\wt \mu_k(t)$ only if $\mu_k(t)< \mu^\star(t)$.}
\item Choose an arm uniformly at random from \[\cA_t=\underbrace{\{k \in [K]: \mu_k(t) =\mu^\star(t)\}}_{\text{Empirical Best arm(s)}} \cup \underbrace{\{k \in [K]: \mu_k(t)< \mu^\star(t) \text{ and } \widetilde \mu_k(t)\geq \mu^\star(t) \}}_{\text{\parbox{25em}{\centering Emp. sub-optimal arms with a posterior sample above the reference threshold}}}.\]
\end{enumerate}
\rev{This structural change makes \TS{} use posterior sampling only for arms for which an exploration strategy is necessary, specifically those that appear sub-optimal based on past observations. Intuitively, comparing with the reference mean $\mu^\star(t)$ mimics the long-term behavior of TS: as $t$ becomes large, both the posterior distribution and the empirical mean of the best arm are expected to be closely concentrated around the true mean (due to a sample size that grows linearly with $t$). Thus, comparing a sample $\wt{\mu}_k(t)$ to $\mu^\star(t)$ should become equivalent to comparing it to a posterior sample of the empirical best arm. Consequently, we can expect the asymptotic behavior of TS and \TS{} to be similar. Section~\ref{sec::experiments} confirms this intuition, with numerical experiments demonstrating that the two algorithms have very close performance even for small horizons, in several settings.} 

In Algorithms~\ref{alg::TS} and \ref{alg::TS_variant} we respectively detail the implementation of TS and \TS{}, with similar notation in order to better highlight the differences between the two algorithms. We assume that a generic sampling distribution $\cS$ is provided: $\cS$ returns the posterior distribution corresponding to some prior (its choice is captured by $\cS$), given an empirical distribution of observations and a sample size. For example, assume that $\cF$ is the set of Bernoulli distributions and $\cS$ is the Beta-Bernoulli sampler with uniform prior. Then, $\cS(\wh F_k, N_k)=\text{Beta}(R_k+1, N_k-R_k+1)$ is a Beta distribution, where $R_k$ is the number of ones collected with arm $k$. This quantity is given with the knowledge of $\wh F_k$. 
\begin{algorithm}[t]
	\SetKwInput{KwData}{Input}
	\SetKwComment{Comment}{\color{purple} $\triangleright$\ }{}
	\SetKwComment{Titleblock}{// }{}
	\KwData{Sampling distribution $\cS$, number of arms $K$}
\For{$t\leq K$}{
For $k=t$: draw arm $k$, collect reward $X_k$, set $\cX_k=\{X_k\}$ and $N_k=1$ \\ 
Initialize $\wh \mu_k = X_k$, $\wh F_k: x \in \R \mapsto\ind(X_k \leq x)$. \Comment*[r]{\color{purple} \small Initialize emp.$\!$ means and cdf}
}
\For{$t> K$}{
		\For{$k \in \K$}{
			Sample $\widetilde \mu_k \sim \cS(\wh F_k, N_k)$\Comment*[r]{\color{purple} \small Sampling step}
		}
		Set $k_t=\aargmax_{k \in \K} \widetilde \mu_k$ \Comment*[r]{\color{purple} \small Draw the arm with the best sampled mean}
		Collect the reward $X_t$ from arm $k_t$, add $X_t$ to $\cX_{k_t}$, set $N_{k_t}=N_{k_t}+1$. \\ 
  Update $\wh F_{k_t}: x \in \R \mapsto \frac{1}{N_{k_t}}\sum_{r \in \cX_{k_t}}\ind(r\leq x)$ and $\wh \mu_{k_t} = \frac{1}{N_{k_t}}((N_{k_t}-1) \wh \mu_{k_t} + X_t)$.
	}  
	\SetKwInput{KwResult}{Return}
	\caption{\rev{Thompson Sampling (TS), usual index policy}}
	\label{alg::TS}
\end{algorithm}

\begin{algorithm}[t]
	\SetKwInput{KwData}{Input}
	\SetKwComment{Comment}{\color{purple} $\triangleright$\ }{}
	\SetKwComment{Titleblock}{// }{}
	\KwData{Sampling distribution $\cS$, number of arms $K$}
\For{$t\leq K$}{
For $k=t$: draw arm $k$, collect reward $X_k$, set $\cX_k=\{X_k\}$ and $N_k=1$ \\ 
Initialize $\wh \mu_k = X_k$, $\wh F_k: x \in \R \mapsto\ind(X_k \leq x)$. \Comment*[r]{\color{purple} \small Initialize emp.$\!$ means and cdf}
}
	\For{$t> K$}{
			Set $\cA
   = \{k \in [K]: \; \wh \mu_k = \max_{j \in [K]}\wh \mu_j\}$ \Comment*[r]{\color{purple} \small Candidate arm(s) = emp. $\!\!\!$ best arm(s)}
	\For{$k \notin \cA$}{
			Sample $\widetilde \mu_k \sim \cS(\wh F_k, N_k)$\Comment*[r]{\color{purple} \small Sampling step only for empirically sub-optimal arms}
			If $\widetilde \mu_k \geq \max_{j\in [K]} \wh \mu_{j}$: add $k$ to $\cA$\Comment*[r]{\color{purple} \small Add arm $k$ to $\cA$ if sampled mean $\geq$ emp.~best}}
		Choose $k_t$ uniformly at random in $\cA$; \\Collect the reward $X_t$ from arm $k_t$, add $X_t$ to $\cX_{k_t}$, set $N_{k_t}=N_{k_t}+1$; \\ 
  Update $\wh F_{k_t}: x \in \R \mapsto \frac{1}{N_{k_t}}\sum_{r \in \cX_{k_t}}\ind(r\leq x)$ and $\wh \mu_{k_t} = \frac{1}{N_{k_t}}((N_{k_t}-1) \wh \mu_{k_t} + X_t)$.
	}
	\SetKwInput{KwResult}{Return}
	\caption{\rev{\TS{} (variant of TS considered in this paper)}}
	\label{alg::TS_variant}
\end{algorithm}


\paragraph{\rev{Analysis of the sampling probabilities under \TS{}}} We now prove the theoretical properties of \TS{} that motivated its introduction. We first introduce a formal name for the probability that an expectation under a sampled model exceeds a given threshold, conditioned on past observations: we call $\bP(\wt \mu_k(t)\geq \mu^\star(t)|\cH_{t-1})$ a \emph{Boundary Crossing Probability} (BCP) in the rest of the paper, and we keep this terminology when $\mu^\star(t)=\mu$ is fixed (which will be always clear from context). 
In the following result, we prove the relation between the BCP and the sampling probabilities under \TS{}. 

\begin{lemma}[\rev{Bounding sampling probabilities with the BCP under \TS{}}]\label{lem::bounds_samp_bcp_TS}
    Under \TS{}, for any $t\in [T]$ and $k\in [K]$ it holds that 
    \[ p_k^{\text{\rm \TS}}(t) \leq \bP(k \in \cA_t|\cH_{t-1}), \]
and that 
\[ p_k^{\text{\rm \TS}}(t) \geq  \frac{1}{2}\max_{k'\neq k}p_{k'}^\text{\rm \TS}(t) \times \bP(k \in \cA_t|\cH_{t-1}) \;.  \]
In addition, $\bP(k \in \cA_t|\cH_{t-1})=\bP(\wt \mu_k(t)\geq \mu^\star(t)|\cH_{t-1}) \coloneqq$ \emph{[BCP]} if $\mu_k(t) < \mu^\star(t)$ , and is equal to $1$ otherwise.
\end{lemma}

\begin{proof}
We first remark that the sampling probability of an arm $k\in[K]$ can be expressed as $p_k^{\text{\TS}} = \bP(k \in \cA_t|\cH_{t-1}) \times \bP(A_t=k|k \in \cA_t, \cH_{t-1})$, by translating the two steps of the policy, where we recall that $\cA_t$ is the candidate set used at time step $t$. The upper bound is direct from that expression, and so we now consider the lower bound. 

We remark that $|\cA_t|=\sum_{j=1}^K \ind(j \in \cA_t)$, and that conditioning on $\{k \in \cA_t\}$ fixes the $k$-th term of the sum to $1$, while the other terms are independent of this event. Thus, omitting the conditioning on $\cH_{t-1}$ in the notation for clarity, the probability to choose $A_t=k$ becomes 
\[\forall k \in [K], \; \bP(A_t=k|k \in \cA_t) = \bE\left[\frac{1}{|\cA_t|}\Bigg| k \in \cA_t\right] = \bE\left[\frac{1}{1+\sum_{j\in [K], j \neq k} \ind(j \in \cA_t)}\right]\;.  \]
Then, for any $k'\neq k$ we introduce the notation $S_{k,k'}(t)=\sum_{j\in [K], j \neq (k, k')} \ind(j \in \cA_t)$, and use this result to further obtain that 
\begin{align*}
 \bP(A_t=k|k \in \cA_t) &= \bE\left[\frac{1+\sum_{j\in [K], j \neq k'} \ind(j \in \cA_t)}{1+\sum_{j\in [K], j \neq k} \ind(j \in \cA_t)} \times \frac{1}{1+\sum_{j\in [K], j \neq k'} \ind(j \in \cA_t)}\right]\\
 & =\bE\left[\frac{\ind(k\in\cA_t) + 1 +S_{k,k'}(t)}{\ind(k'\in\cA_t) + 1 +S_{k,k'}(t)} \times \frac{1}{1+\sum_{j\in [K], j \neq k'} \ind(j \in \cA_t)}\right] \\
 & \geq \frac{1}{2} \times \bE\left[\frac{1}{1+\sum_{j\in [K], j \neq k'} \ind(j \in \cA_t)}\right] \\
& =  \frac{1}{2} \times \bP(A_t=k'|k' \in \cA_t)\;, 
\end{align*}
where in the second line we used that the left-hand term inside the expectation admits a lower bound by $1/2$, that is matched by taking $S_{k,k'}(t)= 0$, $\ind(k\in \cA_t)=0$ and $\ind(k'\in \cA_t)=1$. We can then use this result to conclude the proof, by first obtaining that $p_k^{\text{TS}}(t) \geq \bP(k \in \cA_t|\cH_{t-1}) \times \frac{1}{2} \frac{p_{k'}^\text{\TS}(t)}{\bP(k'\in \cA_t|\cH_{t-1})}$ and using that $\bP(k'\in \cA_t|\cH_{t-1}) \in [0,1]$. 
\end{proof}

\rev{Lemma~\ref{lem::bounds_samp_bcp_TS} implies that upper and lower bounding the sampling probabilities under \TS{} can be reduced to upper and lower bounding the BCP.} Bounds on the BCP are available in the literature for most existing TS policies, and so the analysis presented in this paper will allow us to revisit their proofs and simplify them under this slight algorithmic change. For example, our analysis encompasses the following algorithms: 
\begin{itemize}
	\item \TS{} with a conjugate prior and posterior for SPEF \citep{korda13TSexp, jin22expTS} and under the assumption that the means belong to a finite range $[\mu_0^-, \mu_0^+]$.
	\item Gaussian \TS{} with inverse-gamma priors \citep{honda14} for shape parameter satisfying $\alpha< 0$. 
	\item Non-Parametric \TS{} \citep{RiouHonda20} for bounded distributions with a known range.
\end{itemize}
For completeness, we detail their implementation under the \TS{} framework 
in Appendix~\ref{subsec::old_ts} (Algs.~\ref{alg::TS_SPEF}--\ref{alg::NPTS}). 
The fact that only an upper and lower bound on the BCP is needed to obtain a range on the sampling probability of \TS{} will be particularly helpful to provide a joint analysis of MED and \TS{} in the next section, and by extension to analyze the novel $h$-NPTS algorithm introduced in Section~\ref{sec::NPTS}.

\rev{\paragraph{Potential computational gain with \TS{}} To conclude this presentation, we can also examine the comparison between TS and \TS{} from a computational perspective. As noted in Remark 3 of \citet{RiouHonda20} and in \citet{baudry21DS}, in certain settings, the computation time for posterior sampling increases with the number of observations. They note that this cost might be significantly reduced by a slight change in the algorithm --using the empirical mean instead of a sampled mean for some well-chosen arms-- without altering the theoretical performance, which is the direction that we followed with \TS{}. This is particularly true in non-parametric settings where bootstrapping is used, such as in NPTS \citep{RiouHonda20} for bounded distributions or the $h$-NPTS policy introduced in Section~\ref{sec::NPTS}. 
Specifically, when $T$ is sufficiently large (relative to the gaps) and assuming logarithmic regret, we expect the best arm to be sampled $\Omega(T)$ 
times, while all other arms are sampled only $\cO(\log(T))$ times. Under these conditions, the \TS{} framework can significantly reduce computation time compared to the standard index policy implementation of TS. Section~\ref{sec::experiments} reports on the computational gains observed in practice through experiments.
}
 


\section{\rev{Theoretical results: generic analysis of randomized policies}}\label{sec::results}

In the previous section we introduced some examples of algorithms entering into the scope of our analysis of randomized policies. We now introduce the main theoretical results of this paper, \rev{introducing the general objective of our approach and its potential applications.} 

\paragraph{Summary} In this section, we propose a generic recipe to derive a logarithmic regret upper bound for a randomized policy $\pi$. It consists in
\begin{enumerate}
    \item upper and lower bounding the arms' \emph{sampling probabilities} under the policy $\pi$ as a function of $D_\pi$, for some divergence $D_\pi$, and
    \item proving a set of generic properties of $D_\pi$ for the family of distributions $\cF$.
\end{enumerate}
After detailing the aforementioned conditions, we prove that they are sufficient to obtain logarithmic regret bounds for the policy $\pi$ on $\cF$. More precisely, our main result consists in showing that under these conditions, the following bound can be obtained: 
\begin{equation}\label{eq::target_bound}
\forall c>0, \;\forall k: \Delta_k>0, \; \bE[N_k(T)] \leq (1+c) \frac{\log(T)}{D_\pi(F_k, \mu^\star)} + o_c(\log(T)), \end{equation}
where we use the notation $o_c(\log(T))$ to represent poly-logarithmic terms in $T$ with a scaling smaller than $\log(T)$, along with problem-dependent constants that depend on $c$. These constants are finite for a fixed $c$, but can become large as $c$ goes to $0$ (infinite for $c=0$). Proving a bound of the form of Eq.~\eqref{eq::target_bound} for some function $D_\pi$ is sufficient to guarantee an \emph{instance-dependent} logarithmic regret for the policy $\pi$. Furthermore, if the bound is proved for $D_\pi(F_k,\mu^\star)=\kinf^\cF(F_k, \mu^\star)$, it holds that $\pi$ is an \emph{asymptotically optimal} policy.

\paragraph{Outline} In the next Section~\ref{subsec::generic_cond}, we start by formalizing the two assumptions that we propose to achieve this objective. In Section~\ref{subsec::main_result} we formally state the main theoretical result of this work, and discuss its proof. Then, in Section~\ref{subsubsec::discussion} we compare our main results, Theorem~\ref{th::main_result} and Lemma~\ref{th::first_bound}, with other generic theorems from the literature. We also discuss how the analysis that we introduce can lead to problem-independent guarantees. We leave the applications of these results to specific policies and families of distributions for Section~\ref{sec::applications}.

\subsection{Proposed conditions -- what to verify in order to use the recipe}\label{subsec::generic_cond}

Let us consider a randomized policy $\pi$ and any family of distributions $\cF$. We now detail the conditions that we propose to check, as part of the generic procedure that we introduced to derive the regret guarantees of $\pi$ on $\cF$. 

\subsubsection{Upper and lower bound on the arms' sampling probabilities}\label{subsubsec::sampling}

We start with a condition that strictly depends on the randomized policy $\pi$, which consists in upper and lower bounding the arms' sampling probabilities as a function of some divergence $D=D_\pi$. This divergence takes as arguments an empirical cumulative distribution function, and a scalar. In the rest of the paper we denote by $\cP$ the space of all possible cumulative distribution functions on $\R$: for any $F\in \cP$, it holds that $F$ is non-decreasing, $\lim_{x\to -\infty}F(x)=0$ and $\lim_{x\to +\infty} F(x)=1$. Additionally, the $\cO$ and $o$ notation in the following statement of the assumption hides universal constants only.

\begin{assumption}[Upper and lower bound on $p_k^\pi(t)$] \label{ass::sampling_prob}
$\pi$ is a randomized bandit algorithm, and there exists a function $D: \cP \times \R \mapsto \R^+$, \textbf{positive} sequences $(c_n)_{n \in \N}, (C_n)_{n \in \N}$ and a \textbf{non-negative} and \textbf{non-increasing} sequence $(a_n)_{n \in \N}$ 
such that, for any $k$ in $[K]$ and $t\in[T]$, $p_k^\pi(t)$ admits the upper bound
\begin{equation}\label{eq::sampling_prob_ub} 
\forall t \in [T], \; p_k^\pi(t) \leq C_{N_k(t)} \exp\left(-N_k(t) \frac{D(F_k(t),\mu^\star(t))}{1+a_{N_k(t)}}\right),
\end{equation}
where the sequence $(C_n)_{n\in \N}$ satisfies $\log(C_n)=\cO(n^\alpha)$ for some $\alpha<1$. \rev{Furthermore, $p_k^\pi(t)$ is also lower bounded as follows,
\begin{equation}\label{eq::sampling_prob_lb} 
\forall k'\in [K], k'\neq k: \; p_k^\pi(t) \geq p_{k'}^\pi(t) \times c_{N_k(t)}^{-1} \exp\left(-N_k(t) \frac{D(F_k(t),\mu^\star(t))}{1+a_{N_k(t)}}\right) \;,
\end{equation}
with $c_n = o(e^{n \beta})$ for any $\beta>0$.} Additionally, $(a_n^{-1})_{n \in \N}$ is at most polynomial in $n$.
\end{assumption}
The general interpretation of this assumption is that, in our framework, the sampling probability of each arm must be driven by a term that is exponentially decreasing in the product of two parameters: the number of observations for that arm, and an appropriate divergence between its empirical distribution and the best empirical mean given all observations collected at the current time step. This last term, $\exp\left(-N_k(t) \frac{D(F_k(t),\mu^\star(t))}{1+a_{N_k(t)}}\right)$, appears in both the upper and lower bound on $p_k(t)$ and aligns with the general definition of MED policies presented in Section~\ref{sec::prelim} (Eq.~\ref{eq::MED}). Consequently, our analysis encompasses policies that are asymptotically equivalent to a MED policy, defined by some divergence. While this result looks direct after this discussion, we now formally prove that MED satisfies Assumption~\ref{ass::sampling_prob}, assuming a first property of divergences that we will formalize in the next section. 
\begin{lemma}\label{lem::ass_prob_MED}
Let $D$ be a divergence satisfying $D(F_k(t), \mu_k(t))=0$ for all $k\in [K]$. Specifying $D$ and a non-negative sequence $(a_n)_{n\in \N}$ as its parameters, MED satisfies Assumption~\ref{ass::sampling_prob}, with $c_n=C_n=1$ for all $n\in \N$. 
\end{lemma}
\begin{proof}
At each round, at least one arm satisfies $D(F_k(t), \mu^\star(t))=0$, so it holds that 
\[\sum_{k=1}^K \exp\left(-N_k(t)\frac{D(F_k(t), \mu^\star(t))}{1+a_{N_k(t)}}\right) \geq 1\;, \] which proves the upper bound on $p_k^{\text{MED}}(t)$. For the lower bound, we can write that
\begin{align*}
       \forall k'\neq k,\; p_k^\text{MED}(t) &= p_{k'}^\text{MED}(t) \times e^{N_{k'}(t) \frac{D\left(F_{k'}(t), \mu^\star(t)\right)}{1+a_{N_k(t)}}} \times e^{-N_k(t)\frac{D(F_k(t), \mu^\star(t))}{1+a_{N_k(t)}}} \\
        &\geq p_{k'}^\text{MED}(t) \times \exp\left(-N_k(t)\frac{D(F_k(t), \mu^\star(t))}{1+a_{N_k(t)}}\right)  \;,
    \end{align*}
since $D$ is non-negative, so we obtain the desired result with constants $c_n=1$.
\end{proof}

We now consider Thompson Sampling, under the \TS{} framework introduced in Section~\ref{subsec::presentation_TS}. As a direct application of Lemma~\ref{lem::bounds_samp_bcp_TS}, we can prove the assumption under additional conditions on its \emph{Boundary Crossing Probabilities}. More precisely, we obtain the following corollary.

\begin{corollary}[of Lemma~\ref{lem::bounds_samp_bcp_TS}\label{lem::ass_prob_TS}, Assumption~\ref{ass::sampling_prob} in the \TS{} framework]
\TS{} satisfies Assumption~\ref{ass::sampling_prob} with sequences $\left(\frac{c_n}{2}\right)_{n\in\N}$, $(C_n)_{n\in\N}$, $(a_n)_{n\in\N}$, and a divergence $D$ if the Boundary Crossing Probability satisfies 
\begin{equation}\label{eq::sampling_prob} 
\frac{\bP\left(\wt \mu_k(t)\geq \mu^\star(t)|\cH_{t-1}\right)}{q_k^\text{D}(t)}  \in [c_{N_k(t)}^{-1}, C_{N_k(t)}], \text{ with }
 q_k^{\text{D}}(t) = e^{-N_k(t) \frac{D(F_k(t),\mu^\star(t))}{1+a_{N_k(t)}}}, 
\end{equation}
and all sequences satisfy the requirements of Assumption~\ref{ass::sampling_prob}.
\end{corollary}
The result is straightforward from Lemma~\ref{lem::bounds_samp_bcp_TS}, since its formulation is already very close to the conditions in Assumption~\ref{ass::sampling_prob}.

\paragraph{Comments} The requirements we impose on the sequences $(c_n)_{n\in \N}$, $(C_n)_{n \in \N}$ and $(a_n)_{n \in \N}$ offers some flexibility to this definition, though in most cases considered in this paper, these sequences are at most polynomial. Additionally, it is important to note that the theoretical result presented in Section~\ref{subsec::main_result} is actually proven under the more general Assumption~\ref{ass::relaxed_sp} (presented in Appendix) on the sampling probabilities, which follows the same intuition but has a more complicated formulation. Indeed, Assumption~\ref{ass::relaxed_sp} leverages the following observations from the analysis: first, the upper bound needs to hold only when $F_k(t)$ is already in some neighborhood of $F_k$. Then, regarding the lower bound we observe that the problem-dependent analysis is still valid if (1) it holds only after a large enough sample size, and (2) the exponent of the bound is not tight but an arbitrary precision can be attained. This relaxation is useful to analyze \TS{} policies, when bounds on the BCP are difficult to obtain (e.g. for the policy presented in Section~\ref{sec::NPTS}). For the sake of clarity, we chose to present Assumption~\ref{ass::sampling_prob} in the main body of the paper.

Lastly, we can comment on the form of the lower bound in Eq.~\eqref{eq::sampling_prob_lb} and why it involves the sampling probabilities of multiple arms. The motivation is simply technical, because the analysis involves a ratio of sampling probabilities. Compared to a direct lower bound on $p_k^\pi(t)$, this formulation allows to obtain tighter values of $c_n$ by avoiding normalization effects (essentially, some terms would be multiplied by $K$ in the analysis). This tighter characterization of $c_n$ is interesting when analyzing second-order terms of the regret bounds, that we discuss in Section~\ref{subsubsec::discussion}.

\subsubsection{Generic properties required for a divergence $D$ on $\cF$}\label{subsubsec::properties}

We now establish a set of generic properties, that only depend on the definition of a divergence $D$ applied to distributions belonging to the family $\cF$. These properties are motivated by existing bandit analyses in the literature (see Section~\ref{subsec::literature}), 
after abstracting some salient features that seems to be shared by the families of distributions for which we know asymptotically optimal bandit algorithms.
We group these generic features in Assumption~\ref{ass::sufficient_prop}, that we now introduce. 

\begin{assumption}[Properties of a divergence $D$ on the family $\cF$]\label{ass::sufficient_prop}
Let $F\in \cF$ be any distribution of mean $\mu_F$. Let us denote by $F_n$ the empirical distribution\footnote{If $\cF$ is parametric we assume that $D$ first maps $F_n$ to the distribution of $\cF$ corresponding to the Maximum Likelihood Estimator for its parameters.
} of $n$ i.i.d. samples drawn from $F$, and by $\mu_n$ their average. The divergence $D: \cP \times \R \mapsto \R^+$ admits the following properties:
\begin{itemize}
\item \textbf{Continuity and monotonicity (A1):} $D$ is continuous in its second argument, the mapping $x \mapsto D(F, x)$ is non-decreasing, and $D(F, \mu)=0$ for any $\mu\leq \mu_F$.
\item \textbf{Lower deviation (A2):} for any $\mu > \mu_F$, and $\delta>0$, \rev{there exists a constant $A_\delta<+\infty$ such that} 
\[\sum_{n=1}^{+\infty} \bP\left(D(F_{n}, \mu) \leq D(F, \mu)-\delta\right) \leq A_\delta 
\;.\]
\item \textbf{Separability of mean levels (A3):} For any $\epsilon>0$, there exists $\delta_{\epsilon, \mu_F}>0$ such that \rev{for any $G \in \cF$ satisfying $\mu_G \leq \mu_F-\epsilon$} it holds that
		\[D(G, \mu_F)-D(G, \mu_F-\epsilon)\geq \delta_{\epsilon, \mu_F}\;.\] 
\item \textbf{Upper deviation (A4):} \rev{There exists a sequence $(\alpha_n)_{n \in \N}$, polynomial in $n$, such that for any $x\geq0 $ it holds that}
\[\bP(D(F_{n}, \mu_F)> x) \leq \alpha_{n}\times e^{-nx}. \] 
\end{itemize}
\end{assumption}

These properties might look a bit heavy, but their level of generality is useful to tackle all the settings covered in this paper, \rev{in particular the non-parametric settings}. In the following, we use the example of bounded distributions to demonstrate that, for certain families and divergences, the assumption can be straightforward to verify, while also motivating the general form of its statement. We then provide further intuition on the interpretation of each condition (A1)--(A4).

\rev{\paragraph{Example: distributions supported on $[0,1]$} In many usual cases the assumption} is relatively easy to verify. For example, let us consider bounded distributions with range $[0,1]$ and choose $D(F_n, \mu)=2(\mu-\mu_n)^2\ind(\mu\geq \mu_n)$, as would be done in Maillard Sampling \citep{Maillard11, jun2022}. Then, (A1) is direct from the definition of $D$, and a simple computation proves (A3) with $\delta_{\epsilon, \mu_F}=2\epsilon^2$. The remaining properties can then be obtained with minimal effort using Hoeffding's inequality. 
We prove (A4) by writing that
\[\bP(D(F_n, \mu_F)>x) = \bP(\mu_n < \mu-\sqrt{x/2}) \leq e^{-2 n(\sqrt{x/2})^2} = e^{-nx} \;, \]
and for (A2) it is easy to verify that, for any $\mu\geq \mu_F$ and $\delta \in (0, D(F,\mu)]$,
\begin{align*}
    \bP(2(\mu-\mu_n)^2\ind(\mu\geq \mu_n)\leq 2(\mu-\mu_F)^2-\delta) \leq \bP\left(\mu_n \geq \mu_F + \sqrt{\frac{\delta}{2}}\right) \leq e^{-n\delta} \;.
\end{align*}
Then, summing for all values of $n$ provides a constant $A_\delta = \frac{1}{1-e^{-\delta}} \sim \frac{1}{\delta}$, which proves (A2). 

\rev{Moreover, we motivate the generic formulation of the assumptions by considering the same bounded setting, but for the divergence $D=\kinf^\cF$ corresponding to the lower bound (Eq.~\eqref{eq::asymp_opt_def}) instead of the divergence discussed above. 
In that case, the level of generality allowed in each condition is necessary. First, proving (A2) requires to use some continuity arguments (see Appendix~\ref{app::bounded_P15}) and thus it is hard to make $A_\delta$ explicit in $\delta$ (although it indeed depends solely on $\delta$). Then, Lemma 13 in \citet{honda11MED} provides $\delta_{\epsilon, \mu_F}=\frac{\epsilon^2}{2(\mu_F-\epsilon)(B-\mu_F)}$ in (A3), which depends both on $\mu_F$ and $\epsilon$. Finally, using Lemma 6 in \citet{KL_UCB} we get $\alpha_n=e(n+2)$ when proving (A4), which thus depends on $n$. 
}
\paragraph{Additional intuitions} In some sense, we can say that Assumption~\ref{ass::sufficient_prop} generalizes the properties that are usually derived from concentration inequalities on the empirical means, for simple choices of $D$ as in our example. This generalization is particularly useful for the analysis of non-parametric models, as the last case presented above, where $\kinf^\cF$ cannot be expressed as a function of empirical means alone.

Intuitively, the functions $D(F,\mu)$ of interest measure how ``far'' a distribution $F$ is to the class of distributions from $\cF$ with expectation larger than $\mu$.
In the rest of the paper we will typically consider $D$ close to $\kinf^\cF$ for the families considered, since this choice can lead to instance-dependent optimality, but other choices are possible (e.g. multiples of the squared gap or Bernoulli KL divergence for bounded distributions). Regarding the properties that we expect from $D$, we can first remark that (A2) is slightly weaker than assuming that $D$ is continuous in the first argument w.r.t. some metric and that each $F_{k, n}$ converges ``fast enough" to some neighborhood of $F_k$. It ensures that the expected number of sample sizes for which the empirical distribution of a sub-optimal arm is ``too close to $\mu^\star$'' is finite. On the contrary, (A4) 
guarantees that an empirical distribution that would make the optimal arm appear sub-optimal (its expectation is under-estimated) has a low probability of being very ``far'' from the true distribution.
Finally, (A3) ensures that divergence $D$ allows to properly separate the different mean levels, so that it can be used to distinguish an arm of mean $\mu$ from an arm of mean $\mu-\epsilon$ for any $\epsilon>0$. 

\begin{remark}[\rev{Mean estimates}] While the Assumptions~\ref{ass::relaxed_sp} (relaxation of Assumption~\ref{ass::sampling_prob}) and~\ref{ass::sufficient_prop} are sufficient to avoid additional assumptions on the concentration of the mean estimates $(\mu_{k,n})_{k\in [K], n \in [N]}$,
it is reasonable to assume that better estimates might lead to better practical performance. For that reason, we discuss in Appendix~\ref{app::A5_heavy} the properties of some robust estimates, that can be used in the case that $\cF$ contains heavy-tailed distributions.
\end{remark}

\subsection{Regret bound under the generic assumptions}\label{subsec::main_result}

Now that we proposed precise generic assumptions on the randomized algorithm and the family of distribution, we can present the main result of this work, which is a formalization of the objective that we set with Equation~\eqref{eq::target_bound}. 

\begin{theorem}\label{th::main_result}
Let $\cF$ be a family of distributions and $(F_1,\dots, F_K) \in \cF^K$ be a MAB. 
Let $\pi$ be a randomized policy satisfying Assumption~\ref{ass::sampling_prob} or Assumption~\ref{ass::relaxed_sp} (upper and lower bounds on the sampling probabilities), for a function $D_\pi$ satisfying Assumption~\ref{ass::sufficient_prop} on $\cF$ (generic properties of a divergence), and for sequences $(a_n, c_n, C_n)_{n\in \N}$. 

Then, for any $c > \displaystyle \lim_{n\to \infty} a_n$, if $(a_n)_{n\in \N}$ is \textbf{positive} 
it holds that 
	\[\forall k: \Delta_k>0, \quad \bE[N_k(T)] \leq (1+c) \frac{\log(T)}{D_\pi(F_k, \mu^\star)} + o_c(\log(T)) \;.\]
Furthermore, the bound holds for $a_n=0$ (for any value of $n$), if $D_\pi$ is upper bounded as follows: for any $\mu \leq \mu^\star$ there exists a constant $D_\mu^+\in \R$ such that for any $F\in \cF$ it holds that $D_\pi(F, \mu) \leq D_\mu^+$. 
\end{theorem}
Theorem~\ref{th::main_result} confirms the general methodology to analyze randomized policies, presented at the beginning of this section: proving that the policy $\pi$ satisfies Assumptions~\ref{ass::sampling_prob} (or~\ref{ass::relaxed_sp}, in the appendix), and \ref{ass::sufficient_prop} on the problem considered is sufficient to deduce a logarithmic regret bound. Conversely, this result also provides a general methodology to derive new theoretically sound bandit algorithms for a given family $\cF$, by first finding a suitable function $D_\pi$ that makes Assumption~\ref{ass::sufficient_prop} hold with $D\coloneqq D_\pi$, and then designing a randomized algorithm $\pi$ satisfying Assumption~\ref{ass::sampling_prob} with this same divergence, the easiest way being to implement the corresponding MED policy. 

\rev{Details of the proof of Theorem~\ref{th::main_result} can be found in Appendix~\ref{sec::generic_proof}, where we first prove an intermediate regret bound in Appendix~\ref{app::proof_th1}, and then complete the proof in Appendix~\ref{app::proof_th2}. In the next section, we propose a proof sketch in the context where Assumption~\ref{ass::sampling_prob} holds, with a simplification of the intermediate result in that case. We then compare this result with other generic regret bounds that can be found in the literature in Section~\ref{subsubsec::discussion}.} 

\subsubsection{Proof sketch of Theorem~\ref{th::main_result}} 

In a first part, we only use Assumptions~\ref{ass::sampling_prob} and property (A1) of Assumption~\ref{ass::sufficient_prop} to exhibit the first-order logarithmic term, sub-logarithmic terms, and a last term independent of $T$ that will require the additional assumptions to be further analyzed. 

\begin{lemma}\label{th::first_bound}
Let $\pi$ be a policy satisfying Assumption~\ref{ass::sampling_prob} for a function $D_\pi$ satisfying (A1) of Assumption~\ref{ass::sufficient_prop}. Then, for any $c > \displaystyle \lim_{n\to \infty} a_n$ and for any $\epsilon>0$, it holds for any sub-optimal arm $k$ that 
\begin{equation}\label{eq::th_2}\bE[N_k(T)] \leq  (1+c) \frac{\log(T)}{D_\pi(F_k, \mu_1)} + B(c, \epsilon) + o(\log(T))\;,\end{equation}
where the terms captured in $o(\log(T))$ can be upper bounded independently of $c$, and with \begin{align*}
	B(c, \epsilon) = &\underbrace{\sum_{n=1}^{+\infty} \bP\left(D_\pi(F_{k,n}, \mu_1-\epsilon)\frac{1+c}{1+a_n} \leq D_\pi(F_k, \mu_1)\right)}_{b_1} \\ 
 + &\underbrace{\sum_{n=1}^{+\infty} c_n \bE\left[\ind(\mu_{1,n}\leq \mu_1-\epsilon) e^{n \frac{D_\pi(F_{1,n}, \mu_1-\epsilon)}{1+a_n}}\right]}_{b_2}.
\end{align*}
\end{lemma}
In Appendix~\ref{app::proof_th1}, we prove a more general version of Lemma~\ref{th::first_bound} under Assumption~\ref{ass::relaxed_sp}, from which Lemma~\ref{th::first_bound} can be deduced. While the general proof scheme is inspired by the analysis of Maillard Sampling by \citet{jun2022} and Thompson Sampling by \citet{AgrawalG17}, our main effort consists in adapting the arguments to accommodate the broader assumptions that we consider. This adaptation enables us to obtain an explicit bound on the additional $o(\log(T))$ term in \eqref{eq::th_2}.
In fact, its scaling in $T$ only depends on $(C_n)_{n\in \N}$. We refer to Remark~\ref{rem::scaling} for further discussion, where we notably obtain the following simplification for the case where $(C_n)_{n \in \N}$ is upper bounded by a constant, 
\begin{equation}\label{eq::bound_constant_Cn}
\max_{n\in \N}C_n \leq C < +\infty \quad \Longrightarrow \quad \bE[N_k(T)]\leq (1+c)\frac{\log(T\cdot D_\pi(F_k, \mu_1)) + C}{D_\pi(F_k, \mu_1)} + B(c,\epsilon)\;. 
\end{equation}

\rev{\paragraph{Interpretation of $B(c,\epsilon)$} This term encapsulates the elements in the upper bound that require to carefully analyze the properties of $D_\pi$ on the family of distributions $\cF$, and hence motivated the design of properties (A2)--(A4) in Assumption~\ref{ass::sufficient_prop}. Intuitively, $B(c,\epsilon)$ captures the regret caused by \emph{low-probability events}, as typically considered in the literature. Indeed, the first term (denoted by $b_1$) essentially corresponds to regret caused by \emph{over-performance} of a sub-optimal arm, which is the counterpart of events of the form $\left\{\mu_{k,n} \geq \mu_k + \alpha \Delta_k\right\}$ (for some constant $\alpha>0$) in some other analyses. On the other hand, the second term is an upper bound of the expected regret caused by potential \emph{under-performance} of the best arm. Its derivation critically relies on the lower bound in Assumption~\ref{ass::sampling_prob}, which permits to upper bound ratios of sampling probabilities. We will discuss in the next section how these two terms relate to other generic analyses proposed in the literature, and we now discuss how Assumptions~\ref{ass::sufficient_prop} (A2--A4) allows us to further upper bound $B(c, \epsilon)$ by instance-dependent constants, and thus complete the proof of Theorem~\ref{th::main_result}.}

\paragraph{\rev{Further upper bounding $B(c,\epsilon)$}} We start with the term $b_1$, and demonstrate that (A1) and (A2) in Assumption~\ref{ass::sufficient_prop} 
allows to upper bound it by a constant. For simplicity, we discuss the case when $a_n=0$, leaving the other cases for the appendix. Then, the first step consists in, for all $n\in \N$, writing the probability ($n$-th term of $b_1$) as 
\[\bP\left( D_\pi(F_{k,n}, \mu_1-\epsilon) \leq D_\pi(F_{k}, \mu_1-\epsilon) - \left(D_\pi(F_{k}, \mu_1-\epsilon)-\frac{D_\pi(F_k, \mu_1)}{1+c} \right) \right) \;. \]
It is then natural to try to use (A2) with $\mu=\mu_1-\epsilon$ and for some parameter $\delta$ satisfying \begin{align}
\delta \leq \left( D_\pi(F_{k}, \mu_1-\epsilon)-\frac{D_\pi(F_k, \mu_1)}{1+c}\right)_+.
\label{req_eps_delta}
\end{align}
If we set $\epsilon=0$ then we have $D_\pi(F_{k}, \mu_1-\epsilon)-\frac{D_\pi(F_k, \mu_1)}{1+c}=\frac{cD_\pi(F_k, \mu_1)}{1+c}>0$.
Therefore, thanks to the continuity of $D_\pi$ in the second argument given in (A1), for any $\delta\in\left(0, \frac{cD_\pi(F_k, \mu_1)}{1+c}\right)$ there exists
$\epsilon>0$ (depending on $\delta$ and $c$) satisfying \eqref{req_eps_delta}.
For the rest of the analysis we fix $(\epsilon, \delta)$ satisfying the above relation, guaranteeing both that $\epsilon>0$ and $\delta>0$.

We now sketch how assumptions (A3) and (A4) lead to an upper bound of $b_2$. We first exploit the event $\{\mu_{1,n}\leq \mu_1-\epsilon\}$ to upper bound $D_\pi(F_{1,n}, \mu_1-\epsilon)$ by $D_\pi(F_{1,n}, \mu_1)-\delta_{\epsilon, \mu_1}$ in the expectation. We can thus obtain that 
\begin{align*}
    b_2 & \coloneqq \sum_{n=1}^{+\infty} \bE\left[\ind(\mu_{1,n}\leq \mu_1-\epsilon) e^{n \frac{D_\pi(F_{1,n}, \mu_1-\epsilon)}{1+a_n}}\right] \\
    & \leq \sum_{n=1}^{+\infty}  c_n e^{- n \frac{\delta_{\epsilon, \mu_1}}{1+a_n}} \times \bE\left[\ind(\mu_{1,n}\leq \mu_1-\epsilon) e^{n \frac{D_\pi(F_{1,n}, \mu_1)}{1+a_n}}\right] \\
    & \leq \sum_{n=1}^{+\infty}  c_n \times  e^{- n \frac{\delta_{\epsilon, \mu_1}}{1+a_n}}\bE\left[e^{n \frac{D_\pi(F_{1,n}, \mu_1)}{1+a_n}}\right] \\
    & = \sum_{n=1}^{+\infty}  c_n \times e^{- n \frac{\delta_{\epsilon, \mu_1}}{1+a_n}} + \sum_{n=1}^{+\infty} c_n e^{- n \frac{\delta_{\epsilon, \mu_1}}{1+a_n}} \times \bE\left[e^{n \frac{D_\pi(F_{1,n}, \mu_1)}{1+a_n}}-1\right] \;.
\end{align*}
We then further upper bound $\bE\left[e^{n \frac{D_\pi(F_{1,n}, \mu_1)}{1+a_n}}-1\right]$ for a fixed value of $n$. Using that this is the expectation of a non-negative random varialbe, we first write that
\begin{align*}
	\bE\left[e^{n \frac{D_\pi(F_{1,n}, \mu_1)}{1+a_n}}-1\right]  & = \int_{0}^{+\infty} \bP\left(e^{n \frac{D_\pi(F_{1,n}, \mu_1)}{1+a_n}}-1 > x\right) \mathrm{d}x \\
	& =  \int_{0}^{+\infty} \bP\left(D_\pi(F_{1,n}, \mu_1-\epsilon)> \frac{1+a_n}{n}\log\left(1+x)\right)\right) \mathrm{d}x  \;. \\
\end{align*}
We then conclude by using (A4). For the case $a_n>0$, we obtain
\[\bE\left[e^{n \frac{D_\pi(F_{1,n}, \mu_1)}{1+a_n}}-1\right] \leq  \int_{0}^{+\infty} \frac{\alpha_n}{(1+x)^{1+a_n}} \mathrm{d}x = \frac{\alpha_n}{a_n} \;.  \]
Finally, for the case when $a_n=0$ and the divergence is upper bounded, we similarly obtain the upper bound $\alpha_n \times n D_{\mu_1}^+$ (for $D_{\mu_1}^+$ defined in Theorem~\ref{th::main_result}). In all cases, the scaling of these terms multiplied by $c_n$ (in $n$) guarantees that $b_2$ is finite. Furthermore, an explicit upper bound can be obtained for specific instances of all parameters.

\subsection{\rev{Discussions}}\label{subsubsec::discussion}

In this section, we discuss the regret upper bound established in Theorem~\ref{th::main_result}, along with the intermediate result introduced in its proof sketch (Lemma~\ref{th::first_bound}). We begin by comparing Lemma~\ref{th::first_bound} with other general analyses of Thompson Sampling (TS) from the literature and then expand on how these results can lead to problem-independent bounds.

\subsubsection{\rev{Comparison with other generic results from the literature}} 
As previously mentioned, the proof of Lemma~\ref{th::first_bound} is inspired by the modern analysis of TS in \citet{agrawal13optTS, AgrawalG17}. Their work has inspired subsequent authors to derive general theorems that provide a framework for the analysis of TS (index) policies. Notably, Theorem 1 in \citet{Giro}, applied to a bootstrapping policy, and Theorem 36.2 in \citet{BanditBook} fall into this category. When they are applied to different policies, the bounds of the two theorems match. For completeness, we restate the latter below, adapting the notation to align with the conventions used in this paper.
\begin{theorem}[Theorem 36.2 in \citealp{BanditBook}]\label{th::lattimore}
Consider a TS policy, as defined in Algorithm 4.2 from \cite{russo18survey}. Then, for any $\epsilon>0$,  
\[\forall k:\; \Delta_k >0, \;\bE[N_k(T)] \leq 1 + \underbrace{\bE\left[\sum_{n=0}^{T-1}\ind(G_{kn}>1/T)\right]}_{R_1} + \underbrace{\bE\left[\sum_{n=0}^{T-1} \left(\frac{1}{G_{1n}}-1\right) \right]}_{R_2}\;,
\]
with $G_{kn}=\bP(\wt \mu_{k,n}\geq \mu_1-\epsilon|X_{k,s_1}, \dots, X_{k,s_n})$, and $\wt \mu_{k,n}$ denotes the expectation of arm $k$ under a model randomly sampled from the sampler distribution, using as inputs the $n$ first observations collected from arm $k$. By convention, $\wt \mu_{k0}$ is a sample from the prior distribution.
\end{theorem}
The quantities $(G_{k,n})_{k \in [K], n \in \N}$ corresponds to the Boundary Crossing Probabilities (BCPs). 
We can compare Theorem~\ref{th::lattimore} with the intermediate step in our analysis, presented in Lemma~\ref{th::first_bound}. Below, we outline a simplified comparison between the two results. Our key observations are:
\begin{itemize} 
\item 
The terms $b_1$ and $b_2$ from the lemma are analogous to $R_1$ and $R_2$, in Theorem~\ref{th::lattimore}, but are more explicit thanks to Assumption~\ref{ass::sampling_prob}. 
\item After completing the comparison we find that an additional term, not covered by the lemma, remains in Theorem~\ref{th::lattimore}. As a result, analyzing TS using Theorem~\ref{th::lattimore} requires a deeper understanding of the properties of BCPs than using Lemma~\ref{th::first_bound} to analyze its \TS{} counterpart, a distinction we elaborate on further. \end{itemize}

\paragraph{Detailed comparison} To formalize the comparison, let us assume that the BCPs instead of $p_k^\pi(t)$ satisfy the bounds of Assumption~\ref{ass::sampling_prob} for some divergence $D$. For simplicity, we omit multiplicative terms and directly assume that $G_{k,n} = e^{-nD(F_{k,n}, \mu_1 - \epsilon)}$. The arguments below can be adapted as needed, as demonstrated in the proof of Theorem~\ref{th::main_result}. From this, it is straightforward to derive the first-order logarithmic term in Lemma~\ref{th::first_bound}, and the component $b_1$ of $B(c, \epsilon)$. For any $c>0$, let us define $n_k(t)= \left\lceil \frac{1+c}{D(F_k,\mu_1)}\log(T)\right\rceil$. It holds that

\begin{align*}
\forall c>0,\; R_1 & \leq \bE\left[\sum_{n=0}^{n_k(T)-1}\ind(G_{kn}>1/T)\right] + \bE\left[\sum_{n=n_k(T)}^{T-1}\ind(G_{kn}>1/T)\right]  \\
& \leq n_k(T) + \bE\left[\sum_{n=n_k(T)}^{T-1}\ind\left(G_{kn}>1/T, D(F_{k,n}, \mu_1-\epsilon)\leq \frac{D(F_k,\mu_1)}{1+c} \right)\right] \\ &\quad+ \bE\left[\sum_{n=n_k(T)}^{T-1}\ind\left(G_{kn}>1/T, D(F_{k,n}, \mu_1-\epsilon)\geq \frac{D(F_k,\mu_1)}{1+c} \right)\right]\;.\\
\end{align*}
The fact that we introduced the event $\{D(F_{k,n}, \mu_1-\epsilon)\geq \frac{D(F_k,\mu_1)}{1+c}\}$ now allows to provide an upper bound depending on term $b_1$ of Lemma~\ref{th::first_bound}, as detailed in the following,
\begin{align*}
R_1& \leq n_k(T) + \sum_{n=n_k(T)}^{T-1} \bP\left(D(F_{k,n}, \mu_1-\epsilon)\leq \frac{D(F_k,\mu_1)}{1+c} \right) \\
&\quad+ \bE\left[\sum_{n=n_k(T)}^{T-1}\ind\left(G_{kn}>1/T, D(F_{k,n}, \mu_1-\epsilon)\geq \frac{D(F_k,\mu_1)}{1+c} \right)\right] \\ 
&\leq n_k(T) + \sum_{n=n_k(T)}^{T-1} \bP\left(D(F_{k,n}, \mu_1-\epsilon)\leq \frac{D(F_k,\mu_1)}{1+c} \right) \\
& \leq 1 + \frac{1+c}{D(F_k,\mu_1)}\log(T) + b_1 \;. \\
\end{align*}
Hence, it is relatively straightforward to convert $R_1$ into some of the terms involved in Lemma~\ref{th::first_bound}, under the requested assumption on the BCP. We can then discuss the term $R_2$, corresponding to the regret incurred due to under-estimating the best arm. Let us omit again some multiplicative constants in the approximation $p_{1,n}\approx G_{1,n}$ when discussing Lemma~\ref{th::first_bound}. Then, considering that $p_{1,n}=G_{1,n}$ and that $G_{1,n}\geq c_n^{-1} e^{-nD(F_{1,n}, \mu_1-\epsilon)}$ in the case when $\mu_{1,n}\leq \mu_1-\epsilon$, for some sequence $(c_n)_{n\in \N}$, we can directly write that 
\begin{align*}
R_2 & = \bE\left[\sum_{n=0}^{T-1} \left(\frac{1}{G_{1n}}-1\right)\ind(\mu_{1,n}\leq \mu_1-\epsilon) \right]+ \bE\left[\sum_{n=0}^{T-1} \left(\frac{1}{G_{1n}}-1\right)\ind(\mu_{1,n}\geq \mu_1-\epsilon) \right]  \\ & \leq
b_2 + \underbrace{\bE\left[\sum_{n=0}^{T-1} \left(\frac{1}{G_{1n}}-1\right)\ind(\mu_{1,n}\geq \mu_1-\epsilon) \right] }_{b_3} \;,\\ 
\end{align*}
where the second line is obtained by following the same steps as the proof of Theorem~\ref{th::main_result}, 
ignoring the $-1$ term. This completes the demonstration that under the assumptions of Lemma~\ref{th::first_bound} we can translate the bound of Theorem~\ref{th::lattimore} (applied to standard TS) into the bounds of Lemma~\ref{th::first_bound} up to an additional term, denoted by $b_3$ above. 

\paragraph{Interpretation of $b_3$ and benefits of \TS} Let us now discuss the term $b_3$. Intuitively, the regret due to the under-performance of the best arm, when it is not significantly mis-estimated, should not be very large. However, from a technical standpoint, it is not straightforward to obtain tight bounds on the terms contributing to $b_3$. One possible approach, followed for instance by \cite{RiouHonda20}, is to split the expectation into two scenarios: one where $\mu_{1,n} \in [\mu_1 - \epsilon, \mu_1 - \epsilon/2]$ and another where $\mu_{1,n} \geq \mu_1 - \epsilon/2$. For the latter, an exponentially decreasing upper bound on $1 - G_{1,n}$ is sufficient due to the gap between $\mu_1 - \epsilon$ and $\mu_1 - \epsilon/2$. Indeed, in this case a large deviation occurs when the sampled expectation is below $\mu_1-\epsilon$. However, the intermediate case can be more challenging in some settings, as it requires precise bounds on the BCP in that specific range. Choosing again the analysis of NPTS in \cite{RiouHonda20} for comparison, our analysis of NPTS$^\star$ does not require their Lemma 17, whose proof is quite technical and heavily specific to the considered model.

\rev{For this reason, our claim that the \TS{} framework simplifies the analysis of TS stems from the observation that analyzing \TS{} does not require these additional properties on the BCP, which can be difficult to obtain. Specifically, under the \TS{} framework, it is necessary to lower-bound the BCP \textbf{only in the case where the best arm experiences a large deviation} (in the sense of the large deviation principle, see \citealp{DemboZ98}), that is potentially making it appear sub-optimal. 
}

We can finally mention that in \cite{AgrawalG17} the authors provide a fully explicit bound on $b_2$ for Beta-Bernoulli and Gaussian Thompson Sampling, respectively obtained in their Lemmas 2.9 and 2.13. However, the proofs are also very technical and specific to each distribution considered. Hence, there does not seem to be a clear way to generalize this specific part of their analysis to other families of distributions. 

\subsubsection{\rev{Scaling of the constants and problem-independent analysis}}\label{subsec::pb_ind}

In this section, we discuss the additional steps required to transform the instance-dependent regret bound of Theorem~\ref{th::main_result} into problem-independent bounds.
For clarity, we focus on the regret bounds for MED policies, that all satisfy Assumption~\ref{ass::sampling_prob} (see Lemma~\ref{lem::ass_prob_MED}) for some divergence $D$ and a sequence $(a_n)_{n\in \N}$ chosen by the decision-maker, with constants $c_n = C_n = 1$ for all $n \in \N$. The arguments presented below can however be adapted with other values (leading to different rates in the results).

By following the proof structure of Theorem~\ref{th::main_result}, and Remark~\ref{rem::scaling} ($(C_n)_{n\in \N}$ is constant), if $D$ satisfies Assumption~\ref{ass::sufficient_prop} on $\cF$ then the regret upper bound is obtained by proving that
\begin{equation}\label{eq::explicit_bound_MED}
\forall k: \Delta_k >0, \; \bE[N_k(T)] \leq \frac{(1+c) \log(T\cdot D(F_k, \mu_1)) + C}{D(F_k, \mu_1)} + A_{\delta_c} + \sum_{n=1}^{T} (d_n+1) e^{-n \frac{\delta_{\epsilon, \mu_1}}{1+a_n}} \;,   
\end{equation}
where the sequence $(d_n)_{n\in \N}$ is defined as $d_n=\frac{\alpha_n}{a_n} \wedge \alpha_n nD_{\mu_1}^+$, which unifies 
the two bounds obtained in Theorem~\ref{th::main_result}, whether $(a_n)_{n\in \N}$ is strictly positive or not. The terms in this bound are explicitly detailed in Assumption~\ref{ass::sufficient_prop} and the proof of Theorem~\ref{th::main_result}:
\begin{itemize}
    \item $\delta_c = D(F_k, \mu_1) \times \frac{c-a_\infty}{4(1+c)}$, with $a_\infty=\displaystyle\lim_{n\to \infty} a_n$, 
    \item $A_{\delta_c}$ is an upper bound on $\sum_{n=1}^{+\infty} \bP(D(F_{k,n}, \mu_1)\leq D(F_k, \mu_1)-\delta_c)$, from (A2),
    \item For a fixed value of $c$, $\epsilon$ is defined as the unique solution of
    \[D(F_k, \mu_1-\epsilon) = D(F_k, \mu_1) \times \frac{1+\frac{c-a_\infty}{2}}{1+c} \;, \]
    \item $\delta_{\epsilon, \mu_1}$ is a lower bound for $D(F_{1,n}, \mu_1)-D(F_{1,n}, \mu_1-\epsilon)$, from (A4), when $\mu_{1,n}\leq \mu_1-\epsilon$.
\end{itemize}
Previously, we aimed to optimize the first-order logarithmic term of the bound, so we proved that for any (small) value of $c$ we could derive an upper bound with asymptotic scaling $(1+c)\frac{\log(T)}{D(F_k, \mu_1)}$. Hence, the parameter of interest was $c$, and we tuned other parameters like $\epsilon$ as a function of $c$ in the analysis. Under a problem-independent perspective, we focus on how \emph{each} term scales with respect to problem parameters. Hence, we propose to re-examine the bound in Eq.~\eqref{eq::explicit_bound_MED} by fixing $\epsilon$ proportional to the gap, for instance $\epsilon = \frac{\Delta_k}{3}$, as typical in such analysis. We recall from Assumption~\ref{ass::sampling_prob} that $(a_n)_{n\in \N}$ is non-increasing, so $a_1=\max_{n\in [N]}a_n$.

\paragraph{Re-writing of the regret bound} To proceed further with the analysis, we need to understand three key terms from \eqref{eq::explicit_bound_MED}: (1) how
the maximum possible value of
$\delta_{\epsilon, \mu_1}$
in (A3)
scales with $(\epsilon, \mu_1)$, (2) what is an appropriate value of $c$ given $\epsilon$,
(3) how $A_{\delta}$ in (A2) scales as a function of $\delta$.
To start with, we solve the problem of choosing a suitable value of $c$, distinct from the one used for the problem-dependent analysis and with a different relationship with $\epsilon=\frac{\Delta_k}{3}$, by using the convenient tuning 
\[c = (1+a_1) \frac{D(F_k, \mu_1)}{D(F_k, \mu_1-2\epsilon)}-1,\]
which allows to use $\delta \coloneqq D(F_k, \mu_1-\epsilon)-D(F_k, \mu_1-2\epsilon)$ in the part of the proof where assumption (A2) is involved.
We can then apply (A3) to obtain the inequality $\delta \geq \delta_{\epsilon, \mu_1-\epsilon}$. For convenience, we define $\gamma_k=\min\{\delta_{\epsilon, \mu_1-\epsilon}, \delta_{\epsilon, \mu_1}\}$ for the rest of the analysis. By substituting all constants into the regret bound~\eqref{eq::explicit_bound_MED}, we obtain that 
\begin{equation}\label{eq::explicit_alternative_MED}
\forall k: \Delta_k >0, \; \bE[N_k(T)] \leq (1+a_1)\frac{\log(T \cdot D(F_k, \mu_1)) + C}{D\left(F_k, \mu_1-\frac{2\Delta_k}{3}\right)} + A_{\gamma_k} + \sum_{n=1}^{T} (d_n+1) e^{-n \frac{\gamma_k}{1+a_n}} \;. 
\end{equation}
It is clear that further upper bounds can only be obtained with precise bounds on $A_{\gamma_k}$ (A2), $\gamma_k$ (A3), $(\alpha_n)_{n\in \N}$ (A4), and on $D\left(F_k, \mu_1-\frac{2\Delta_k}{3}\right)$. These bounds have to be specified for each policy and family of distributions considered. In the following, we first present the case of the sub-Gaussian divergence, which allows to highlight the strengths and limitations of the generic framework that we propose regarding the derivation of problem-independent guarantees.

\paragraph{Sub-Gaussian case} We consider sub-Gaussian distributions, and the corresponding divergence $D: (F, \mu_1)\mapsto \frac{(\mu_1-\mu_F)^2}{2}\ind(\mu_1\geq \mu_F)$. Note that this corresponds to the MED instance known as \emph{Maillard Sampling} \citep{Maillard11, jun2022} or \emph{SoftElim} \citep{boutilier20}. We also assume that the means belong to a finite range $\cR = [\mu_0^-, \mu_0^+]$, with $\mu_0^+-\mu_0^-\coloneqq R<\infty$. This is necessary to obtain problem-independent bounds, and otherwise we could consider a problem where one pull of a sub-optimal arm could already cost a linear regret. We introduce the notation $\lesssim$ to avoid writing (problem-independent) multiplicative constants, and prove the following result.
\begin{lemma}[Problem-independent bound for sub-Gaussian MED]\label{lem::pb_ind_sg}
Let $\cF$ be the family of $1$-sub-Gaussian distributions with means in $\cR=[\mu_0^-, \mu_0^+] \subset \R$, and consider MED implemented with $D: (F, \mu_1) \in \cF \times \cR \mapsto \frac{(\mu_1-\mu_F)^2}{2}\ind(\mu_1\geq \mu_F)$ and a non-negative and non-increasing sequence $(a_n)_{n \in \N}$. Then, the upper bound presented in Eq.~\eqref{eq::explicit_alternative_MED} provides the problem-independent regret bound
\[ \cR_T \lesssim  K^\frac{1}{4}T^\frac{3}{4} \wedge \sqrt{KT\times (\log(K)+a_T^{-1})} \;. \]
\end{lemma}

\begin{proof} 
First, we use Hoeffding's inequality and direct computation to obtain the following bounds 
\[A_{\gamma_k}\lesssim \frac{1}{\gamma_k},\quad \text{and} \quad \gamma_k = D\left(F_k, \mu_k+\frac{\Delta_k}{3}\right) = \frac{\Delta_k^2}{18},\]
and in addition Hoeffding' inequality also provides that $\alpha_n=1$ in (A4) for the setting considered.
Then, plugging these values in~\eqref{eq::explicit_alternative_MED} we obtain that for any $\Delta>0$ (that we will tune later) the following bound holds,
\begin{equation}\label{eq::pb_ind_subgauss}
\cR_T \lesssim  \Delta T + \sum_{k: \Delta_k>\Delta} \left\{\frac{\log(T\Delta_k^2)}{\Delta_k}+\frac{1}{\Delta_k} + \underbrace{\Delta_k \sum_{n=1}^T \left\{\frac{1}{a_n}\wedge n\frac{R^2}{2} \right\} e^{-n \frac{\Delta_k^2}{36}}}_{S} \right\} \;. 
\end{equation}
In the worst case (e.g. if $a_T=0$), the right-hand term is dominated by the sum $S$, and by putting the range $R$ in the constants we obtain that
\[ \cR_T \lesssim \Delta T + \sum_{k: \Delta_k>\Delta} \Delta_k \sum_{n=1}^T n e^{-n\frac{\Delta_k^2}{18}} \lesssim \Delta T + \frac{K}{\Delta^3}\;, \]
which further provides $\cR_T \lesssim K^\frac{1}{4} T^\frac{3}{4}$, by choosing $\Delta=(K/T)^\frac{1}{4}$. This gives the first part of the bound, and we now prove the potential refinement depending on $a_T$. First, we use that $(a_n)_{n\in \N}$ is non-increasing and thus lower bounded by $a_T$. Hence, if $a_T>0$ it holds that $S\leq \frac{a_T^{-1}}{\Delta_k}$, which in turn provides 
\[\cR_T \lesssim \Delta T + \sum_{k:\Delta_k>\Delta}\frac{\log(T\Delta_k^2)+a_T^{-1}}{\Delta_k}. \]
It then remains to choose $\Delta \gtrsim \sqrt{(\log(K)+a_T^{-1})\frac{K}{T}}$ to prove the desired bound, where the multiplicative constant here only comes from making sure that $\frac{\log(T\Delta_k^2)}{\Delta_k}$ is non-increasing in $\Delta_k$ for $\Delta_k\geq \Delta$. This proves the lemma.
\end{proof}

\paragraph{Discussion} The scaling of the second term in the regret bound is directly influenced by the choice of the sequence $(a_n)_{n\in \N}$, determined by the decision-maker. On the other hand, while the bound $K^\frac{1}{4}T^\frac{3}{4}$ does not reach the optimal $\sqrt{KT}$ rate, it remains valid universally. This proves that the generic analysis that we propose leads to sub-linear problem-independent bounds in addition to the logarithmic problem-dependent bound.

To gain further insight, we adopt the notion of the \emph{minimax ratio} (as discussed in \citealp{jun2022, Qin23}), which is the scaling of $\cR_T/\sqrt{KT}$. From the bound of the lemma, it becomes evident that the decision-maker can optimize this ratio by choosing a sequence $(a_n)_{n\in \N}$ with a lower bound $a>0$. Then, the minimax ratio $\sqrt{\log(K)}$ is obtained, matching the result proved by \cite{Qin23} for Maillard Sampling. However, this optimization comes at some cost for problem-dependent performance: the best achievable asymptotic rate becomes $(1 + a) \frac{\log(T)}{D(F_k, \mu_1)}$. Therefore, our approach does not simultaneously yield the best problem-dependent rate (obtained only if $\lim_{n \to +\infty} a_n = 0$) and the minimax ratio of $\sqrt{\log(K)}$, contrarily to the more involved proofs of \cite{Qin23} for kl-MS or \cite{AgrawalG17} for Gaussian TS with known variance. 

We believe that the looseness of the problem-independent bound for $a_n=0$ comes from the fact that, with (A3) and (A4), we tried to identify minimal properties needed to upper bound \[\sum_{n=1}^T \bE\left[\ind(\mu_{1,n}\leq \mu_1-\epsilon)e^{n \frac{D(F_{1,n}, \mu_1-\epsilon)}{1+a_n}}\right] \coloneqq \; b_2 \text{ from Lemma~\ref{th::first_bound}},\] 
instead of computing the expectation in a more precise manner with more assumptions.  
We propose to discuss the case of Gaussian distributions to gain further insights. Indeed, with $a_n=0$ for all $n\in \N$, $b_2$ can be computed explicitly as follows,
\begin{equation}\label{eq::b2_gauss}
F_1 = \cN(\mu_1,1) \Longrightarrow b_2 = \sum_{n=1}^T \frac{e^{-n\frac{\epsilon^2}{2}}}{\epsilon\sqrt{2\pi n}} \Longrightarrow b_2 \lesssim \frac{1}{\epsilon^2}, \; \end{equation}
where we omit the computation because it is straightforward, and the last step is proved by splitting the sum depending on the value of $n$ with respect to $\epsilon^{-2}$.
By using this bound we can replace $S \lesssim K/\Delta^3$ in \eqref{eq::pb_ind_subgauss} with $\Delta \epsilon^{-2}\approx \Delta^{-1}$ for $\epsilon \approx \Delta$, which leads to the guarantee of $\cR_T \lesssim \sqrt{KT\log(K)}$ for the Gaussian case.
This proves that better problem-independent bounds requires more precise analysis of the term $b_2$ in the regret upper bound provided by Lemma~\ref{th::first_bound}. While this is certainly a limitation of our framework, we believe that it is interesting that at least sub-linear problem-independent bounds can be derived under the generic assumptions that we propose. We leave the investigation of problem-independent upper bounds that would be both tighter and obtained under generic assumptions for future work.

\paragraph{Generalization beyond the sub-Gaussian case} To the best of our knowledge, in the literature problem-independent guarantees are always achieved by a reduction to the sub-Gaussian case discussed above, in the sense that all the quantities involved in the regret bounds are expressed as squared gaps. We refer for instance to \cite{Qin23, jin22expTS} for recent examples of such reduction in the case of exponential families of distributions, respectively for the analysis of MS and TS. It is straightforward that this reduction can be used within our framework, if the following properties are established in addition of Assumption~\ref{ass::sufficient_prop}. 

\begin{assumption}[Additional properties for problem-independent analysis]\label{ass::reduction_sg}
The divergence $D$ satisfies Assumption~\ref{ass::sufficient_prop}, and in addition it holds that
\begin{enumerate}
    \item $\forall k: \Delta_k>0$, $D\left(F_k, \mu_k+\frac{\Delta_k}{3}\right) \gtrsim \Delta_k^2$ (the divergence scales with squared gaps),
    \item $\forall k: \Delta_k>0$, $\delta_{\frac{\Delta_k}{3}, \mu_1-\frac{\Delta_k}{3}} \gtrsim \Delta_k^2$ (same for the difference of divergences in (A4)),
    \item in (A2), $A_\delta \lesssim \frac{1}{\delta}$,
\end{enumerate}
where we recall that we use the notation $\lesssim$ to hide absolute constants.
\end{assumption}
This assumption is a summary of the ingredients needed for the proof steps of Lemma~\ref{lem::pb_ind_sg} to hold, from which the next proposition naturally follows. 
\begin{proposition}[Problem-independent bounds for MED (general case)] If MED is implemented with a sequence $(a_n)_{n \in \N}$ and a divergence $D$ satisfying Assumption~\ref{ass::reduction_sg} on $\cF$, and in addition that $\alpha_n=1$ in (A4) for all $n\in \N$, then it admits the problem-independent bound of Lemma~\ref{lem::pb_ind_sg}.
\end{proposition}
This assumption is easy to verify for simple parametric settings (see \cite{jin22expTS} for SPEF), but we leave the investigation of these properties for non-parametric settings for future works.
Hence, in the remainder of this work we focus on applications of Theorem~\ref{th::main_result} to derive problem-dependent regret upper bounds for MED and \TS{} policies in various settings. 


\section{\rev{Applications of Theorem~\ref{th::main_result}: regret analysis of various policies}}\label{sec::applications} 

In this section, we propose some applications of the proposed framework by proving instance-dependent regret guarantees for some MED under various models for the family of distributions, and for some TS$^\star$ policies adapted from existing TS policies.

In each case, we demonstrate that the policies in question satisfy Assumptions~\ref{ass::sampling_prob}, for a divergence satisfying Assumption~\ref{ass::sufficient_prop}, enabling us to apply Theorem~\ref{th::main_result}. For \TS{} policies we believe that our framework offers a simplified proof compared to the existing analyses of their TS counterpart.

For both policy types, verifying Assumption~\ref{ass::sampling_prob} is made easier by the groundwork laid in Section~\ref{subsubsec::sampling}. Specifically, we showed in Lemma~\ref{lem::ass_prob_MED} that MED policies meet the conditions of Assumption~\ref{ass::sampling_prob}, and Corollary~\ref{lem::ass_prob_TS} provides a method to prove this for \TS{} policies by bounding the Boundary Crossing Probabilities (BCP) in the specific case considered.

The results in this section address the key question raised in Section~\ref{sec::introduction}: there is indeed a set of fundamental properties shared by all the main families of distributions that have been studied in the bandit literature, that seems to account for the success of common algorithmic principles across them. Assumption~\ref{ass::sufficient_prop} appears to capture these shared properties. Moreover, we prove that this assumption typically holds for $\kinf^\cF$, or a function that coincides with $\kinf^\cF$ once the empirical distribution of a sub-optimal arm is sufficiently close to its true distribution, thereby establishing the asymptotic optimality of the considered policies.

\subsection{Single-Parameter Exponential Families (SPEF)}\label{subsec::spef}

We start by recalling the definition of these families of distributions. 

\begin{definition}\label{def::spef} $\cF$ is a SPEF if there exists a set $\Theta\subset \R$, a measure $\eta$ and some function $b: \Theta \mapsto \R$ such that \[\forall \nu \in \cF\;, \exists \theta \in \Theta\;:\; \forall x \in \R\;, \; \frac{\mathrm{d}\nu_\theta}{\mathrm{d}\eta}(x) = \exp(\theta x - b(\theta)) \;. \]
We thus denote this specific SPEF by $\cF_b^\Theta$.
\end{definition}
Furthermore, for any parameter $\theta \in \Theta$, the expectation of $\nu_\theta$ is $\mu_\theta=b'(\theta)$ (one-to-one mapping). Hence, we can indifferently write the KL-divergence between two distributions $\nu_\theta$ and $\nu_{\theta'}$ in terms of their parameters or as a function of their expectations. We choose the second formulation, and further recall that for any SPEF it holds that
\begin{equation} \label{eq::kl_spef}\forall F \in \cF_b^\Theta:\; \kinf^{\cF_b^\Theta}(F, \mu) = \kl(\mu_F, \mu) \coloneqq \int_{\mu_F}^{\mu} \frac{x-\mu_F}{V_b(x)} \mathrm{d}x\;, \end{equation}
for $\mu\geq \mu_F$, and $\kinf^{\cF_b^\Theta}(F, \mu)=0$ otherwise. The function $V_b:x\to \R^+$ is used to denote the variance of the distribution with expectation $x$. A proof for this expression of $\kl$ can be found for instance in  \citep{jin22expTS}.
SPEF include many standard parametric types of distributions: Bernoulli, Gaussian with known variance, Exponential, Poisson, \dots, and thus are a standard model in bandits.

\begin{lemma}\label{lemm::cond_spef}
For any SPEF $\cF_b^\Theta$, Assumption~\ref{ass::sufficient_prop} holds for the $\kinf^{\cF_b^\Theta}$ divergence 
    \end{lemma}
We defer the proof to Appendix~\ref{subsec::SPEF}. Most of the properties can be deduced from Chernoff's inequality (see e.g. \citealp{KL_UCB}), for any arm $k \in [K]$ it holds that
\begin{equation}\label{eq::chernoff} \forall x \geq \mu_k : \; \bP(\mu_{k,n} \geq x) \leq e^{-n \kl(x, \mu_k)} ,\; \forall y \leq \mu_k: \; \bP(\mu_{k,n} \leq y) \leq e^{-n \kl(y, \mu_k)}\;.\end{equation}    
In addition, the formulation of $\kl$ in Eq.~\eqref{eq::kl_spef} allows to prove (A3) with $\delta_{\epsilon, \mu_1}=\kl(\mu_1-\epsilon, \mu_1)$, and $\alpha_n=1$ in (A4), we refer to the appendix for details.

\begin{corollary}[Instance-dependent optimality of MED and \TS{} for SPEF]\label{cor::spef}
$\!$
\begin{itemize}
    \item MED implemented with the $\kinf^{\cF_b^\Theta}$ divergence 
    is asymptotically optimal on $\cF_b^\Theta$.
    \item if the means belong to a known range $[\mu^-, \mu^+]$, \TS{} with Jeffreys prior \citep{korda13TSexp} is also asymptotically optimal on $\cF_b^\Theta$.
\end{itemize}
\end{corollary}
For completeness, the implementation of the \TS{} policy for this model can be found in Algorithm~\ref{alg::TS_SPEF} in appendix. The main missing ingredient to prove this result is showing that the proposed \TS{} satisfies Assumption~\ref{ass::sampling_prob}. We use Corollary~\ref{lem::ass_prob_TS} and Lemma~\ref{lem::bcp_spef} (detailed and proved in Appendix~\ref{app::ts_spef}), to obtain that the assumption holds with $c_n=\cO(n)$ and $C_n=\cO(\sqrt{n})$. 

The result presented in Corollary~\ref{cor::spef} is novel for the MED part, although not very surprising because it was already proved for the specific case of Bernoulli and Gaussian distributions. Regarding \TS, we retrieve the asymptotic optimality proved in \citet{korda13TSexp} for TS, under the same assumptions. 

\subsection{Gaussian distributions with unknown variances}
We now apply our framework to a second type of parametric families of distributions: Gaussian distributions with unknown variances, that we denote by $\cF_G$. Our main reference for this part is \citet{honda14}, from which we obtain that
\begin{align*} &\kinf^{\cF_G}(F, \mu) = \frac{1}{2}\log\left(1+\frac{(\mu-\mu_F)^2}{\sigma_F^2}\right)\times \ind(\mu\geq \mu_F)\;, \\ \text{with}& \quad \mu_F=\bE_{X\sim F}[X]\quad \text{and} \quad \sigma_F^2 = \bE_{X\sim F}\left[(X-\bE_{X\sim F}[X])^2\right] \;.
\end{align*}
Hence, the distribution $F$ is used in the computation by being mapped to its mean $\mu_F$ and standard deviation $\sigma_F$. Contrarily to the single-parameter case studied in previous section, the fact that the parameter space is unbounded poses some challenges in the analysis. In particular, $\kinf^{\cF_G}(F,\mu)-\kinf^{\cF_G}(F, \mu-\epsilon)$ could be made arbitrarily small in some (low probability) scenarios. In order to simplify the presentation, we further assume that (1) expectations are bounded by a known lower bound $\mu^-\in \R$, and (2) all variances are smaller than a known constant $V<\infty$. We further assume that all policies discussed below clip the empirical estimates they use to satisfy these constraints. In Remark~\ref{rem::gauss} at the end of this section, we propose an adaptation for the general case, with the same idea of clipping the empirical estimates but with adaptation to the sample size. We now prove Assumption~\ref{ass::sufficient_prop} on $\kinf^{\cF_G}$. 

\begin{lemma}\label{lemm:cond_gauss} 
Consider the family $\cF=\{F\in \cF_G:\; \mu_F\geq \mu^-, \sigma_F^2\leq V\}$, then assumptions (A1)--(A3) hold on $\cF$ for $D \coloneqq \kinf^{\cF_G}$, if the mean and variance estimates are respectively clipped to the range $[\mu^-, +\infty)$ and $(0, V]$. Furthermore, the following variant of (A4) holds: for any $x_0>0$, there exists a sequence of constants $(\alpha_n(x_0))_{n \in \N}$ such that
 \[\forall x \geq x_0, \; \bP(D(F_n, \mu_F)\geq x) \leq \alpha_n(x_0) \times e^{-(n-1)x} \;. \]
\end{lemma}
We detail the proof of Lemma~\ref{lemm:cond_gauss} in Appendix~\ref{app::gauss_p15}, where we notably prove that the clipped estimates guarantee that (A3) holds for $\delta_{\epsilon, \mu_1} = \frac{1}{4}\frac{\epsilon^2}{(\mu_1-\mu^-)\vee V}$. Contrarily to all other settings considered in this section, we could not obtain (A4) but a slightly weaker version for (A4). We prove this result in Lemma~\ref{lem::conc_gaussian} in Appendix~\ref{app::gauss_p15}, which is (as far as we know) a novel concentration inequality for the $\kinf^{\cF_G}$ divergence. As discussed in the same appendix, it implies a mild adaptation of the analysis, and we can still deduce the following problem-dependent regret bounds. 

\begin{corollary}[Instance-dependent optimality of MED and Gaussian-\TS{} on $\cF_G$]\label{cor::G}
MED\footnote{We assume that in this setting MED starts by sampling each arm twice, so that all variance estimates are almost surely non-zero} using divergence $D$ (with clipped estimates) and $a_n = \frac{4}{n}$, and Gaussian \TS{} (Algorithm~\ref{alg::TS_G}) with inverse-gamma priors with shape parameter $\alpha< 0$, are both asympotically optimal for Gaussian distributions with a known lower bound on their means and upper bound on their variances.
\end{corollary}
We provide in Lemma~\ref{lem::A4_gauss} the main ingredient of the adaptation of the proof of Theorem~\ref{th::main_result} to the modified condition (A4), from which the result is direct for MED.
For Gaussian-\TS{} (Algorithm~\ref{alg::TS_G}), it remains to upper and lower bound the BCP in the appropriate form. In Lemma~\ref{lem::bcp_gts} (Appendix~\ref{app::gauss_p15}) we prove these bounds, proving that Assumption~\ref{ass::relaxed_sp} holds with $c_n=\cO(\sqrt{n})$, $C_n=\cO(1/\sqrt{n})$ for this \TS{} policy. 
This concludes the proof.

As in the previous section, the optimality result for MED for Gaussian distributions is novel, and the analysis of Gaussian-\TS{} recovers the bound proved for the corresponding TS by \cite{honda14}.

\begin{remark}[Guarantees with unbounded parameters]\label{rem::gauss}
In the case where $V$ and $\mu^-$ are infinite or unknown one can use a clipping adaptive to the sample size (e.g. with $\mu_0^-(n) =-\sqrt{n}$ and $V=\sqrt{n}$) and recover the asymptotic guarantees of Corollary~\ref{cor::G}. The adaptation of the proof is straightforward.
\end{remark}

\subsection{Bounded distributions}\label{subsec::bounded}

In this part we consider the family of distributions with a bounded support $[0,B]$, that we denote by $\cF_B$. The particularity of this example is that it is well established that several divergences can lead to theoretically sound regret bounds. We consider the optimal one, which is the KL-divergence involved in the lower bound \eqref{eq::asymp_opt_def}, that can be formulated (see e.g. \citealp{Honda15IMED}) as the following optimization problem,
\begin{equation}\label{eq::kinf_B}
\forall F\in \cF_B, \mu \in [0,B]: \; \kinf^{\cF_B}(F, \mu) =  \max_{\lambda \in \left(0, \frac{1}{B-\mu}\right]}\bE_{X\sim F}\left[\log\left(1-\lambda(X-\mu)\right)\right]\;.\end{equation}
This choice is not the most popular in the literature because it is harder to compute than some relaxations: for $B=1$, one can also consider the Bernoulli divergence $\kl(\mu_F, \mu)$, or the $\frac{1}{4}$-sub-Gaussian divergence, $2(\mu-\mu_F)^2$. It is folklore (see e.g. \citealp{KL_UCB}) that they satisfy the following ordering,
\[ \forall F\in \cF_1, \mu\in [\mu_F,1]: \; \kinf^{\cF_1}(F, \mu) \geq \kl(\mu_F,\mu) \geq 2(\mu-\mu_F)^2 \;.  \]
However, in the case where a distribution $F$ has a small density near the boundaries of the support, $\kinf^\cF(F,\mu)$ can be much larger than its relaxations. We refer to \citet{Baudry23fast} for a formal discussion of this phenomenon, as well as concrete illustrations. They prove that in some cases asymptotically optimal algorithms can perform significantly better both in theory and practice that non-optimal ones, which motivates further investigation of this setting. We first prove the following result. 

\begin{lemma} \label{lemm::bounded_div} Assumption~\ref{ass::sufficient_prop} holds on $\cF_B$ for the $\kinf^{\cF_B}$ divergence.
\end{lemma}
The proof relies on several results that can be found in the literature \citep{HondaTakemura10, KL_UCB}, we refer to Appendix~\ref{app::bounded_P15} for the details. We obtain for instance that $\delta_{\epsilon, \mu}= \frac{\epsilon^2}{2(\mu-\epsilon)(B-\mu)}$, and that $\alpha_n=e(n+2)$. We now establish the regret bound for the MED and \TS{} algorithm of this section.

\begin{corollary}[Instance-dependent optimality of MED and \TS{} on $\cF_B$]\label{cor::bounded}
If the upper bound $B$ is known, then the MED policy using the $\kinf^{\cF_B}$ divergence and $a_n=0$ for all $n$, and the Non-Parametric TS$^*$ algorithm (Algorithm~\ref{alg::NPTS}) are both asymptotically optimal. 
\end{corollary}

As in previous sections, the result is direct for MED given Lemma~\ref{lemm::bounded_div} and Theorem~\ref{th::main_result}. For NPTS$^\star$, we prove in Lemma~\ref{lem::bcp_npts_B} in Appendix that Assumption~\ref{ass::relaxed_sp} holds, or that Assumption~\ref{ass::sampling_prob} holds for $c_n=\cO(e^{\alpha\sqrt{n}})$ for some $\alpha>0$ and $C_n=\cO(n)$. 

Corollary~\ref{cor::bounded} is the first proof of optimality for this version of MED. This result can also be found in \citet{Baudry23fast}, but their proof is based on the framework that we propose from an earlier version of this work. Notably, they also propose the Fast-MED and Online-MED as variants of MED that improve its computational and memory efficiency (for the latter) while preserving its theoretical guarantees, essentially by using tight approximations of $\kinf^{\cF_B}$ instead of solving the optimization problem \eqref{eq::kinf_B} for all arms at each time step. Furthermore, in Section~\ref{subsubsec::discussion} we discussed how the proof of optimality of NPTS$^\star$ with our framework is more direct than the original proof of \cite{RiouHonda20} for NPTS, and we detailed which of their technical results become unnecessary with our framework. 

\subsection{Distributions satisfying a $h$-moment condition}\label{subsec::hcond}

We now consider families of distributions satisfying a $h$-moment condition \citep{agrawal2020optimal, agrawal2021regret}. Following their definition, we assume that the learner knows a positive convex function $h$ satisfying $\frac{h(|x|)}{|x|^{1+\eta}} \to + \infty$ for some $\eta>0$, a constant $B>0$. We then assume that the distribution comes from a distribution $\cF_{h,B}$ defined as follows,
\begin{equation}\label{eq::uncentered_main}\cF_{h,B} = \left\{F \in \cP:\; \bE_{X\sim F}\left[h(|X|)\right] < B \;, \; \bE_{X\sim F}[h(|X|)^{2+\epsilon_F}]<+\infty \text{ for some } \epsilon_F>0 \right\} \;.\end{equation} 
In that case, we say that the $h$-moment condition is \emph{uncentered}. We will later also consider the \emph{centered} case, where $|X|$ is replaced $|X-\bE_{X\sim F}[X]|$. The main element of this definition is the condition $\bE_{X\sim F}\left[h(|X|)\right] < B$. For instance, if $h(x)=x^2$ this can be an assumption of the second-moment of the distribution (the variance in the centered case). Compared to the definition of \cite{agrawal2020optimal}, we added the second condition $\bE_{X\sim F}[h(|X|)^{2+\epsilon_F}]<+\infty$ for technical reasons, justified by the regret analysis. Indeed, this condition ensures with sufficiently large probability that the empirical distributions belong to $\cF$. We refer to Appendix~\ref{app::h_P15} for details.

\subsubsection{Analysis for the uncentered case}
We can first restate from \cite{agrawal2020optimal} that
\begin{equation}\label{eq::kinf_h}\kinf^{\cF_{h,B}}(F, \mu) = 
	\max_{(\lambda_1, \lambda_2) \in \cR_2}\bE_{F}\left[\log\left(1-\lambda_1(X-\mu)-\lambda_2 (B-h(|X_i|))\right)\right] \;,
 \end{equation}
with $\cR_2= \{(\lambda_1,\lambda_2) \in (\R^+)^2 \;:\; \forall x \in \R \;, 1-\lambda_1 (x-\mu) - \lambda_2 (B-h(|x|)) \geq 0 \}$ is the set for which the logarithm is defined. We now prove the assumption for a divergence build with $\kinf^{\cF_{h,B}}$. 

\begin{lemma}\label{lemm::cond_h}
    Assumption~\ref{ass::sufficient_prop} holds on $\cF_{h, B}$ for the divergence \[D: (F, \mu) \mapsto \kinf^{\cF_{h,B}}(F, \mu) \ind(F \in \cF).\]
\end{lemma}
The detailed proof can be found in Appendix~\ref{app::h_P15}. This family of distributions shows the interest of the generality of our framework. Indeed, compared to simpler parametric families, it is particularly challenging to relate quantities involving $\kinf^{\cF_{h,B}}$ to explicit parameters of the bandit problem, such as the gaps. 

\begin{corollary}[Instance-dependent optimality of MED on $\cF_{h,B}$]\label{cor::h}
The version of MED implemented with the divergence $D$ presented in Lemma~\ref{lemm::cond_h}, and $a_n=0$ for all $n$, is asymptotically optimal on $\cF_{h,B}$.
\end{corollary}
The result is direct from Lemma~\ref{lem::ass_prob_MED} and Lemma~\ref{lemm::cond_h}. 
To the best of our knowledge, MED is the first asymptotically optimal algorithm for general $h$, while \cite{agrawal2021regret} proved the optimality of a KL-UCB algorithm when $h$ is of the form $h(|x|)=|x|^{1+\eta}$ for some $\eta>0$. In addition, we now extend the analysis for the \emph{centered} case, for which no algorithm is yet known to be optimal.

\subsubsection{Analysis for the centered case}\label{subsubsec::centered} The centered case, i.e.\! considering a moment condition of the form $\bE_{X\sim F}[h(|X-\bE_{X\sim F}[X]|)]\leq B$, is more challenging because in general the space of distributions is no more convex nor compact. To overcome this technical difficulty, we further impose the knowledge of a lower bound of the possible means $\mu^- \in \R$, and consider a family of distributions defined as follows,  
\begin{equation}\label{eq::centered_main}\cF_{h,B}^c = \left\{F \in \cP:\; \bE_F\left[h(|X-\mu_F|)\right] < B \;, \; \bE_F[h(|X|)^{2+\epsilon_F}]<+\infty \text{ for some } \epsilon_F>0, \mu_F\geq \mu^- \right\} \;.\end{equation}
Again, the main element of this definition remains the bounded $h$-moment condition. Regarding $\kinf^{\cF_{h,B}^c}$, we prove in Appendix~\ref{app::cent_h_P15} that it has the same form as $\kinf^{\cF_{h,B}}$, where it is just necessary to replace $h(|X|)$ by $h(|X-\mu|)$ in Equation~\eqref{eq::kinf_h}. We can then prove similar results as in the uncentered case.

\begin{lemma}\label{lemm::cond_h_cent}
    Assumption~\ref{ass::sufficient_prop} holds on $\cF_{h, B}^c$ for the divergence \[D: (F, \mu) \mapsto \kinf^{\cF_{h,B}^c}(F, \mu) \ind(F \in \cF).\]
\end{lemma}
The detailed proof can be found in Appendix~\ref{app::cent_h_P15}, as well as the proof of the following regret bound. 

\begin{corollary}[Instance-dependent optimality of MED on $\cF_{h,B}^c$]\label{cor::centered_h}
The version of MED implemented with the divergence $D$ presented in Lemma~\ref{lemm::cond_h_cent}, and $a_n=0$ for all $n$, is asymptotically optimal on $\cF_{h,B}^c$.
\end{corollary}
Some illustrative examples of families covered by this result are distributions with bounded variance ($\bE_F[(X-\mu_F)^2] < B$), or some characterization of a subgaussian model (assuming that $\bE_F[\exp(s^{-1}(X-\mu_F)^2)]< B$ for some known $s, B$). We detail this second setting in Appendix~\ref{app::SG}. To the best of our knowledge, MED is for instance the first algorithm with regret guarantees for the bounded variance setting, and the first algorithm to match the lower bound of \cite{burnetas96LB} for a subgaussian model.

\subsection{Summary and discussion}

We conclude this section by proposing a summary of the results that we obtained, and discussing the insights that can be derived from them.

\paragraph{Asymptotic optimality of MED} First,
Corollaries~\ref{cor::spef}--\ref{cor::centered_h}
of Theorem~\ref{th::main_result} establish that MED is asymptotically optimal for all the families of distributions introduced in Section~\ref{sec::introduction}. This finding may not be very surprising for parametric and bounded settings (though not proved yet except for the multinomial case), but is particularly insightful for models involving $h$-moment conditions given in \eqref{eq::uncentered_main}. Indeed, this setting was so far tackled only in the uncentered case and for $h(|x|)=|x|^{1+\epsilon}$ for some $\epsilon>0$, for the regret minimization problem \citep{bubeck_heavy, agrawal2021regret}.
We allow more general definitions of $h$ inspired by \citet{agrawal2020optimal} and tackle the centered case, under the additional assumptions that $\bE_F[h(|X|)^{2+\epsilon}]<+\infty$ for some $\epsilon>0$ and that the means admit a known lower bound. This analysis is permitted by the abstraction of the sufficient properties required for regret upper bounds that we propose in Assumption~\ref{ass::sufficient_prop}.

\paragraph{Simple analysis of \TS} The second type of contributions presented in this section is the novel analysis of some existing TS algorithms implemented according to our proposed variant \TS. In each case, the additional work after completing the analysis of MED consists in upper and lower bounding the BCP in a form that matches \eqref{eq::sampling_prob} with the function $\kinf^\cF$. 
Thanks to this reduction, the same posterior sampling method seems easier to analyze through the \TS{} framework than under the standard (index policy) TS. 

\paragraph{TS as an approximation of MED?} Regarding their very similar guarantees and analysis, we can reasonably interpret that MED and \TS{}  
are actually two variants of the same exploration strategy\rev{, at least in their asymptotic behavior}. \rev{We propose the interpretation} that \TS{} can be seen as a way to approximate MED through sampling of the parameters, instead of directly computing $D_\pi$. \rev{As discussed in Section~\ref{subsec::presentation_TS}, since there is strong evidence that \TS{} and TS are (again, at least asymptotically) equivalent, we believe that this insight carries over the standard TS framework}. In the light of this result, a natural question to ask is if TS or MED should be preferred in practice. In our opinion, in the case of non-informative priors, this depends on two factors: (1) the ability to provide a TS with tight bounds on its BCP under the model $\cF$, and (2) how costly it is to compute $D_\pi$ compared to performing the sampling step of TS. For example, for parametric families the function $D_\pi$ is very easy to compute, and so MED may be an interesting option, but for non-parametric families computing $D_\pi$ (chosen close to $\kinf^\cF$) at each step may be burdensome and TS may be more appealing. This is the main motivation for the novel Non-Parametric TS algorithm that we propose in the next section, for non-parametric families of distributions satisfying an $h$-moment condition.

\section{A simple NPTS algorithm for families with $h$-moment conditions}\label{sec::NPTS}

In this section we illustrate the benefits of the generic approach proposed in this paper for the design and analysis of novel bandit algorithms. We consider the family of distributions satisfying a \emph{centered} $h$-moment condition, which was already introduced in the previous section. The definition of this family and the assumptions made on the function $h$ are detailed in Section~\ref{subsec::hcond}. We choose to study more specifically this family of distributions because it can capture non-parametric assumptions outside the traditional scope of bounded and sub-gaussian distributions. For instance, for any constants $m>1, B>0$, $\epsilon>0$ and $\mu^-\in \R$, we can provide algorithms adapted to a family of distributions with bounded moments: 
\[\cF_{m}^B = \left\{F \in \cP \;, \bE_F[|X-\bE_F[X]|^m] < B \;,\; \bE_F[|X|^{2m+\epsilon}]<+\infty, \bE_F[X]\geq \mu^-  
\right\} \;,\]
as we proved in Section~\ref{subsubsec::centered}.
This model may be relevant in some problems where limited prior knowledge on the distributions of rewards is available. \citealp{bubeck_heavy, agrawal2021regret} provided algorithms for the uncentered version of this constraint, but they do not assume that $\bE_F[|X|^{2m+\epsilon_F}]<+\infty$ and so their model can consider heavier-tail distributions than our definition. We think that providing algorithms for the centered case is a significant advance, since it captures the class of bounded-variance distributions and a subgaussian model as discussed before, even if additional assumptions are needed for the analysis. These assumptions are furthermore not used by the algorithms that we study.
From the technical viewpoint, the centered case is more difficult, mainly because $\mathcal{F}_m^B$ becomes no more convex.

Lemmas~\ref{lemm::cond_h} and \ref{lemm::cond_h_cent} respectively state that Assumption~\ref{ass::sufficient_prop} holds for the divergence satisfying $D_\pi(F,\mu)=\kinf^\cF(F,\mu)\ind(F\in \cF)$ for any $F\in \cF$ and $\mu\in \R$, in both the centered and uncentered case, and hence that the corresponding MED algorithm (with $a_n=0$) is asymptotically optimal. However, computing $D_\pi$ requires to solve an optimization problem at each time step, which can be computationally expensive. Building on the findings of previous sections, we naturally consider a novel Thompson Sampling algorithm for families $\cF$ satisfying a centered $h$-moment condition, as a computationally efficient version of MED.
We present all of our results and algorithms for the centered case and omit details for uncentered for simplicity, since the adaptation to the uncentered case is straightforward by replacing all expressions of the type $h(|x-\mu|)$ (for some $x \in \R$ and a ``target mean'' $\mu \in \R$) by $h(|x|)$ in the algorithm and the proofs.
See also Remark~\ref{remark_adaptation} in Appendix~\ref{app::lb_npts}
for this point.

\subsection{From NPTS to $h$-NPTS}

In this section we present how $h$-NPTS is derived from the existing NPTS policy.

\paragraph{Non-Parametric TS} We build on NPTS \citep{RiouHonda20}, since the $\kinf^\cF$ for bounded distributions and families based on $h$-moment conditions are similarly expressed as optimization problems. Considering some data $X_1,\dots, X_n$ and their upper bound $B$, a sampling step of NPTS returns a re-weighted mean of the form 
\begin{equation}\label{eq::dirichlet_bounded} \sum_{i=1}^{n} w_i X_i + w_{n+1}B \;, \text{ with } (w_1,\dots, w_{n+1}) \sim \cD_{n+1}\;, \end{equation} where $\cD_{n+1}$ is the Dirichlet distribution $\cD_{n+1}\coloneqq \text{Dir}(1,\dots, 1)$, which is the uniform distribution on the simplex of dimension $n+1$. To adapt this principle, we need to tackle two questions. First, we need to replace the upper bound $B$ by another ``exploration bonus'', as the role of $B$ is to ensure that the best arm has a reasonable chance to be sampled even if its first draws are bad. Then, we need to introduce the $h$-moment condition in the algorithm.

\paragraph{$h$-NPTS} We propose the $h$-NPTS algorithm, combining the structure of \TS{} (Algorithm~\ref{alg::TS_variant}) and a sampling step inspired by NPTS. At each time step, we first build a set $\cA$ of candidate arms with the currently optimal arms, and arms with empirical distributions that do not yet satisfy the desired $h$-moment condition.
Thanks to our assumptions on $\cF$, this check will fail only a finite number of times in expectation. Then, ``challenger'' arms (currently suboptimal arms) are only
compared to the best empirical mean, that we denote by $\widehat \mu^\star$. As in NPTS, we draw some weights $(w_1, \dots, w_{n+1}) \sim \cD_{n+1}$. Finally, we check if
there exists an exploration bonus $x\in\mathbb{R}$ such that the ``re-weighted'' empirical distribution over $(X_1, X_2,\dots,X_n, x)$ belongs to the family $\mathcal{F}$ with expectation at least $\widehat\mu^*$, which is expressed as
\begin{equation}\label{eq::check_npts} \exists x \geq \widehat\mu^\star: \; \sum_{i=1}^{n} w_i X_i + w_{n+1} x \geq \widehat \mu^\star\;,\; \text{and } \sum_{i=1}^{n} w_i h(|X_i-\widehat\mu^\star|) + w_{n+1} (h(|x-\widehat \mu^\star|)-\gamma) \leq B\;,\end{equation}
where $\gamma>0$ is a parameter of the algorithm that slightly advantages the exploration bonus, that we introduce for technical reasons. If \eqref{eq::check_npts} holds for an arm $k$, then it is added to $\cA$. Interestingly, this condition can be checked with a closed formula, and so no optimization is needed.
To be more specific (see Appendix~\ref{app::proof_check_npts} for details), \eqref{eq::check_npts} is equivalent to 
\begin{equation}\label{eq::check_practice}
h\left(\left(\frac{1}{w_{n+1}}\sum_{i=1}^{n} w_i (\wh \mu^\star - X_i)\right)^+\right) \leq B + \gamma + \frac{1}{w_{n+1}} \sum_{i=1}^{n} w_i (B-h(|X_i - \wh \mu^\star|)) \;,
\end{equation}
where $(x)^+=\max\{x, 0\}$. This check is relatively easy to implement, which is another advantage of the new algorithm structure that we introduced. On the other hand, working with \eqref{eq::check_npts} will be more convenient in the theoretical analysis. We provide the detailed $h$-NPTS in Algorithm~\ref{alg::h-NPTS}, that is arguably much simpler to implement than other optimal algorithms for this setting.

\begin{algorithm}[t]
	\SetKwInput{KwData}{Input}
	\SetKwComment{Comment}{\color{purple} $\triangleright$\ }{}
	\SetKwComment{Titleblock}{// }{}
	\KwData{$K$ arms, function $h$, constant $B$, parameter $\gamma>0$}
\For{$t\leq K$}{
For $k=t$: draw arm $k$, collect reward $X_k$, set $\cX_k=\{X_k\}$ and $N_k=1$ \\ 
Initialize $\wh \mu_k = X_k$, $\wh F_k: x \in \R \mapsto\ind(X_k \leq x)$. \Comment*[r]{\color{purple} \small Initialize emp.$\!$ means and cdf}
}
	\For{$t> K$}{
		Set $\cA= \{k \in [K]: \; \wh \mu_k = \max_{j \in [K]}\wh \mu_j\}$ \Comment*[r]{\color{purple} \small Candidate arm(s) = emp. $\!\!\!$ best arm(s)}
		\For{$k \notin \cA$}{
			\textbf{If} $\wh F_k \notin \cF$: add $k$ to $\cA$ \Comment*[r]{\color{purple} \small Check if empirical distribution is in $\cF$} 
			\vspace{1mm}
			\textbf{Else:} draw $w \sim \cD_{N_k+1}$, add $k$ to $\cA$ if ~\eqref{eq::check_practice} holds \Comment*[r]{\color{purple} \small NPTS step}
			\vspace{1mm} 
		}
		Choose $k_t$ uniformly at random in $\cA$; \\Collect the reward $X_t$ from arm $k_t$, add $X_t$ to $\cX_{k_t}$, set $N_{k_t}=N_{k_t}+1$; \\ 
  Update $\wh F_{k_t}: x \in \R \mapsto \frac{1}{N_{k_t}}\sum_{r \in \cX_{k_t}}\ind(r\leq x)$ and $\wh \mu_{k_t} = \frac{1}{N_{k_t}}((N_{k_t}-1) \wh \mu_{k_t} + X_t)$.}
	\SetKwInput{KwResult}{Return}
	\caption{Non Parametric Thompson Sampling for $h$-moment conditions ($h$-NPTS)}
	\label{alg::h-NPTS}
\end{algorithm}


    

\subsection{Analysis} 

Since we already proved Assumption~\ref{ass::sufficient_prop} in Section~\ref{subsec::hcond}, in order to derive a regret upper bound thanks to Theorem~\ref{th::main_result}, it remains to prove Assumption~\ref{ass::relaxed_sp} on the sampling probabilities for $h$-NPTS. 
Thanks to the \TS{} framework, as discussed in Section~\ref{subsec::presentation_TS}, this can be done by upper and lower bounding the \emph{Boundary Crossing Probabilities} (BCP) of $h$-NPTS. To do that, for any parameters $\gamma>0, \eta>0$ we define the following function
\[\Lambda_{\eta, \gamma}^\star(F,\mu) = \max_{(\alpha, \beta) \in \cR_2^{\eta, \gamma}} \bE_F\left[\log(1-\alpha(X-\mu)-\beta(B-h(|X-\mu|))\right] \;, \]
where $\cR_2^{\eta, \gamma} = \{(\alpha, \beta)\in (\R^+)^2 \;: \; \forall x\in \R\;,\; 1-\alpha(x+\eta-\mu)-\beta(B+\gamma-h(|x-\mu|)\geq 0\}$. When $\eta=0$ we use the notation $\cR^\gamma$ and $\Lambda_\gamma^\star$. We prove in Appendix~\ref{app::proof_th_npts} that the sampling probabilities of $h$-NPTS satisfy Assumption~\ref{ass::relaxed_sp}, with $D_\pi: (F, \mu)\mapsto \Lambda_{\gamma}^\star(F, \mu) \ind(F\in \cF)$, where $\Lambda_\gamma^\star$ matches $\kinf^\cF$ for
$\gamma\downarrow 0$.
The proof is based on the following upper and lower bounds on the BCP.

\begin{lemma}[Bounds on the BCP]\label{lem::bounds_bcp} Let $X_1,\dots, X_n$ be i.i.d. observations with empirical cdf $F_n$, and consider any threshold $\mu^\star>\bE_{F_n}[X]$.
Using the notation $Z_i=h(|X_i-\mu|)$ for $i \in [n]$ and $Z_{n+1}=h(|X_{n+1}-\mu|)-\gamma$, let the Boundary Crossing Probability be
\[\text{\emph{[BCP]}} \coloneqq \bP\left(\exists X_{n+1} \geq \mu^\star: \; \sum_{i=1}^{n+1} w_i X_i\geq \mu^\star, \sum_{i=1}^{n+1} w_i Z_i\leq B \right) \;.\]
Firstly, there exists a mapping $C$ such that for any $\eta>0$ and any $n\in \N$
	\begin{equation}\label{eq::ub_bcp_hnpts}\text{\emph{[BCP]}}\leq C(F_n, \eta) \times e^{-n \Lambda_{\eta, \gamma}^\star(F_n, \mu^\star)}\;,\end{equation} 
where $C$ is continuous w.r.t. the Wasserstein metric in $F_n$, is continuous in $\eta$, and scales in $\eta^{-2}$. Furthermore, for any $\delta \geq 0$, there exists two constants $n_{\delta, \gamma} \in \N$ and $c_\delta>0$ such that 
\[\text{\emph{[BCP]}}\geq \left\{
\begin{array}{ll}
c_\delta \times e^{-n\left(\Lambda_{\gamma}^\star(F_n, \mu^\star)+\delta\right)} & \mbox{if } n \geq n_{\delta, \gamma}\;, \\
e^{-n \log(nC_{B,\mu})} & \hspace{-12mm} \mbox{if } n \leq n_{\delta, \gamma}, \; \text{ with } C_{B,\mu}= \max\left\{\frac{3h^{-1}(B)-\mu}{h^{-1}(B)-\mu}  , \frac{B}{B-h\left(\frac{ h^{-1}(B)-\mu}{2}\right)} \right\}\;.
\end{array}
\right.\]
\end{lemma}

The detailed proofs and constants for the upper and lower bounds are presented respectively in Appendices~\ref{app::ub_npts} and \ref{app::lb_npts}. These two results are novel, and our proof techniques are of independent interest. Indeed, some recent works have focused on proving similar anti-concentration bounds for weighted sums, with weights drawn from a Dirichlet distribution \citep{RiouHonda20, baudry21DS, tiapkin2022}. Unfortunately, their results are not enough for our setting since they consider the BCP associated with the weighted sum \eqref{eq::dirichlet_bounded} (bounded distributions),
while we need to consider the joint distribution of the weighted sum and the moment condition as given in \eqref{eq::check_npts}.
Instead,
as an intermediate step in our analysis we prove a novel lower bound for this case in Appendix~\ref{app::lb_bcp_bounded} (Lemma~\ref{lem::novel_lb_bcp_npts}). This bound can thus compare with the results provided in these works. We can for instance remark that the lower bound of \cite{tiapkin2022} is tighter than ours, but they require to increase the exploration bonus in \eqref{eq::dirichlet_bounded} from $B$ to $2B$ and its parameter in the Dirichlet distribution from $1$ to $\Omega(\log(n))$, which is not suitable to our setting.


 
\paragraph{Proof ideas} We upper bound the BCP by partitioning the possible values for $X_{n+1}$ as $(x_j)_{j \in \N} = (\mu+j\eta)_{j \in \N}$ and using a union bound. Then, we upper bound each term using the Chernoff method, and denoting by $(\alpha_n^\star, \beta_n^\star)$ the optimizers of $\Lambda_{\eta, \gamma}^\star(F_n,\mu))$ we obtain that 
\[\text{[BCP]} \leq \sum_{j=1}^{+\infty}\frac{e^{-n \Lambda_{\eta, \gamma}^\star(F_n, \mu^\star)}}{1-\alpha_n^\star(x_j+\eta-\mu)- \beta_n^\star(B+\gamma-h(|x_j-\mu|))} \;. \]
At this step, we can remark that once $x_j$ is large enough the denominator is decreasing and dominated by the term $h(|x_j-\mu|)$, that satisfies $h(|x_j-\mu|)\geq |x_j-\mu|^{1+\eta}$ for some $\eta>0$ when $x_j$ is large enough by definition of $\cF$. Hence, only a finite number of terms are actually significant and the sum converges. 

For the lower bound, we select a proper value for $X_{n+1}$, use exponential tilting to get the term in $e^{-n \Lambda_{\alpha, \beta}(F_n, \mu^\star)}$ from a change of distribution, and then analyze the probability that \eqref{eq::check_npts} hold under the new (more favorable) distribution of the weights. The parameter $\gamma>0$ ensures that even in the least favorable cases this probability is larger than a constant when $n$ is large enough, which concludes the proof. The result may still hold for $\gamma=0$, but proving it would require additional technicality. Finally, for $n\leq n_\delta$ our lower bound is simply $\bP\left(w_{n+1}\geq 1-\frac{1}{C_{B,\mu}n}\right)$. 

We now state the main result of this section, which is a regret upper bound for $h$-NPTS.

\begin{theorem}[Logarithmic regret for $h$-NPTS] \label{th::npts}
	If $\cF$ is defined by a centered $h$-moment condition (Eq.~\eqref{eq::centered_main}), then for any $c>0$, $\gamma>0$ and for any sub-optimal arm $k$, $h$-NPTS satisfies
 \[\bE[N_k(T)] \leq \frac{1+c}{\Lambda_\gamma^\star(F_k, \mu^\star)} \log(T) + o_c(\log(T)) \;. \]
\end{theorem}
We detail the proof in Appendix~\ref{app::proof_th_npts}, where we show that Lemma~\ref{lem::bounds_bcp} gives sufficient bounds on the sampling probability of each arm in $h$-NPTS, and then use Corollary~\ref{th::main_result}. As $\gamma>0$, the algorithm may not be asymptotically optimal, since $\cR_2^\gamma \subset \cR_2^0$. However, this is the case only for problems for which arm $k$ is ``far'' from being optimal, so the constant before the logarithm is small: if for some distribution $F$ and $\mu \in \R$ the optimizers in $\cR_2^0$ belong to $\cR_2^\gamma$ then $\Lambda_\gamma^\star(F, \mu) = \Lambda_0^\star(F, \mu)$. 

Theorem~\ref{th::npts} completes the picture that $h$-NPTS is an easy to implement, computationally efficient, and theoretically sound alternative to MED for distributions satisfying a centered $h$-moment condition. In our opinion, our methodology to derive and analyze this algorithm shows another interest of the general insights provided in this paper: we first tried to find the best function $D_\pi$ for which MED would work for this family, and then provided a TS algorithm in order to approach this MED strategy with sampling.

\begin{remark}[Further generalization of $h$-NPTS] As stated above, the algorithm and its analysis are first easily translated to the uncentered case. Furthermore, we think that the same proof techniques as for Lemma~\ref{lem::bounds_bcp} may be used with only minor changes if several 
conditions with functions $h_1,\dots, h_m$ and constants $B_1,\dots, B_m$ were considered, for $m\in \N$.
\end{remark}
\section{Numerical Experiments}\label{sec::experiments}

In this section, we present experiments on synthetic data to illustrate the theoretical results developed in this paper. Our focus is not to benchmark our approach against all possible policies but to provide a detailed comparison of MED, TS, and its variant TS$^\star$, in several settings. For broader experimental evaluations of bandit algorithms, including UCB \citep{auer2002finite}, KL-UCB \citep{KL_UCB}, IMED \citep{Honda15IMED}, as well as bootstrapping and sub-sampling methods, we refer to works such as \citet{chapelle11, Giro, SDA, baudry21DS}. These studies consistently demonstrate that Thompson Sampling and IMED (the deterministic version of MED) rank among the top-performing policies in $K$-armed bandit problems. Given this, we believe that focusing our comparison on TS provides sufficient context to showcase our findings. The experiments presented here (regret and computation time) were conducted using a Jupyter notebook, which, along with the code, is provided for reproducibility\footnote{\url{https://github.com/DBaudry/unified_MED_TS}}.

\paragraph{Experimental Setup} We compare the empirical regret of TS, TS$^\star$, and MED across various simulations. The experiments cover standard families of distributions, including Bernoulli and Gaussian (with known variances) for the SPEF instance of the policies (Section~\ref{subsec::spef}), as well as Beta distributions for testing the non-parametric bounded variant (Section~\ref{subsec::bounded}). All implementations were done in Python. Each bandit instance was simulated 500 times with a time horizon of $T=5000$. For each policy, we report the empirical regret over the 500 simulations, along with the $10$th and $90$th percentiles of the results. 

\paragraph{Bernoulli experiments} We conducted four experiments with Bernoulli rewards. In the first experiment (\emph{Bernoulli 1}), we consider a simple $2$-armed instance with means $\mu_1=0.6$ and $\mu_2=0.5$. The second experiment (\emph{Bernoulli 2}) increases the number of arms and narrows the gap: $\mu_1=0.55$,  with $9$ other arms having a mean of $0.5$. The third experiment (\emph{Bernoulli 3}) explores distributions with relatively low rewards, with $6$ arms with means $(\mu_k)_{k \in [6]} =(0.1,\, 0.05,\, 0.15,\, 0.1,\, 0.03,\, 0.06)$. Finally, in the fourth experiment (\emph{Bernoulli 4}) the expected rewards are relatively high, with $K=7$ arms and $(\mu_k)_{k \in [7]}=(0.85,\, 0.85,\, 0.6,\, 0.7,\, 0.9,\, 0.92,\, 0.95)$. The results are summarized in Figure~\ref{fig::exp_ber}. Across all experiments, we observe that TS and TS$^\star$ perform nearly identically, and outperform MED in most cases.

\begin{figure}[htbp]
    \centering
    \includegraphics[width=0.45\linewidth]{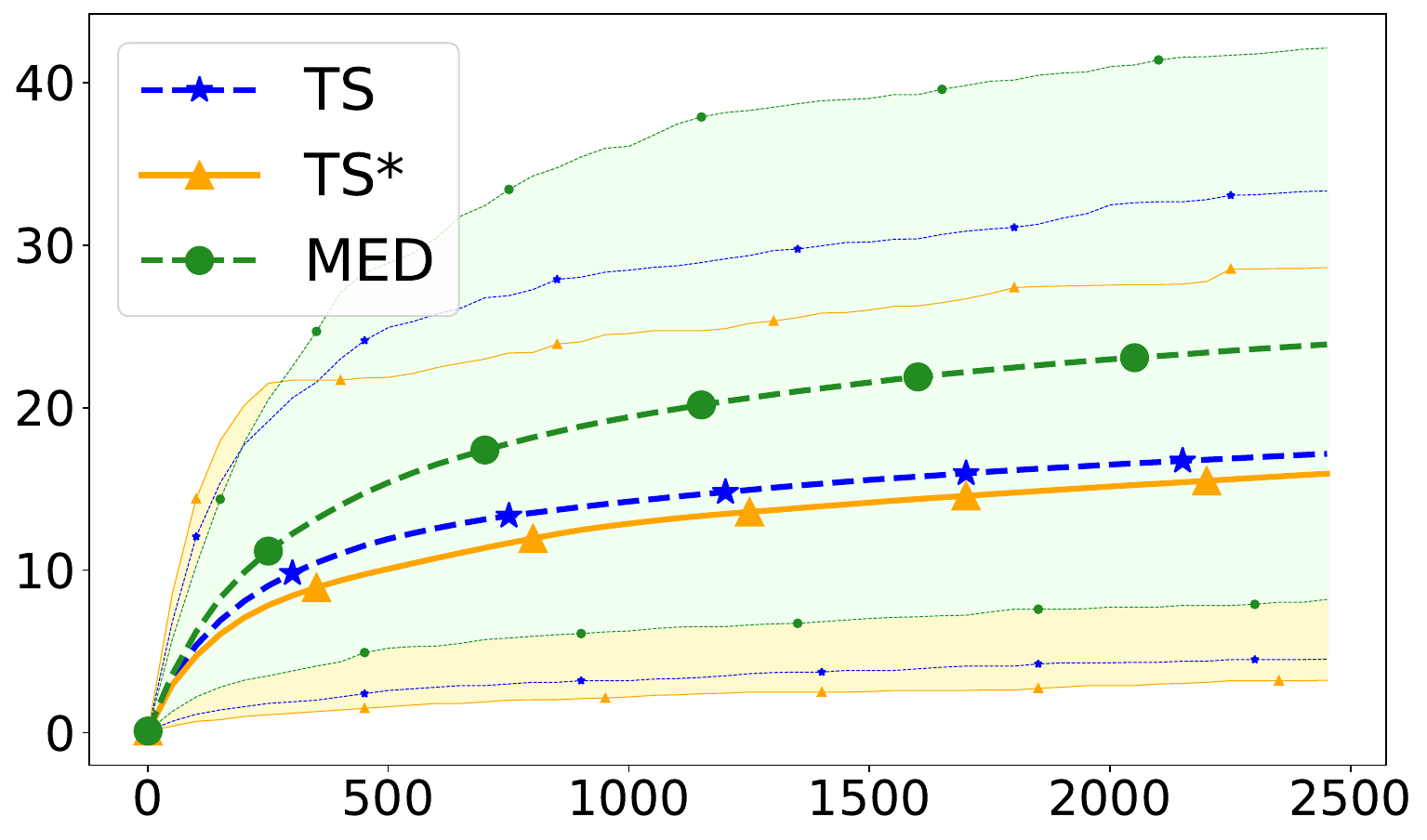}
    \includegraphics[width=0.45\linewidth]{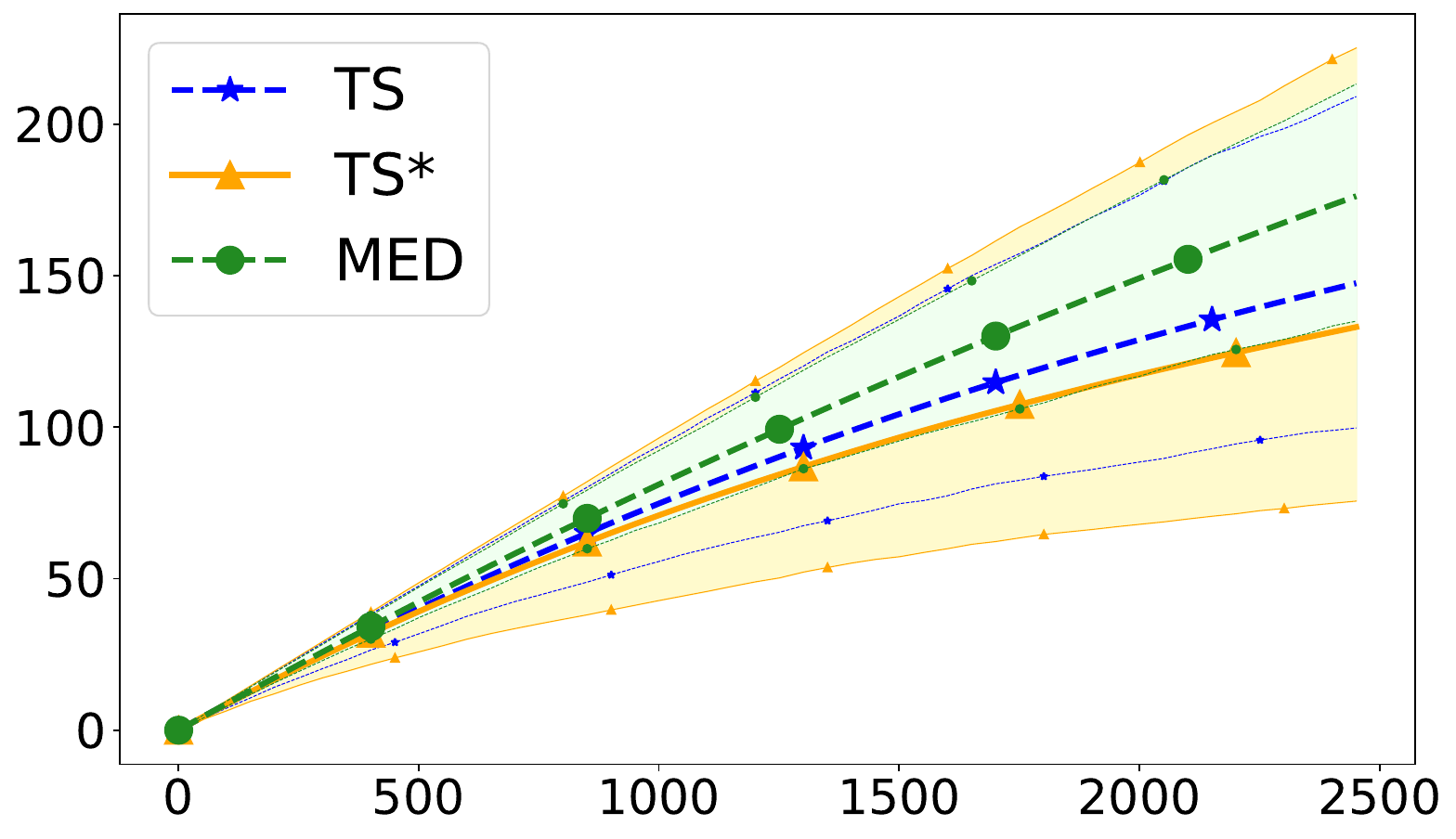}
    \includegraphics[width=0.45\linewidth]{figures/Ber1.pdf}
    \includegraphics[width=0.45\linewidth]{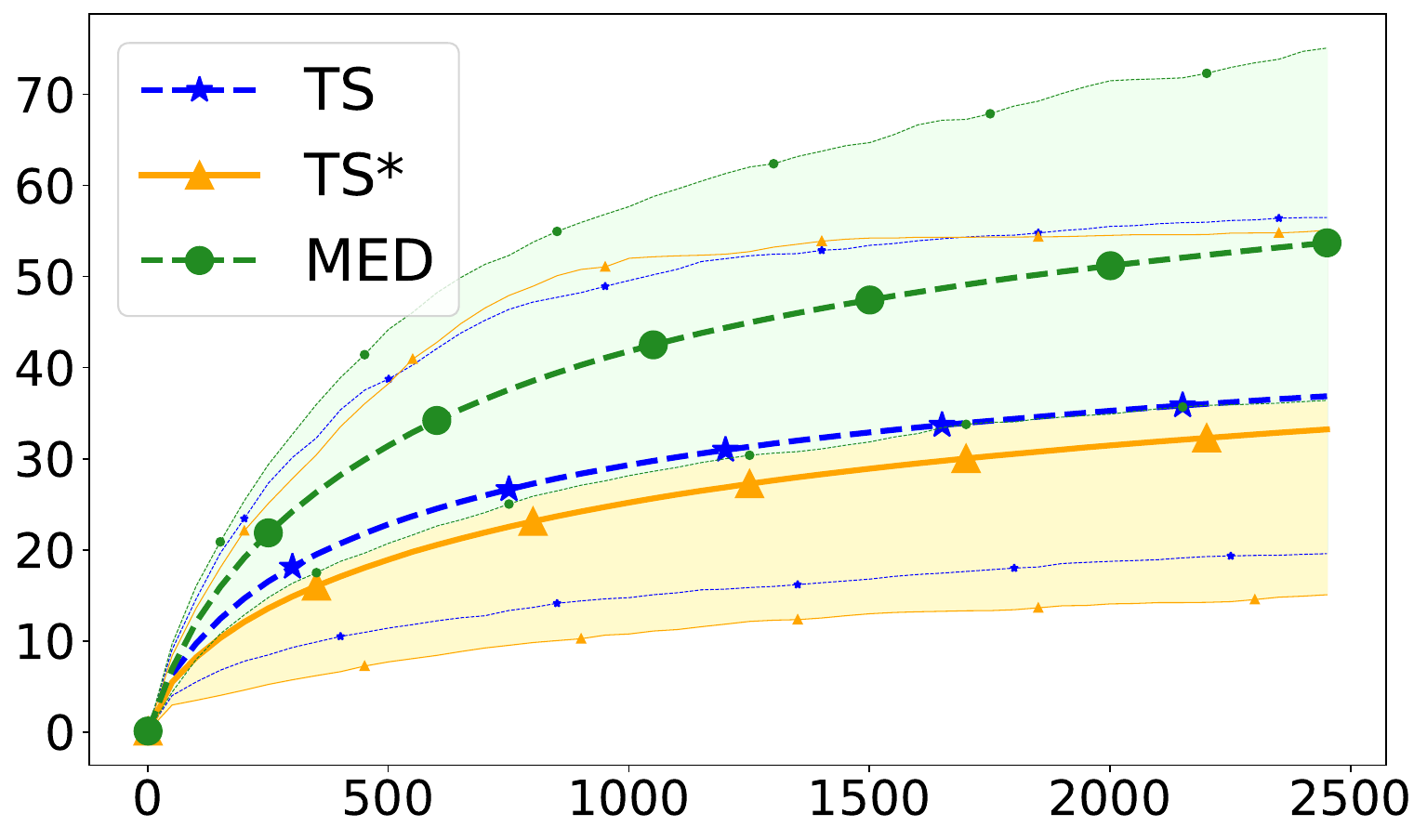}
    \caption{Average regret and $10$--$90\%$ percentiles as a function of $T$ on $500$ runs, for experiments \emph{Bernoulli 1} (top left), \emph{Bernoulli 2} (top right), \emph{Bernoulli 3} (bottom left) and \emph{Bernoulli 4} (bottom right).}
    \label{fig::exp_Gauss}
\end{figure}

\paragraph{Gaussian Distributions with Known Variance} We now present experiments using Gaussian rewards with a fixed variance of $\sigma^2=1$. 
In the first experiment (\emph{Gauss 1}), we set $\mu_1=0.3$ and $\mu_2=0$. In the second (\emph{Gauss 2}), we define $\mu_k=\frac{k-1}{K}$ for $k\in [6]$. The third experiment uses $\mu_1=0.2$ and $\mu_2=\dots=\mu_{10}=0$, while in the fourth (\emph{Gauss 4}) we set $\mu_1=0.1$ and $\mu_2=\dots=\mu_{20}=0$. The results, shown in Figure~\ref{fig::exp_Gauss}, reflect similar trends to the Bernoulli experiments: TS and TS$^\star$ perform almost identically, with both slightly outperforming MED.

\begin{figure}[htbp]
    \centering
    \includegraphics[width=0.45\linewidth]{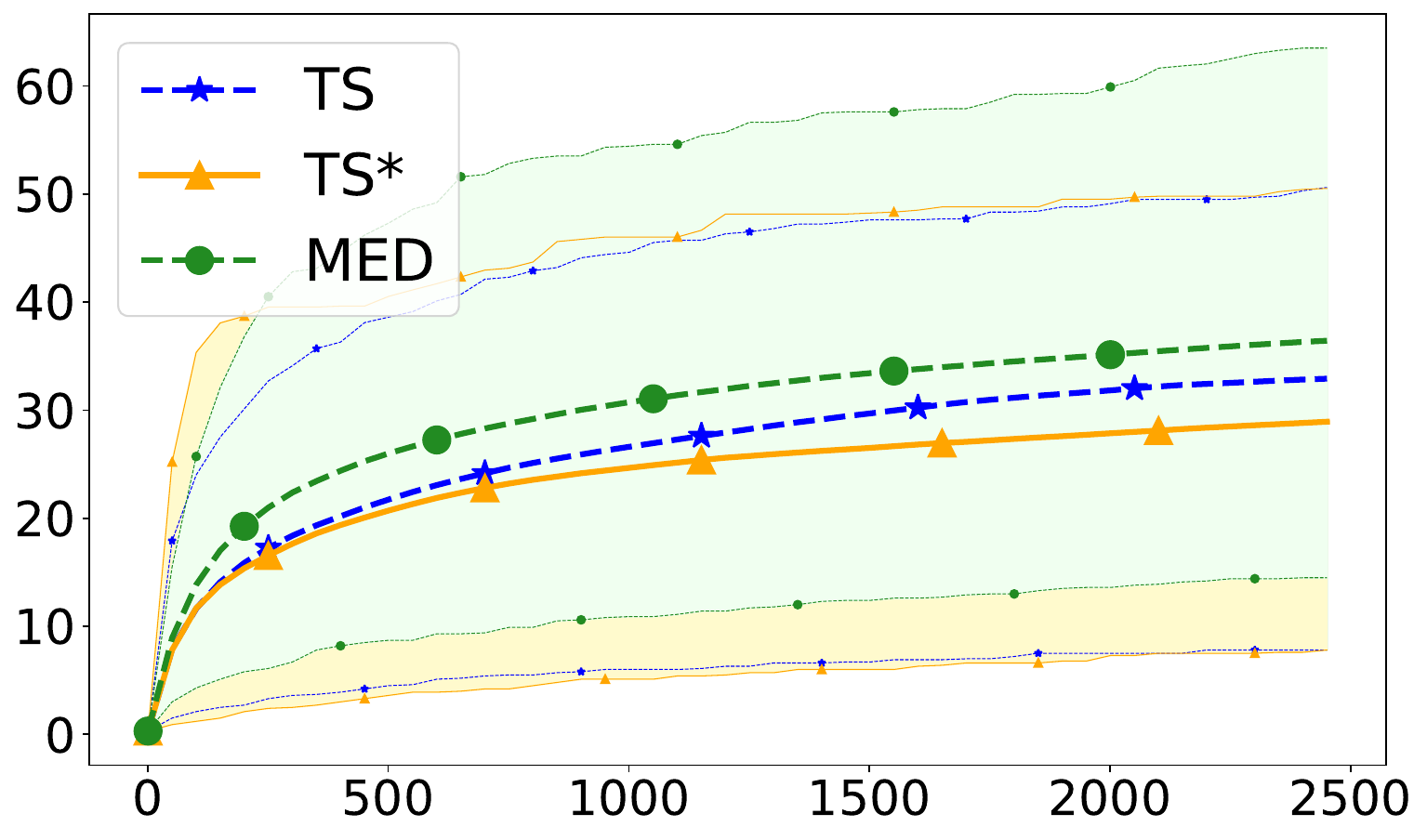}
    \includegraphics[width=0.45\linewidth]{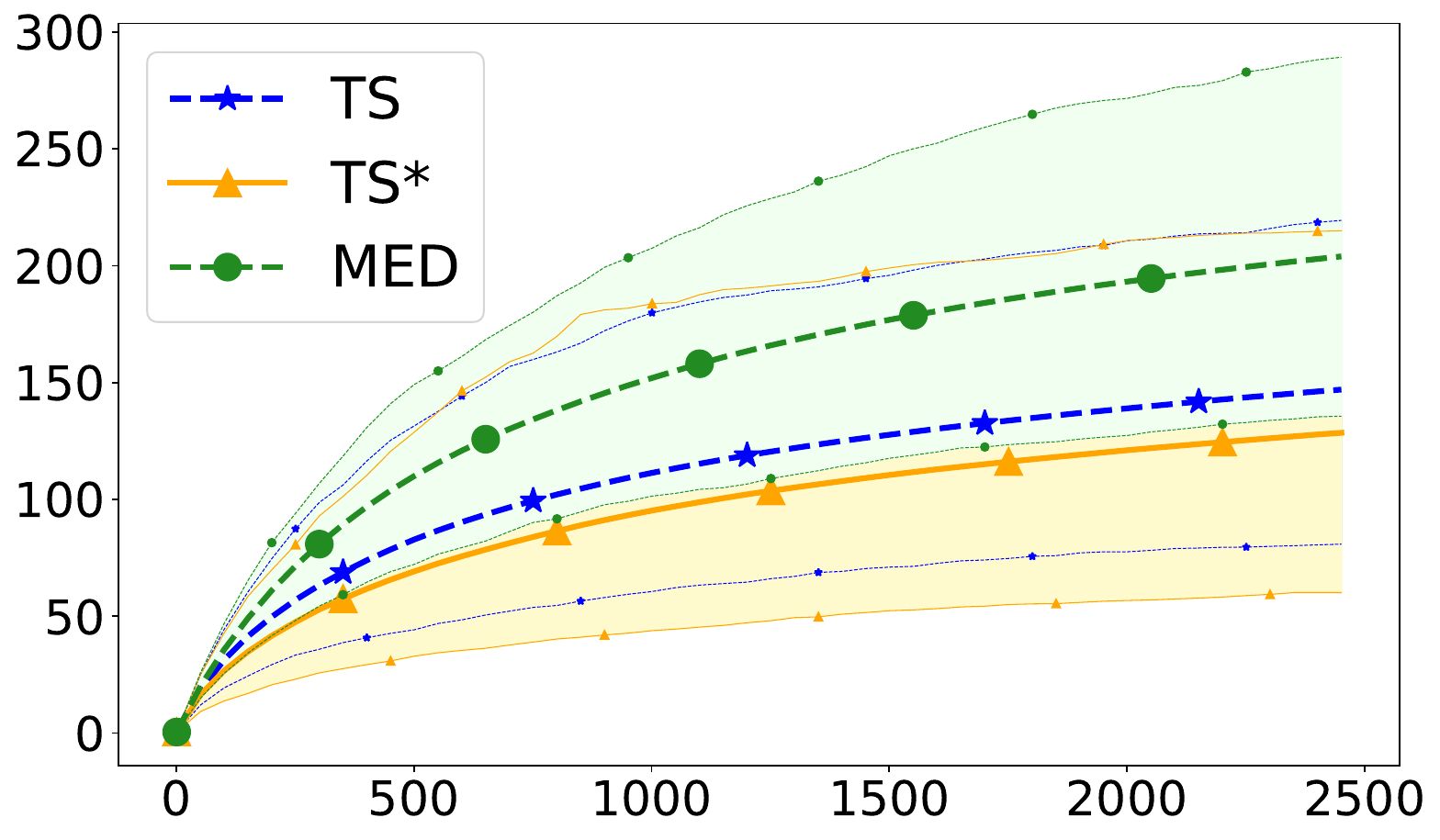}
    \includegraphics[width=0.45\linewidth]{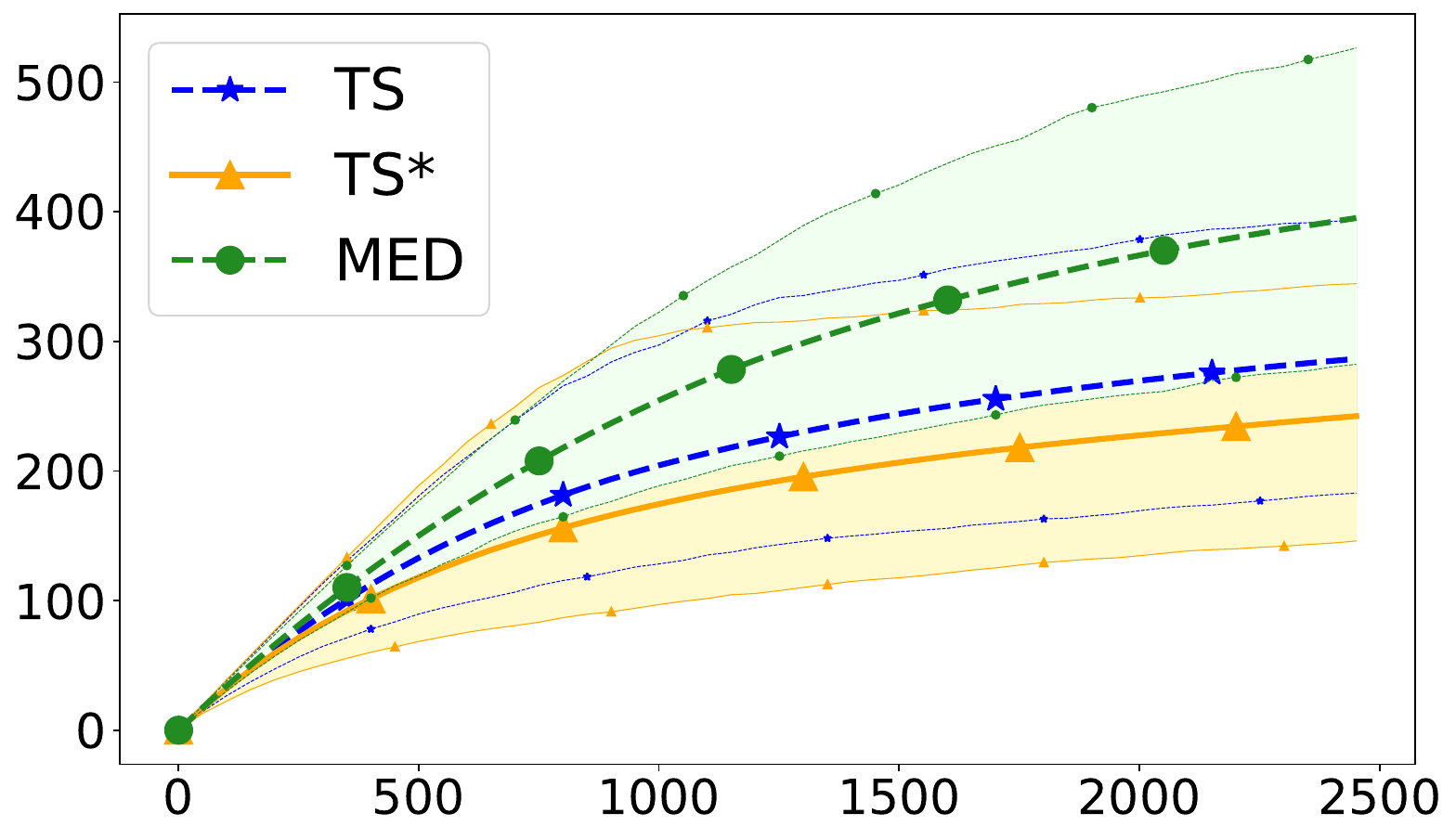}
    \includegraphics[width=0.45\linewidth]{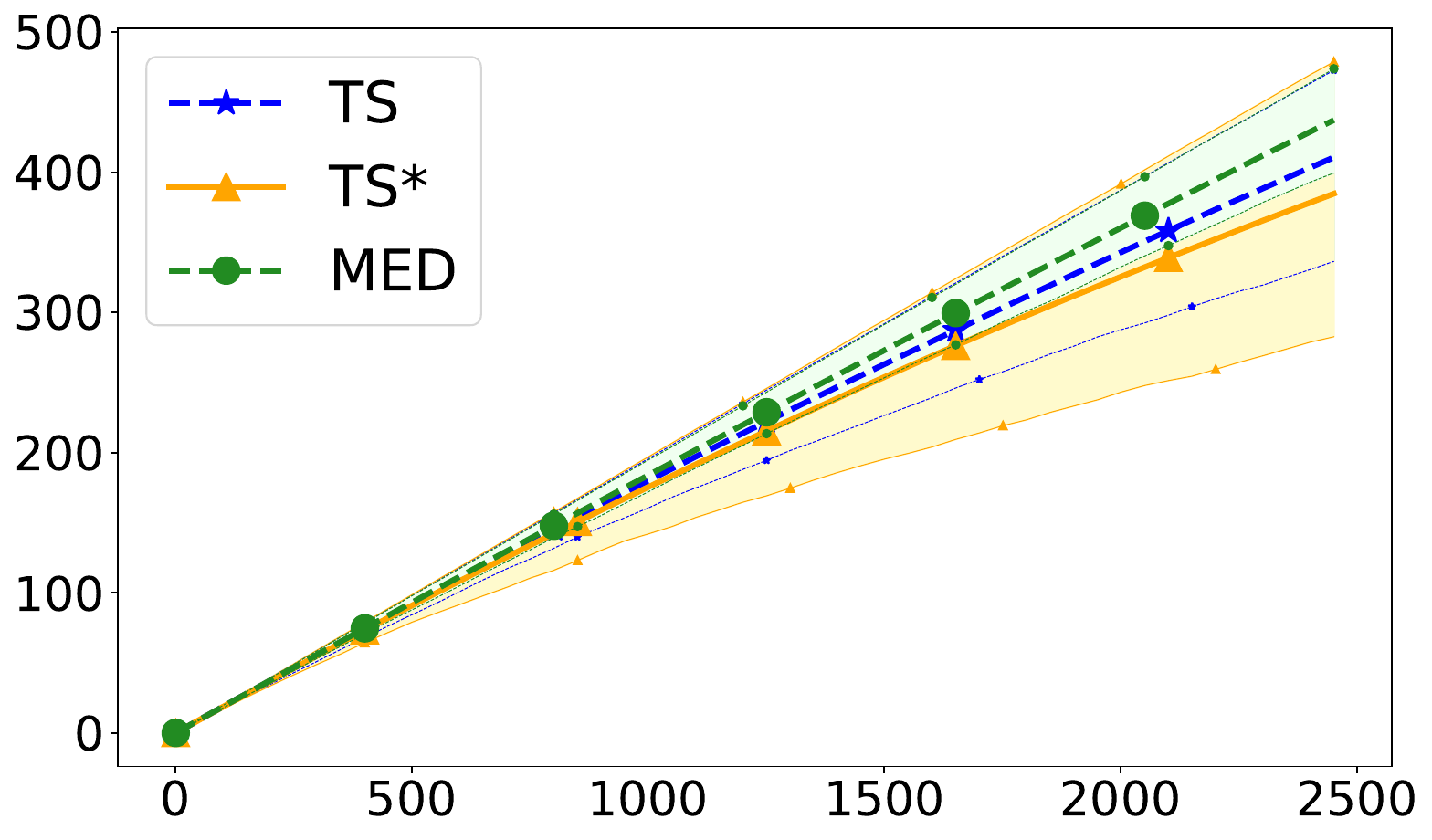}
    \caption{Average regret and $10$--$90\%$ percentiles as a function of $T$ on $500$ runs, for experiments \emph{Gauss 1} (top left), \emph{Gauss 2} (top right), \emph{Gauss 3} (bottom left) and \emph{Gauss 4} (bottom right).}
    \label{fig::exp_ber}
\end{figure}

\paragraph{Beta Distributions} We now conduct experiments using Beta distributions to evaluate the non-parametric policies designed for bounded distributions. Recall that a Beta distribution is characterized by two parameters, $(a,b) \in (\R^{+})^2$, which define a density function $f(x) \propto x^{a-1}(1-x)^{b-1}$ for $x \in [0,1]$. The expectation is given by $\frac{a}{a+b}$, and the variance by $\frac{ab}{(a+b)^2(a+b+1)}$. By adjusting these parameters, we can create distributions with various shapes: they can resemble Gaussian distributions (when both $a$ and $b$ are large, with a mean around $0.5$), exponential distributions (when $b \gg a$), uniform distributions (when $a \approx b \approx 1$), or even Bernoulli distributions (when $a$ and $b$ are very small). We illustrate two such examples in Figure~\ref{fig::beta_shapes} below. Since we have already conducted experiments with Bernoulli distributions, we focus here on the case where $a \geq 1$ and $b \geq 1$. 
\begin{figure}[hbtp]
    \centering
\includegraphics[width=0.45\linewidth]{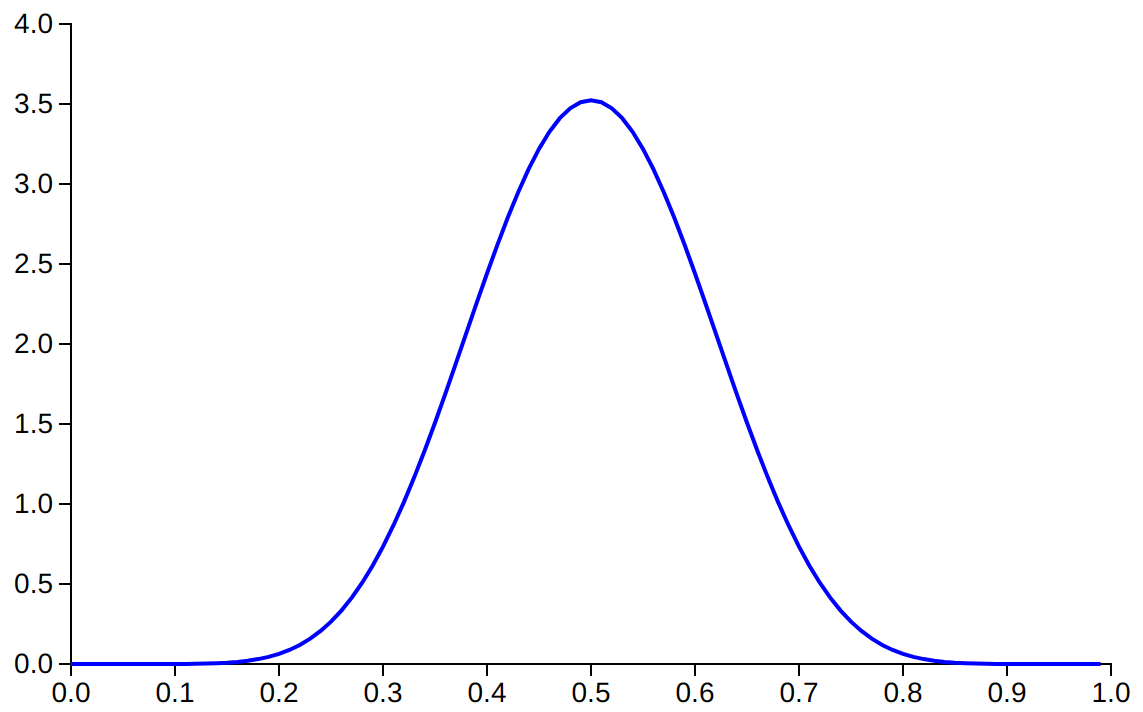}
    \includegraphics[width=0.45\linewidth]{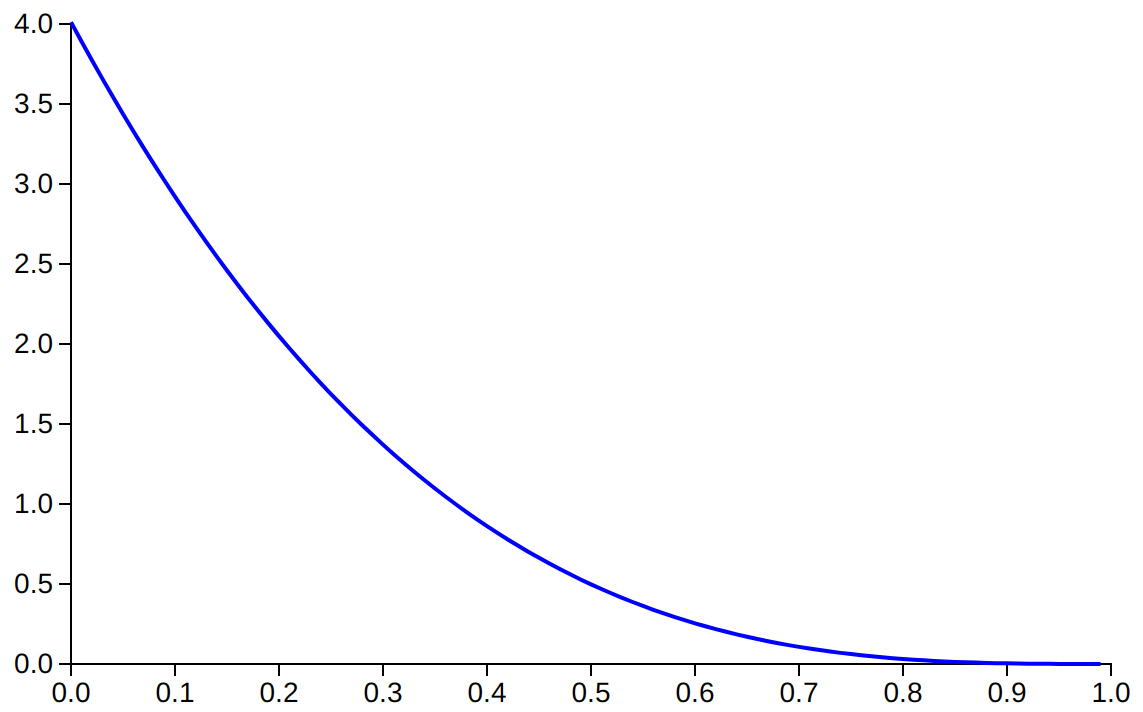}
    \caption{Example of Beta distribution with Gaussian shape (Left, $a=b=10$), \\ and Exponential shape (Right, $a=1,b=4$).}
    \label{fig::beta_shapes}
\end{figure}

In the first two experiments (\emph{Beta 1} and \emph{Beta 2}), we consider distributions close to (or equal to) the uniform distribution, which corresponds to $a=b=1$. Specifically, we set the parameters as $(a_k, b_k)_{k \in [K]} = ((1, 1), (2, 3))$ in \emph{Beta 1}, and $(a_1, b_1) = (1, 1)$ and $(a_2, b_2) = (1, 1.5)$ in \emph{Beta 2}. In the third experiment (\emph{Beta 3}), we increase the parameter values to produce distributions with ``Gaussian-like'' shapes, defining $((a_k, b_k))_{k \in [6]} = ((10, 10), (10, 9), (10, 11), (8, 10), (11, 10), (11, 11))$. Finally, in \emph{Beta 4}, we examine distributions with ``exponential'' shapes by setting $((a_k, b_k))_{k \in [5]} = ((1, 4), (1, 4.5),\allowbreak (1, 5), (1, 5.5), (1, 6))$. The results are presented in Figure~\ref{fig::exp_beta}. Here, MED and NPTS show comparable performance, with NPTS$^\star$ consistently outperforming the others. However, the performance gap is relatively small, suggesting that all the policies perform similarly on this problem.

\begin{figure}[htbp]
    \centering
    \includegraphics[width=0.45\linewidth]{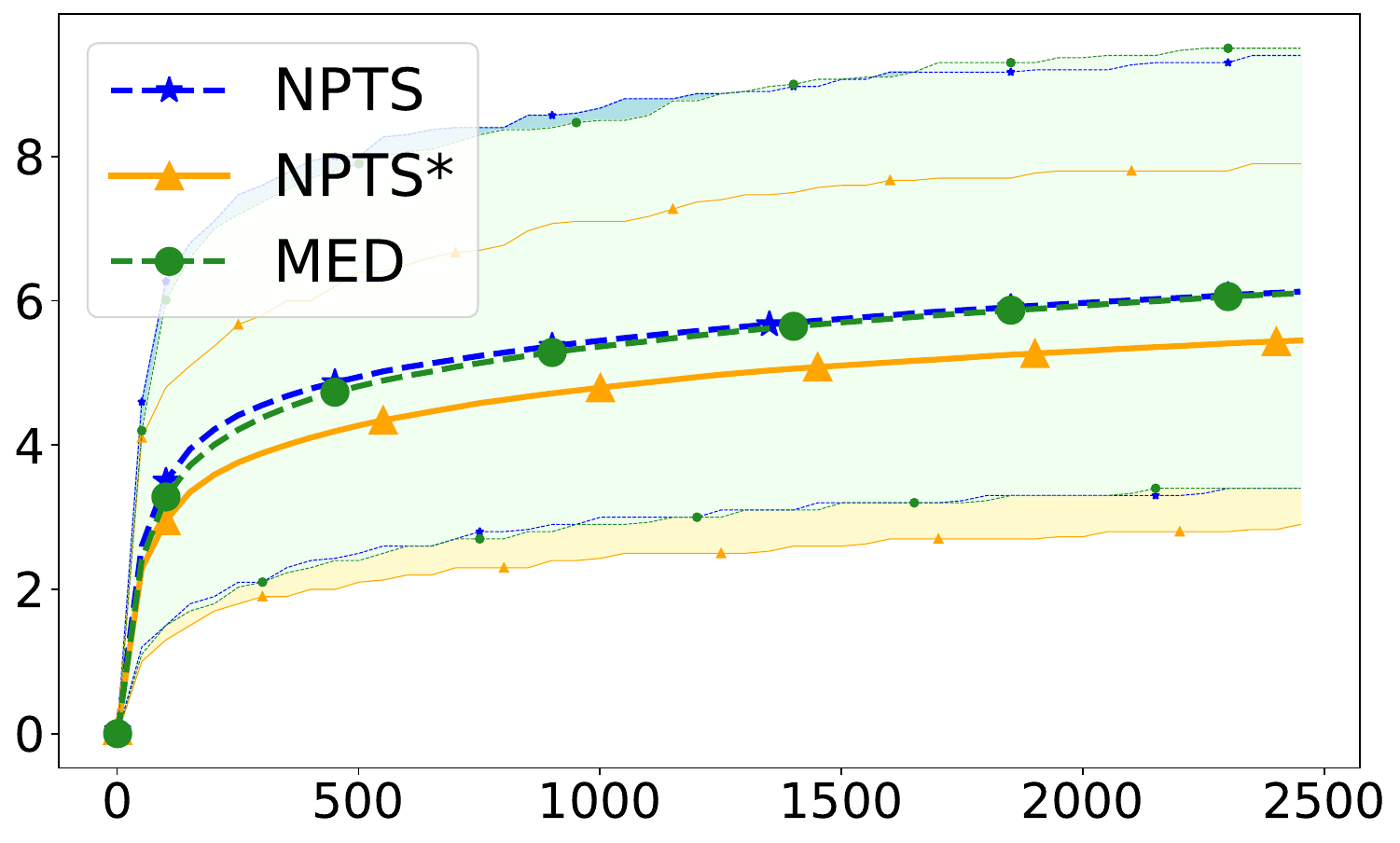}
    \includegraphics[width=0.45\linewidth]{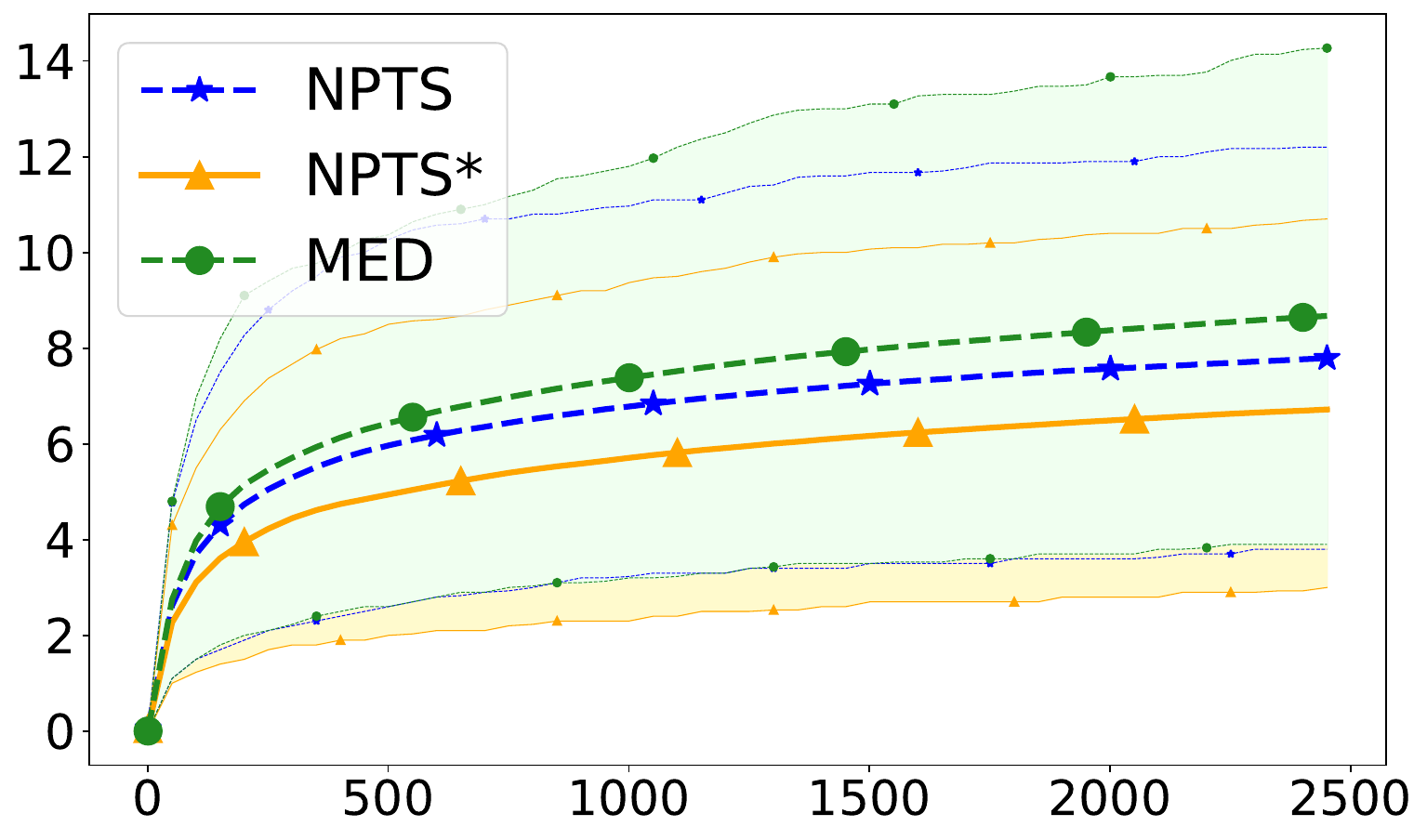}
    \includegraphics[width=0.45\linewidth]{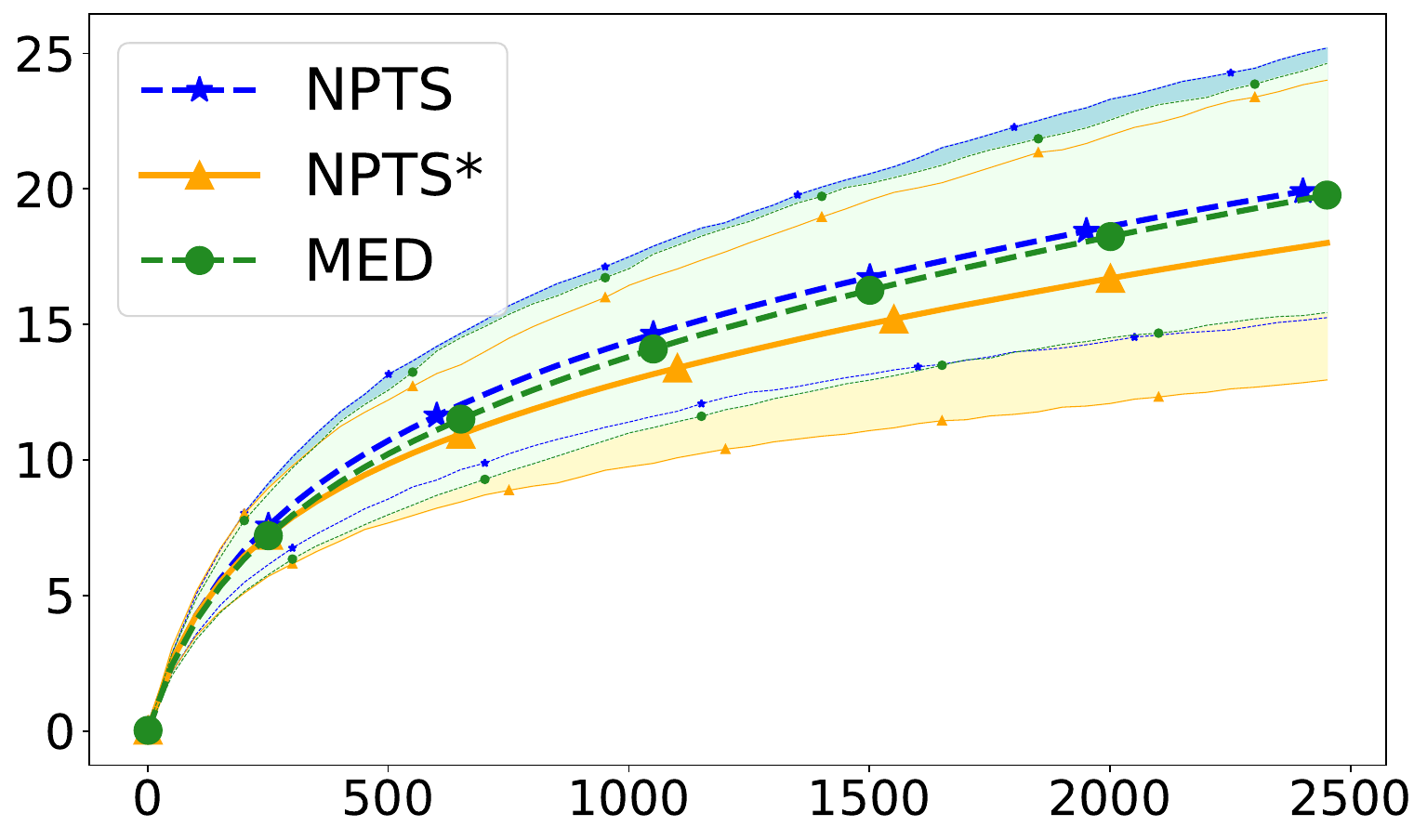}
    \includegraphics[width=0.45\linewidth]{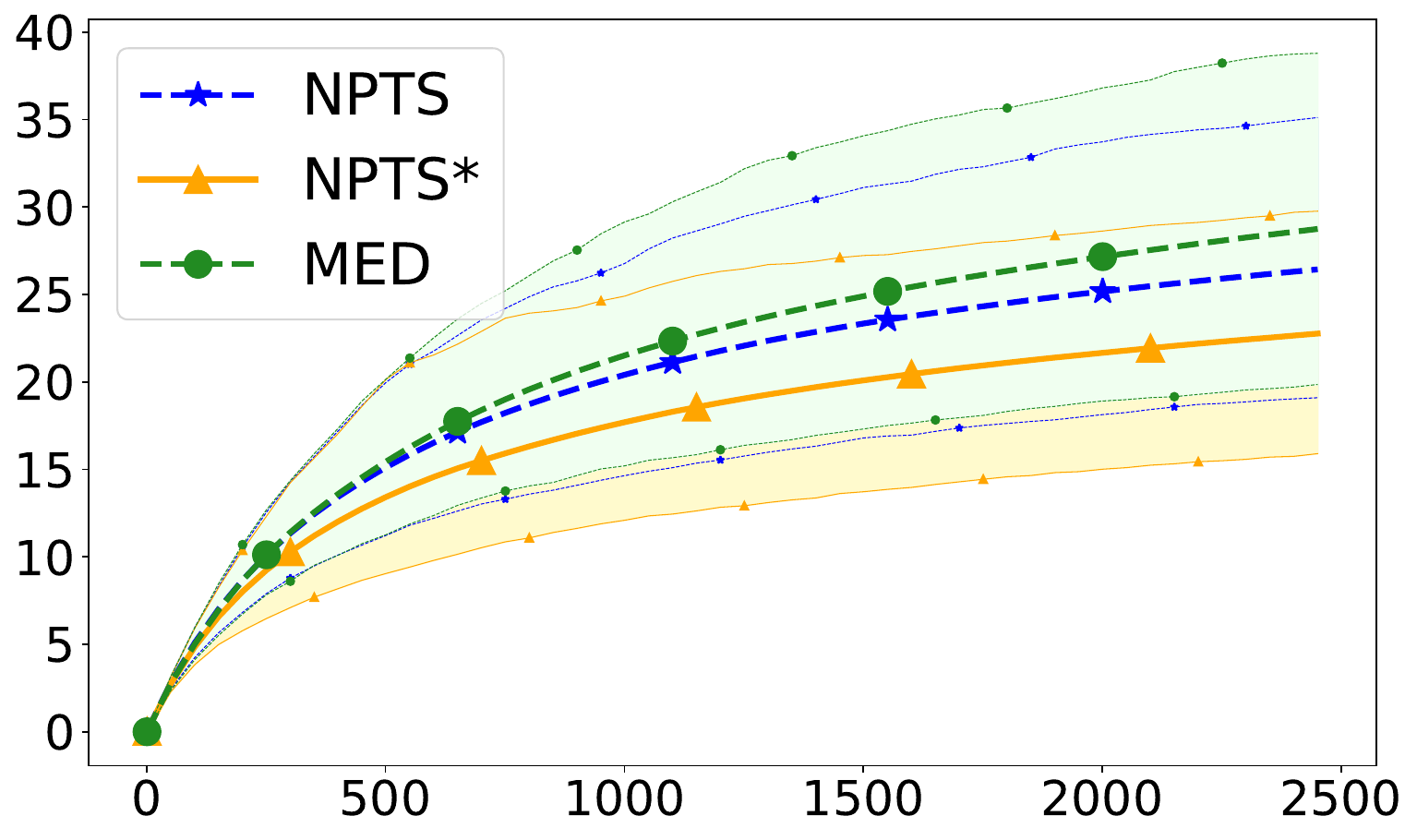}
    \caption{Average regret and $10$--$90\%$ percentiles as a function of $T$ on $500$ runs, for experiments \emph{Beta 1} (top left), \emph{Beta 2} (top right), \emph{Beta 3} (bottom left) and \emph{Beta 4} (bottom right).}
    \label{fig::exp_beta}
\end{figure}

\paragraph{Comparison of Computation Time for NPTS and NPTS$^\star$} We now provide an empirical evidence supporting the claim that TS$^\star$ can be more computationally efficient than TS, using the case of NPTS. In Tables~\ref{tab:comp_time_1}--\ref{tab:comp_time_4}, we report the average computation time\footnote{we used the \texttt{\%timeit} command in Python, with default settings, where the number of trajectories used to compute the estimate is automatically determined.} required to run a single trajectory of both NPTS and NPTS$^\star$ across various time horizons, ranging from $T=100$ to $T=5000$. These experiments correspond to the four Beta distribution setups introduced earlier. 

We observe that while the computation times for the two algorithms are similar at smaller horizons, a significant difference emerges as the horizon increases, particularly when the best empirical arm accumulates significantly more observations than the others. By the time horizon reaches $T=5000$, the ratio of computation times between NPTS and NPTS$^\star$ is approximately 6 for the simpler experiments (\emph{Beta 1} and \emph{Beta 2}), 2.5 for \emph{Beta 4}, and 1.74 for \emph{Beta 3}. Furthermore, when we tested $T=5 \times 10^4$ for \emph{Beta 3}, where the ratio increased to 7, with the algorithms taking 44 seconds and 6.3 seconds per run, respectively. This further illustrates that the computation gain of NPTS$^\star$ becomes significant 
as an ``asymptotic'' regime begins and the algorithm starts to exploit a seemingly best arm.

\begin{table}[tbp]
  \begin{subfigure}[b]{0.48\columnwidth}
    \centering
    \caption{\emph{Beta 1}}
    \label{tab:comp_time_1}
    \begin{tabular}{l|ccccc}
        Alg./$T$ & $100$ & $500$ & $1000$ & $2000$& $5000$\\
        \hline
        NPTS & 1.5 & 11.2 & 31.1 & 96.5 & 496 \\
        NPTS$^\star$ & 1.4 & 7.1 & 14.6 & 29.3 & 75 \\
        ratio & 1.07 & 1.57 & 2.1 & 3.3 & 6.6 \\
    \end{tabular}
  \end{subfigure}
  \hspace{0.02\columnwidth}
  \begin{subfigure}[b]{0.48\columnwidth}
    \centering
    \caption{\emph{Beta 2}}
    \label{tab:comp_time_2}
    \begin{tabular}{l|ccccc}
        Alg./$T$ & $100$ & $500$ & $1000$ & $2000$& $5000$\\
        \hline
        NPTS & 1.5 & 11.1 & 30.9 & 95.8 & 492 \\
        NPTS$^\star$ & 1.4 & 7.2 & 14.6 & 30.1 & 75.2 \\
        ratio & 1.07 & 1.54 & 2.1 & 3.2 & 6.5 \\
    \end{tabular}
  \end{subfigure}

\vspace{5mm}
    \begin{subfigure}[b]{0.48\columnwidth}
    \centering
    \caption{\emph{Beta 3}}
    \label{tab:comp_time_3}
    \begin{tabular}{l|ccccc}
        Alg./$T$ & $100$ & $500$ & $1000$ & $2000$& $5000$\\
        \hline
        NPTS & 3.6 & 22 & 53 & 141 & 613 \\
        NPTS$^\star$ & 3.6 & 20.5 & 45 & 103 & 351 \\
        ratio & 1 & 1.07 & 1.17 & 1.36& 1.74 \\
    \end{tabular}
  \end{subfigure}
  \hspace{0.02\columnwidth}
  \begin{subfigure}[b]{0.48\columnwidth}
    \centering
    \caption{\emph{Beta 4}}
    \label{tab:comp_time_4}
    \begin{tabular}{l|ccccc}
        Alg./$T$ & $100$ & $500$ & $1000$ & $2000$& $5000$\\
        \hline
        NPTS & 3.1 & 19.2 & 47.2 & 130 & 582 \\
        NPTS$^\star$ & 3.1 & 17.5 & 38.3 & 88.1 & 234 \\
        ratio & 1 & 1.09 & 1.23 & 1.47& 2.49 \\
    \end{tabular}
  \end{subfigure}
    \caption{Comparison of average computation time (in milliseconds) of trajectories for different values of $T$.}
\end{table}

\paragraph{Showcasing the potential of $h$-NPTS} We conclude this section by demonstrating the potential of $h$-NPTS for family of distributions characterized by a centered $h$-moment condition. We consider the function $h: x \mapsto x^2$ and an upper bound $B$, defined in each context. The only assumption used by the algorithm is thus that the distributions have a variance upper bounded by $B$. To test this policy, we constructed problem instances with Gaussian mixtures. Fitting Gaussian mixtures under standard assumptions—such as upper bounds, or sub-Gaussianity with \emph{known} proxy variance—can be challenging. An upper bound on the variance, however, may be the most straightforward property for a decision-maker to infer. Furthermore, the additional properties on $\cF$ required for our analysis (Theorem~\ref{th::npts}) are easy to verify. Although we implemented MED in the code, it was excluded from the experiments due to its prohibitive computational cost (around 1000 times slower than $h$-NPTS on \emph{Gauss 1} with $T=500$). This underscores the value of $h$-NPTS as a computationally efficient alternative to MED, with comparable asymptotic theoretical performance.

In these experiments, we compare $h$-NPTS to the \emph{Robust Dirichlet Sampling} (RDS) policy from \citet{baudry21DS}, which is also based on NPTS and achieves $\cO(\log(T)\log\log(T))$ regret on any \emph{light-tailed} instance, which is the case of Gaussian mixtures. As such, RDS serves as a suitable baseline for our problem instances. For a fair comparison, we use a conservative value of $B$, unless specified otherwise, approximately $20\%$ higher than the maximal value of $\bE[h(|X-\bE[X]|)]$ across all distributions in each problem instance.

As a sanity check, we first test $h$-NPTS on the \emph{Gauss 1} experiment presented above, with exactly $B=1$, and compare it with both RDS and TS. The performance of all three algorithms, displayed in Figure~\ref{fig::exp_GM} (top left), is very similar, affirming the potential of $h$-NPTS. 
We now turn to Gaussian mixtures, where each arm is characterized by three vectors $p=(p_1,\dots, p_M)$ (probability of each Gaussian mode), $\mu=(\mu_1,\dots, \mu_M)$ and $S=(\sigma_1, \dots, \sigma_M)$ (resp. the mean and standard deviation of each mode). We first consider a problem (\emph{GM 1}) with two Gaussian mixtures: $p_1=p_2=(0.5, 0.5)$, $S_1=S_2=(1,1)$, and $\mu_1=(0, 1)$ while $\mu_2=(-1, 0)$ (distribution $2$ is a translation of distribution $1$ by $1$), and we use $B=1.5$. In a second experiment (\emph{GM 2}), the two distributions have the same Gaussian modes $\cN(0,1)$ and $\cN(1,1)$, but different probabilities $p_1=(0.5, 0.5)$ and $p_2=(0.6, 0.4)$, and we still use $B=1.5$. Then, we still consider two distributions (\emph{GM 3}) that differ only by their mode probabilities, but with three modes: $\cN(-3, 0.25)$, $\cN(3, 0.25)$, and $\cN(0,1)$; and $p_1=(0.2, 0.6, 0.2)$ and $p_2=(0.15, 0.6, 0.25)$. We present the results for these experiments in Figure~\ref{fig::exp_GM}. Finally, we consider a fourth experiment (\emph{GM 4}) with five mixture arms and summarize the distributions and average regret for $h$-NPTS and RDS in Figure~\ref{fig::exp_GM4}.

In these experiments, we observe that $h$-NPTS generally performs very closely to RDS, except in \emph{GM 1}, where it significantly outperforms RDS. Notably, \emph{GM 1} represents an easier problem compared to the others, and in that case RDS might be suffering from its conservative theoretical guarantees. These results highlight the strong empirical performance of $h$-NPTS, which efficiently leverages its assumptions to perform well in all scenarios.

\begin{figure}[htbp]
    \centering
    \includegraphics[width=0.45\linewidth]{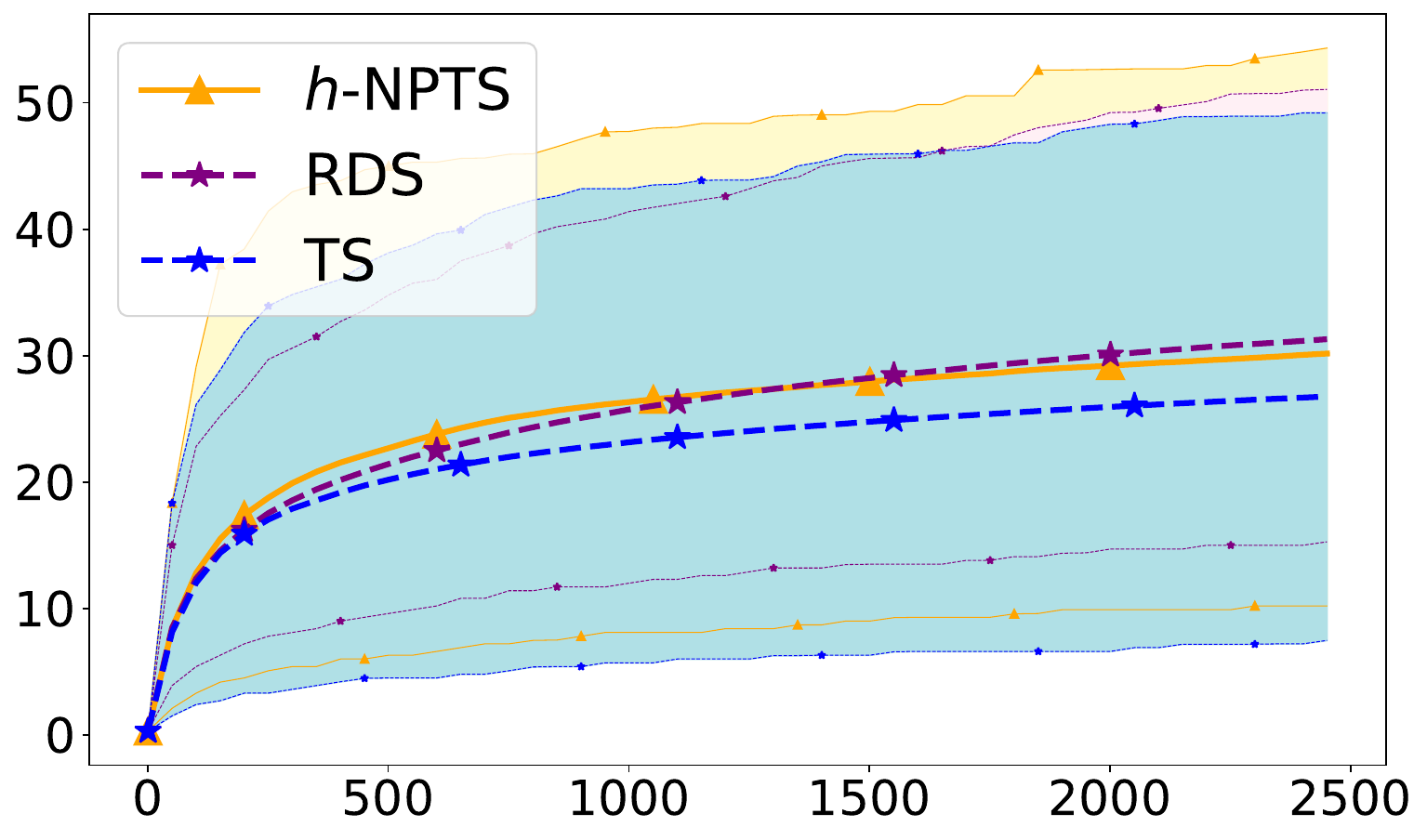}
    \includegraphics[width=0.45\linewidth]{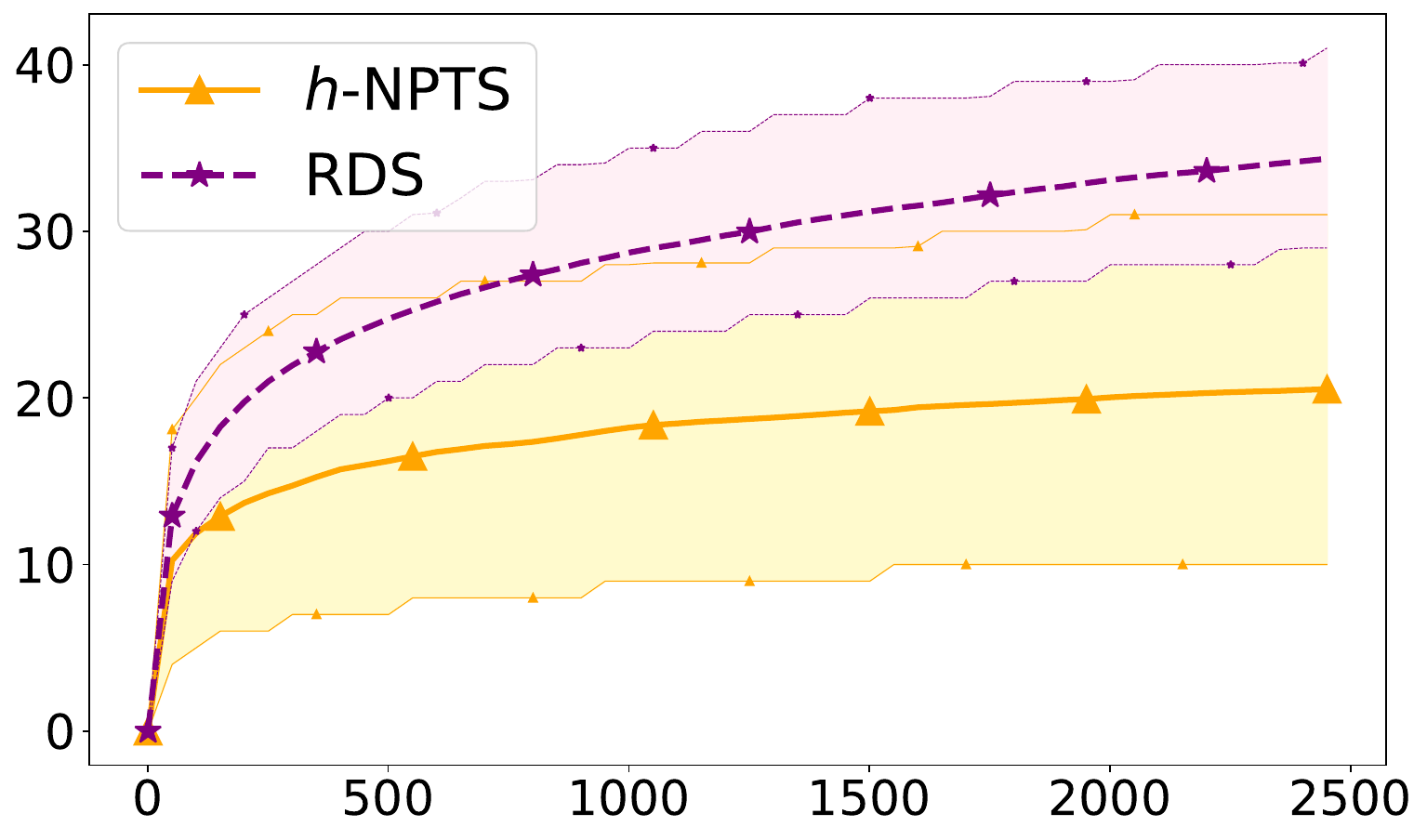}
    \includegraphics[width=0.45\linewidth]{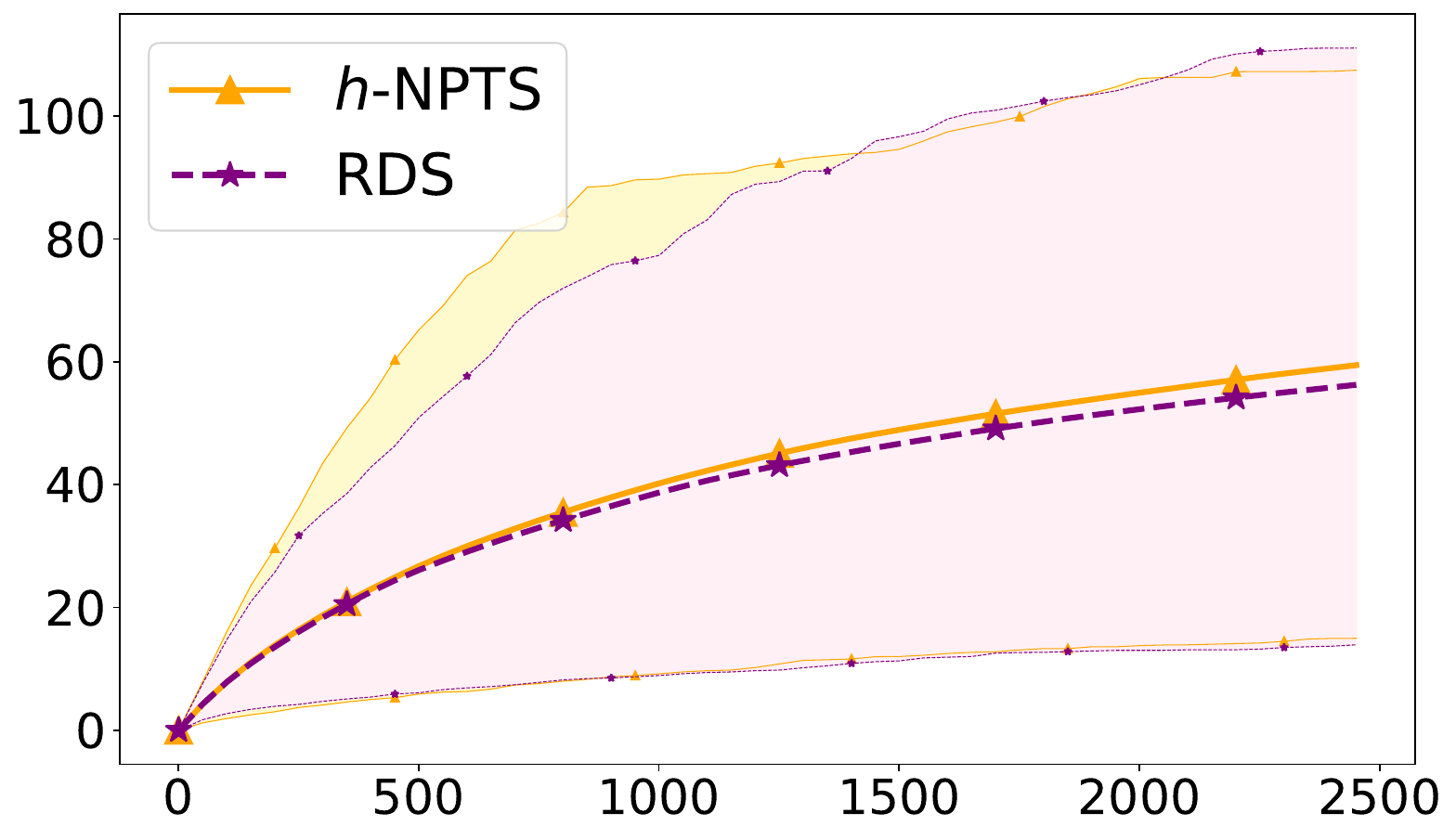}
    \includegraphics[width=0.45\linewidth]{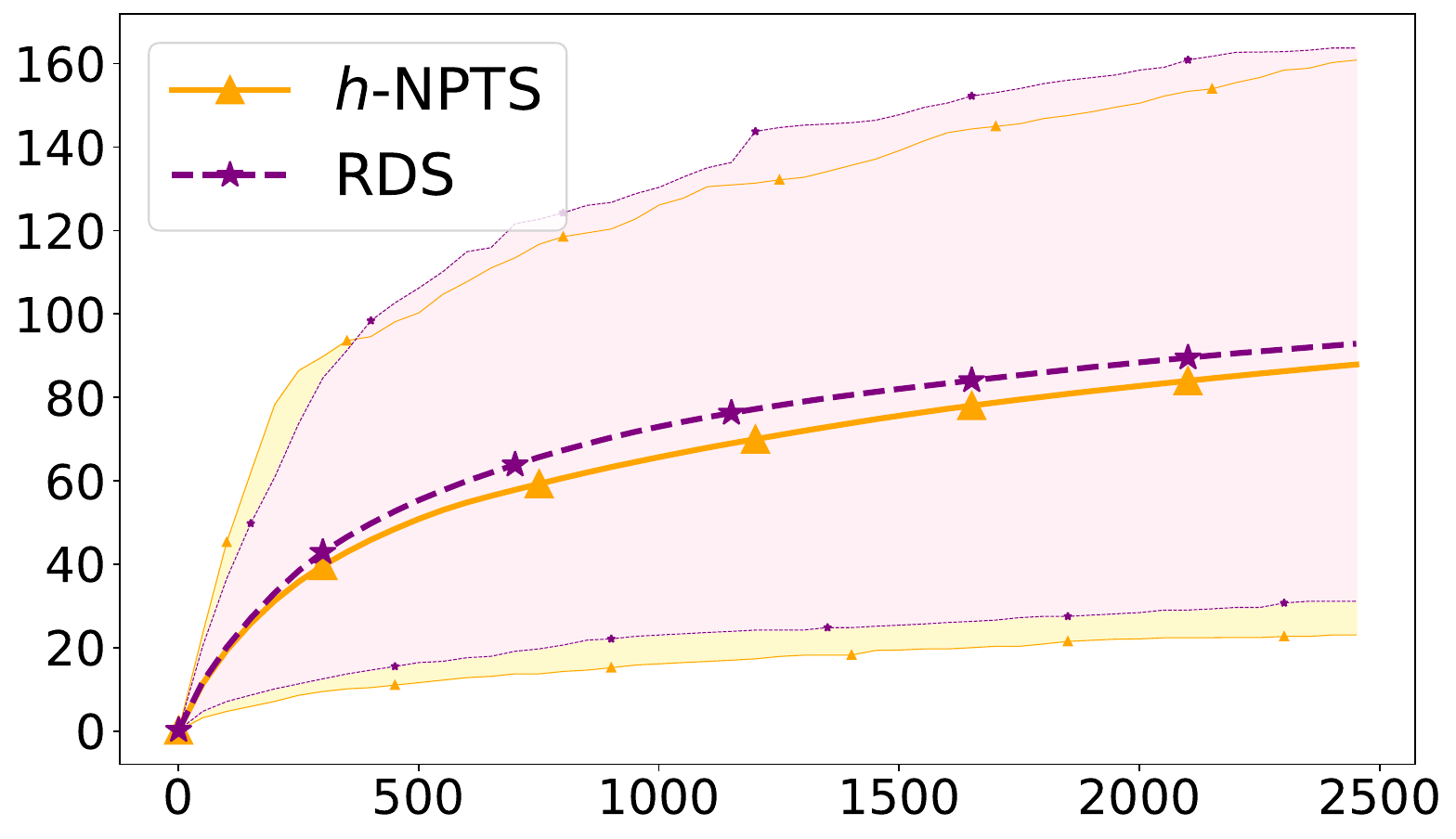}
    \caption{Average regret and $10$--$90\%$ percentiles as a function of $T$ on $500$ runs, for experiments \emph{Gauss 1} with $h$-NPTS (top left), \emph{GM 1} (top right), \emph{GM 2} (bottom left) and \emph{GM 3} (bottom right).}
    \label{fig::exp_GM}
\end{figure}

\begin{figure}[htbp]
    \centering
    \includegraphics[width=0.45\linewidth]{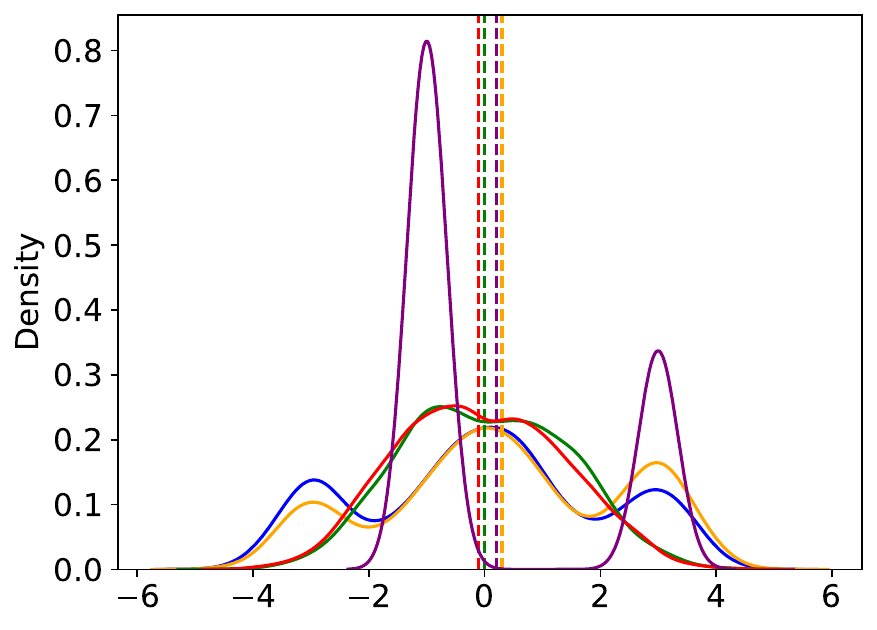}
    \includegraphics[width=0.45\linewidth]{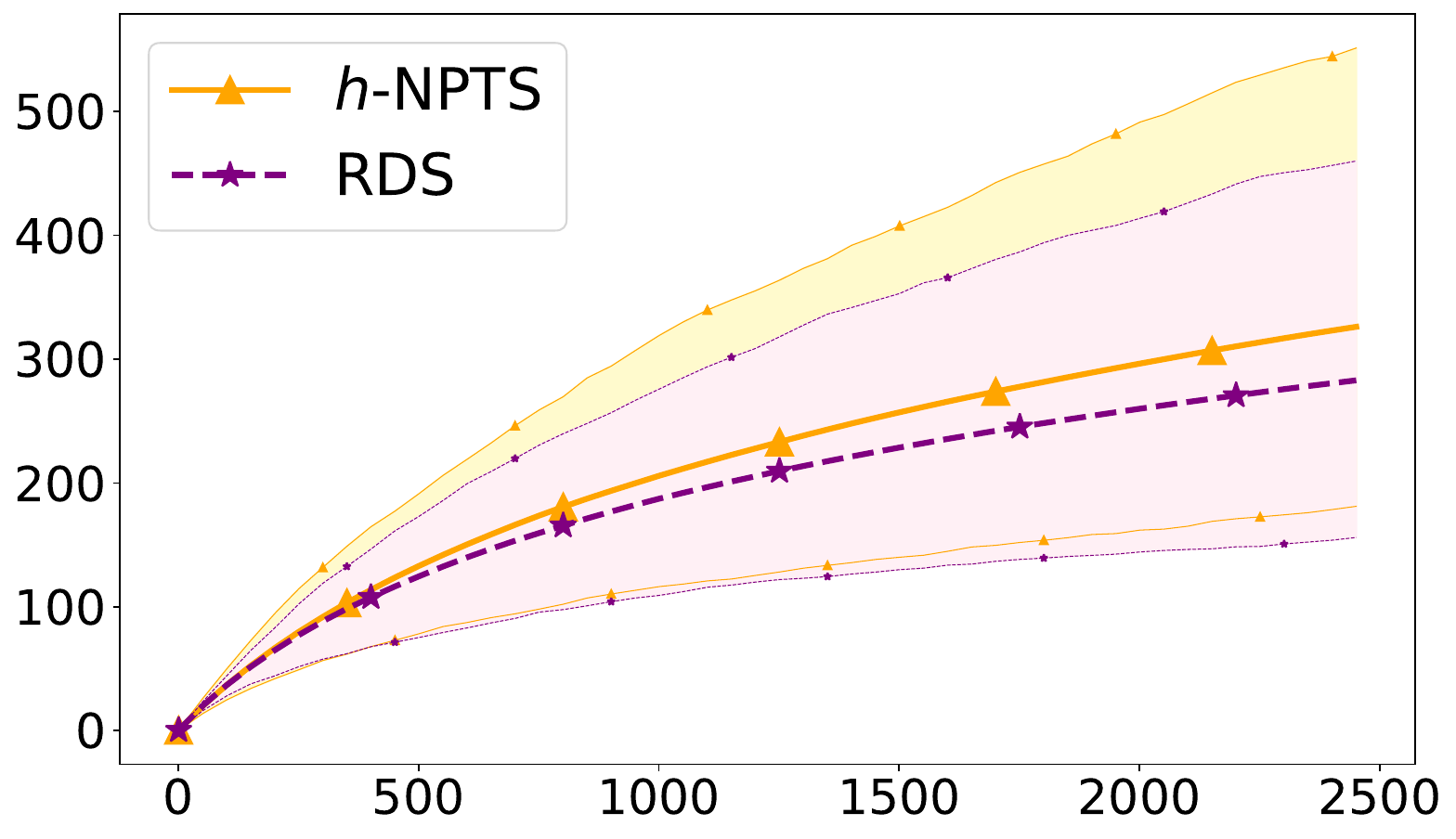}
    \caption[Distributions of the $K=5$ arms in experiment \emph{Gauss 4} (Left), and average regret and $10$--$90\%$ percentiles as a function of $T$ on $500$ runs on this problem (Right).]{Distributions of the $K=5$ arms in experiment \emph{Gauss 4} (Left)\footnotemark, and average regret and $10$-$90\%$ percentiles as a function of $T$ on $500$ runs on this problem (Right).}
    \label{fig::exp_GM4}
\end{figure}

\footnotetext{The exact parameter values can be found in the notebook of the companion Github repository.}

\paragraph{Experimental Conclusions} These experiments, conducted in both parametric and non-parametric settings, provide strong empirical support for the claim in Section~\ref{subsec::presentation_TS} that TS$^\star$ is a close variant of TS. The two policies exhibit nearly identical performance across most cases, with TS$^\star$ occasionally outperforming TS slightly in terms of regret, with significantly improved computational efficiency in non-parametric settings. Our results show that MED consistently underperforms compared to TS and TS$^\star$ in parametric settings but matches their performance in non-parametric cases. The finite-time performance gap between MED and TS in parametric settings could be attributed to potential polynomial multiplicative factors (bounded by $[c_n, C_n]$ for $n \in \mathbb{N}$) in the sampling probabilities for sub-optimal arms. Finally, our experiments with $h$-NPTS confirm that it is indeed an efficient approach --in terms both of performance and computation time compared to MED-- to tackle in a principled way the families of distributions characterized by a simple non-parametric assumption.

\section{Conclusion and perspectives}

In this paper we provided a general recipe to derive regret bounds for randomized bandit algorithms, that allowed us to revisit the Minimum Empirical Divergence (MED) and Thompson Sampling (TS) algorithms for several families of distributions. Guided by our theoretical results, we suggest that \rev{a slight variant of} TS may be interpreted as a way to approximate MED through sampling, that is appealing when the sampling probabilities of MED are hard to compute. Driven by these new insights, we could study in more details some families of distributions satisfying some $h$-moment conditions, for example
distributions with a known variance upper bound. We proved the optimality of a MED algorithm in this setting, and proposed the more computationally efficient $h$-NPTS. While its analysis is intricate, $h$-NPTS is very simple to implement and has (close to optimal) logarithmic regret, making it an appealing solution under such model. 

\rev{An open question remains as to whether supplementary conditions could be identified to achieve the $\sqrt{KT}$ problem-independent bound up to logarithmic factors, which would enable a unified approach to addressing all facets of frequentist regret analysis.} Finally, an interesting research direction may be to find an equivalent unified proof for deterministic algorithms, that would for instance allow to analyze KL-UCB and IMED under the same framework. In particular, it would be interesting to identify if the same Assumption~\ref{ass::sufficient_prop} is sufficient for these algorithms, and how Assumption~\ref{ass::sampling_prob} translates to index comparisons in the deterministic case. 
\newpage

\appendix
\newpage 

{
\tableofcontents
}
\newpage

\section{Index and notation}\label{app::notation}

\paragraph{Notation} We summarize some recurring notation used in this paper.
\begin{itemize}
    \item For an integer $n \in \N$, $[n] = \{1,\dots, n\}$.
    \item $\cP$: set of all possible cumulative distribution functions on $\R$: for any $F\in \cP$, it holds that $F$ is non-decreasing, $\lim_{x\to -\infty}F(x)=0$ and $\lim_{x\to +\infty} F(x)=1$.
.
    \item $\cF$ is a family of distributions.
    \item $(F_k)_{k \in [K]}$: distributions of a $K$-armed bandit, assumed to come from the same family of distributions $\cF$. $(\mu_k)_{k \in [K]}$ is used to denote the means of these distributions. 
    \item $T$ is the time horizon of the bandit algorithm. For any $t\in \N$, $\cH_t = (A_s, r_s)_{s \in [t]}$ is the sequence of arms pulled and rewards collected at the end of each time step.
    \item $N_k(t) = \sum_{s=1}^t \ind(A_s=i)$ is the number of pulls of an arm $k$ up to (and including) time $t$.
    \item $\mu_k(t) = \frac{1}{N_k(t)}\sum_{s=1}^t r_t \ind(A_s=k)$ is the empirical mean of an arm $k$ at round $t$, and the function $F_k(t) : x \in \R \mapsto = \frac{1}{N_k(t)}\sum_{s=1}^t \ind(A_s=k, r_t\leq x)$ is its empirical cdf.
    \item For a sample size $n\in \N$, we use the notation $\mu_{k,n}$ and $F_{k,n}$ respectively for the empirical mean and cdf corresponding to the $n$ first observations collected from arm $k$.
\end{itemize}

\paragraph{Index of $\kinf$ functions} We detail the computation of the $\kinf$ function of each family of distributions considered in this paper. We recall that these functions are used to define the lower bound on the regret for the family $\cF$ \citep{burnetas96LB}, as each optimal arm $k$ satisfies
\[ \liminf_{T\rightarrow \infty} \frac{\bE[N_k(T)]}{\log(T)} \!\geq\! \frac{1}{\kinf^\cF(\nu_k, \mu^\star)} \;,\;\;\; \kinf^\cF(\nu_k, \mu^\star)=\inf_{\nu' \in \cF} \left\{\KL(\nu_k, \nu')\!:\! \bE_{\nu'}(X)\!>\! \mu^\star \right\}\;.\]

\begin{itemize}
	\item \textbf{Single Parameter Exponential Families:} the distribution $F$ is mapped to its mean $\mu_F$, and for any $\mu$ it holds that \[\kinf^\cF(F, \mu)\coloneqq \kl(\mu_F, \mu) \ind(\mu\geq \mu_F).\] 
 The function $\kl$ denotes the Kullback-Leibler divergence corresponding to the family $\cF$, as a function of two means.
	\item \textbf{Bounded distributions} with a known upper bound $B$, \[\kinf^\cF(F, \mu) =  \max_{\lambda \in \left(0, \frac{1}{B-\mu}\right]}\bE_{F}\left[\log\left(1-\lambda(X-\mu)\right)\right]\;.\]
	\item \textbf{Gaussian distributions} with unknown means and variances: the distribution $F$ is mapped to its mean $\mu_F$ and variance $\sigma_F^2$, and
	\begin{align*} \kinf^\cF(F, \mu) &= \frac{1}{2}\log\left(1+\frac{(\mu-\mu_F)^2}{\sigma_F^2}\right)\times \ind(\mu\geq \mu_F)\;,\end{align*}.
	\item \textbf{$h$-moment condition:} for a function $h$ and a constant $B$,
	\begin{align*}\kinf^\cF(F, \mu) &= 
	\max_{(\lambda_1, \lambda_2) \in \cR_2}\bE_{F}\left[\log\left(1-\lambda_1(X-\mu)-\lambda_2 (B-h(X_i))\right)\right] \;,\end{align*}
	with $\cR_2= \{(\lambda_1,\lambda_2) \in (\R^+)^2 \;:\; \forall x \in \R \;, 1-\lambda_1 (x-\mu) - \lambda_2 (B-h(|x|)) \geq 0 \}$ is the set for which the logarithm is defined. We show in Appendix~\ref{app::cent_h_P15} that for centered conditions $\kinf^\cF$ has the same form, except that $h(|X|)$ is replaced by $h(|X-\mu|)$ in the formula.
\end{itemize}

\section{Proofs related to the generic regret analysis} \label{sec::generic_proof}

As explained in the main paper, in this section we detail the proof of Theorem~\ref{th::main_result}. In a first part we state Assumption~\ref{ass::relaxed_sp}, which is slightly more generic than Assumption~\ref{ass::sampling_prob} presented in the main paper. We then prove an intermediate regret bound under this assumption, which in turns proves Lemma~\ref{th::first_bound} in the main paper. We then complete the proof of Theorem~\ref{th::main_result} using the additional statements of Assumption~\ref{ass::sufficient_prop}.

\subsection{\rev{Presentation of Assumption~\ref{ass::relaxed_sp} (relaxation of Assumption~\ref{ass::sampling_prob})}}

The relaxation is based on two simple facts from the analysis: (1) the upper bound on the sampling probability is only necessary when the empirical distribution of a sub-optimal arm is ``close enough'' to its true distribution, and (2) a slightly looser lower bound is sufficient in the analysis if an arbitrary precision in the exponential term can be reached, and (3) these two conditions are only needed for ``large enough'' sample size. Hence, while its statement is relatively heavy (compared to Assumption~\ref{ass::sampling_prob}), the intuition behind the relaxation is very natural. While we showed that all MED policies satisfy Assumption~\ref{ass::sampling_prob} directly, we need the results to hold under Assumption~\ref{ass::relaxed_sp} for the analysis of several \TS{} policies, including $h$-NPTS in Section~\ref{sec::NPTS}.

\begin{assumption}[Relaxation of Assumption~\ref{ass::sampling_prob}]\label{ass::relaxed_sp} $\!$

\textbf{Upper bound:} There exists a neighborhood of $F_k$ (defined with any metric), denoted by $\cB(F_k)$, and $n_0\in \N$ such 
that, for all $t\in [T]$, the sampling probability satisfies 
\begin{equation}\label{eq::ub_sp} N_k(t)\geq n_0,\; F_k(t)\in \cB(F_k)\; \Longrightarrow \; p_k^\pi(t) \leq C_{N_k(t)} \exp\left(-N_k(t) \frac{D_\pi(F_k(t), \mu_\star(t))}{1+a_{N_k(t)}}\right)\;,\end{equation}
where
$(C_{n})_{n\in\mathbb{N}}$, that can depend on $F_k$ through the definition of $\cB(F_k)$, is a sequence satisfying $\log(C_n) = o(n^\alpha)$ for some $\alpha < 1$,
and $(a_n)_{n \in \N}$ is a non-negative sequence. Furthermore, it holds that $\sum_{n=1}^{+\infty} \bP(F_{k,n} \notin \cB(F_k)) < +\infty$: the empirical distribution belongs to $\cB(F_k)$ with high probability for $n$ large enough.

\textbf{Lower bound:} For any 
$\eta>0$, there exists $n_\eta \in \N$ and a sequence $(c_n)_{n\in \N}$ (that might depend on $\eta$) such that for $N_k(t)\geq n_\eta$ it holds that 

\begin{equation}\label{eq::lb_sp}\forall k'\neq k, \; p_k^\pi(t) \geq p_{k'}^\pi(t)\times  c_{N_k(t)}^{-1} \exp\left(-N_k(t)\frac{D_\pi(F_k(t), \mu_\star(t))+\eta}{1+a_{N_k(t)}}\right)\;,\end{equation}
for some sequence $(c_n)_{n\in\N}$ satisfying $c_n = o(\exp(n\eta/2))$. Additionally, there exists some constant $p_\eta$ such that $p_k(t) \geq p_\eta$ for $N_k(t)< n_\eta$.
\end{assumption}

It is clear that Assumption~\ref{ass::relaxed_sp} is strictly more general than Assumption~\ref{ass::sampling_prob}, that we introduced in the main part of this paper to ease the presentation. Indeed, under Assumption~\ref{ass::sampling_prob} the lower bound holds for $\eta=0$ and $n_\eta=1$, and the upper bound holds for $\cB(F_k) = \cP$, the set of all probability distributions on the real line. Hence, proving the result under the relaxed assumption is sufficient to prove them under the stronger assumption.

\subsection{Proof of the intermediate bounds (Lemmas~\ref{th::first_bound} and \ref{lem::first_bound_app})} \label{app::proof_th1}

We now state the more general version of Lemma~\ref{th::first_bound}, and detail its proof.

\begin{lemma}\label{lem::first_bound_app}
Let $\pi$ be a policy satisfying Assumption~\ref{ass::sampling_prob} for a function $D_\pi$ satisfying (A1) of Assumption~\ref{ass::relaxed_sp}. Then, for any $c > \displaystyle \lim_{n\to \infty} a_n$ and for any $\epsilon>0$, and for any parameter $\eta>0$ (and corresponding $p_\eta, n_\eta$ in Assumption~\ref{ass::relaxed_sp}), any sub-optimal arm $k$ satisfies 
\begin{equation}\label{eq::th_2_app}\bE[N_k(T)] \leq  (1+c) \frac{\log(T)}{D_\pi(F_k, \mu_1)} + B(c, \epsilon) + (n_\eta-1) p_\eta^{-1} + o(\log(T))\;,\end{equation}
where the terms captured in $o(\log(T))$ do not depend on $c$, and with \begin{align*}
	B(c, \epsilon) = &\underbrace{\sum_{n=1}^{+\infty} \bP\left(D_\pi(F_{k,n}, \mu_1-\epsilon)\frac{1+c}{1+a_n} \leq D_\pi(F_k, \mu_1)\right)}_{B_1} \\ 
 + &\underbrace{\sum_{n=1}^{+\infty} c_n \bE\left[\ind(\mu_{1,n}\leq \mu_1-\epsilon) e^{n \frac{D_\pi(F_{1,n}, \mu_1-\epsilon)+\eta}{1+a_n}}\right]}_{B_2}
\end{align*}
\end{lemma}

\begin{proof}
The beginning of the proof follows closely standard analysis of Thompson Sampling \citep{AgrawalG17}, assuming that Assumption~\ref{ass::relaxed_sp} holds for a generic function $D_\pi$ and sequences $(a_n, c_n, C_n)_{n \in \N}$. We consider some parameters $c>0, \epsilon\leq \epsilon_0$ (defined in \ref{ass::relaxed_sp}). Then, for any $u \in \N$ we can first upper bound the number of pulls of a sub-optimal arm $k$ by

\begin{align*}
\bE\left[N_k(T)\right] & \leq u + \bE\left[\sum_{t=u}^{T-1} \ind(A_{t+1}=k, N_k(t)>u)\right] \;.
\end{align*}
Classically, we consider two ``good'' events that ensures (when both of them are satisfied) that the sampling probability for a sub-optimal arm $k$ is proportional to $e^{-N_k(t) D_\pi(F_k, \mu_1)}$,
\begin{align}\cG(t)=\{\mu_\star(t) \geq \mu_1 -\epsilon \} \;, \; \text{and} \quad \cJ_k(t)= \left\{\frac{D_\pi(F_k(t), \mu_\star(t))}{1+a_{N_k(t)}} \geq \frac{D_\pi(F_k, \mu_1)}{1+c}, F_k(t) \in \cB(F_k) \right\} \;,\label{def_Jk}
\end{align}
and we denote by $\bar \cG(t)$ and $\bar \cJ_k(t)$ their complementary events. Then, we write that 

\begin{align}
\bE\left[N_k(T)\right] & \leq u + \underbrace{\bE\left[\sum_{t=u}^{T-1} \ind(A_{t+1}=k, N_k(t)>u, \cJ_k(t))\right]}_{\text{(F1)}} \nonumber\\
& + \underbrace{\bE\left[\sum_{t=u}^{T-1} \ind(A_{t+1}=k, N_k(t)>u, \cG(t), \bar \cJ_k(t))\right]}_{\text{(F2)}} + \underbrace{\bE\left[\sum_{t=u}^{T-1} \ind(A_{t+1}=k, N_k(t)>u, \bar \cG(t))\right]}_{\text{(F3)}}  \;.\label{decomp_F123}
\end{align}

\paragraph{Upper bounding (F1)}
Since $\log(C_n)=o(n^{\alpha})$ is assumed in Assumption~\ref{ass::relaxed_sp}, there exists $\beta>0$ such that $C_n\le  \exp(\beta n^{\alpha})$ for all $n\in\N$.
Then
the event $\cJ_k(t)$ provides an upper bound on $p_k^\pi(t)$ and we obtain 
\begin{align}
\text{(F1)} & = \bE\left[\sum_{t=u}^{T-1} \ind(A_{t+1}=k, N_k(t)>u, \cJ(t))\right] \nn\\
& \leq \bE\left[\sum_{t=u}^{T-1} \ind(N_k(t)>u, \cJ_k(t)) \times p_k^\pi(t)\right] \nn\\
& \leq \bE\left[\sum_{t=u}^{T-1} \ind(N_k(t)>u, \cJ_k(t)) \times C_{N_k(t)}\exp\left(-N_k(t) \frac{D_\pi(F_k(t), \mu_\star(t))}{1+a_{N_k(t)}}\right)\right] \nn\\
& \leq \bE\left[\sum_{t=u}^{T-1} \ind(N_k(t)>u, \cJ_k(t))\times C_{N_k(t)} \exp\left(-N_k(t) \frac{D_\pi(F_k, \mu_1)}{1+c}\right)\right] \label{use_Jk}\\
& \leq  \bE\left[\sum_{t=u}^{T-1} \ind(N_k(t)>u, \cJ_k(t)) \exp\left(\beta u^{\alpha}-u \frac{D_\pi(F_k, \mu_1)}{1+c}\right)\right]\nn \\
& \leq T \exp\left(\beta u^{\alpha}-u \frac{D_\pi(F_k, \mu_1)}{1+c}\right) \;,\nn
\end{align}
where we used \eqref{def_Jk} (definition of $\cJ_k(t)$) in \eqref{use_Jk}, assumed $u>n_0$ for $n_0$ defined in Assumption~\ref{ass::relaxed_sp}, and further assumed that $C_n \exp\left(-n \frac{D_\pi(F_k,\mu_1)}{1+c}\right)$ is decreasing in $n$ for $n\geq n_0$. Note that this last condition can be guaranteed easily, since $C_n$ is dominated by the exponential term, and up to redefining the first values of $(C_n)$ is this is not the case. Here we set
\begin{align}
u = \left\lceil (1+c)\frac{\log(T)+ \gamma \log(T)^{\alpha}}{D_\pi(F_k, \mu_1)}\right\rceil \vee n_0
\label{choice_u}
\end{align}
for a fixed value of $\gamma$ that will be specified later.
This choice satisfies $u\geq n_0$, that we can now ignore for the rest of the analysis (it is captured by the $o(\log(T))$ notation in the statement). Therefore, we obtain
\begin{align*}
	\text{(F1)} & \leq T \exp\left(\beta u^\alpha -u \frac{D_\pi(F_k, \mu_1)}{1+c}\right)\\
&\leq \exp\left( \beta \left(1 +(1+c)\frac{\log(T)+ \gamma \log(T)^{\alpha}}{D_\pi(F_k, \mu_1)}\right)^\alpha -\gamma \log(T)^\alpha \right) \\
&\leq \exp\left(\left(\beta\left(1 +\frac{(1+c)(1+ \gamma)}{D_\pi(F_k, \mu_1)}\right)^\alpha -\gamma\right) \log(T)^\alpha \right) \;.
\end{align*}
Then, we can choose any $\gamma\geq 1$ such that $\gamma \geq \beta\left(1+\frac{(1+c)(1+ \gamma)}{D_\pi(F_k, \mu_1)}\right)^\alpha$, which exists since the right-hand term is of order $\gamma^\alpha$ because $\alpha<1$ by assumption.

For such choice, (F1) vanishes when $T$ becomes large and the $o_c(\log(T))$ term in \eqref{eq::th_2} scales more precisely with $\log(T)^\alpha$ because the term $u$ becomes dominant in \eqref{decomp_F123}.

\rev{\begin{remark}[More precise scaling for smaller $(C_n)_{n\in \N}$]\label{rem::scaling}
With a very similar proof, for polynomial $C_u$ we could choose $u$ of the form \[u = (1+c)\frac{\log(T)+ \gamma \log(\log(T)))}{D_\pi(F_k, \mu_1)},\] 
instead of \eqref{choice_u}
and obtain the same result, for a well-chosen $\gamma$. Finally, for $(C_n)_{n\in \N}$ satisfying $C_n\leq C$ for some constant $C>0$, it is clear from ~\eqref{use_Jk} that \emph{(F1)} can be upper bounded by $CTe^{-u\frac{D_\pi(F_k, \mu_1)}{1+c}}$, and so we can directly choose
\[u = (1+c)\frac{\log(T \cdot D_\pi(F_k, \mu_1))}{D_\pi(F_k, \mu_1)}\]
and obtain
\[\text{\emph{(F1)}} \leq \frac{(1+c) \log(T \cdot D_\pi(F_k, \mu_1)) + C}{D_\pi(F_k, \mu_1)} \;, \]
that we use in Section~\ref{subsubsec::discussion} to discuss problem-independent guarantees under our framework. 
\end{remark}
}

\paragraph{Upper bounding (F2)} We remark that as $D_\pi$ is non-decreasing in the second argument we have \[\left\{\cG(t) \cap \bar \cJ_k(t)\right\} \subset \cM_k(t) \coloneqq \left\{\frac{D_\pi(F_k(t), \mu_1-\epsilon)}{1+a_{N_k(t)}} \leq \frac{D_\pi(F_k, \mu_1)}{1+c} \right\}\cup \{F_k(t) \notin \cB_k(F_k) \} \;.\]
For simplicity we denote by $\cM_{k, n}$ the corresponding event replacing $F_k(t)$ by $F_{k,n}$. Then,

\begin{align*}
\text{(F2)} &= \bE\left[\sum_{t=u}^{T-1} \ind(A_{t+1}=k, N_k(t)>u, \cG(t), \bar \cJ_k(t))\right] \\
& \leq \bE\left[\sum_{t=u}^{T-1} \ind(A_{t+1}=k, N_k(t)>u, \cM_k(t))\right] \\
& \leq \bE\left[\sum_{t=u}^{T-1} \sum_{n=u}^{T} \ind(A_{t+1}=k, N_k(t)=n, \cM_{k, n})\right] \\
& \leq \bE\left[\sum_{n=u}^{T-1} \ind(\cM_{k, n}) \sum_{t=u}^{T-1} \ind(A_{t+1}=k, N_k(t)=n)\right] \\
& \leq \sum_{n=u}^{+\infty} \bP(\cM_{k,n}) \;.
\end{align*}

We then upper bound $\bP(\cM_{k,n})$ by the sum of the probabilities of the two events, and thus obtain that 
\[(\text{F2}) \leq \sum_{n=1}^{+\infty} \bP\left(\frac{D_\pi(F_k(t), \mu_1-\epsilon)}{1+a_{N_k(t)}} \leq \frac{D_\pi(F_k, \mu_1)}{1+c}\right) + \sum_{n=1}^{+\infty} \bP(F_{k, n} \notin \cB(F_k)), \]
which gives the term $B_1$ in the statement of Lemma~\ref{lem::first_bound_app}, and a constant term thanks to Assumption~\ref{ass::relaxed_sp}. Note that this term is upper bounded by $0$ under Assumption~\ref{ass::sampling_prob}, since $\cB(F_k)=\cP$. 

\paragraph{Upper bounding (F3)} We use a proof scheme from \cite{AgrawalG17} and the lower bound part of Assumption~\ref{ass::relaxed_sp} to obtain that
\begin{align*}
\text{(F3)} & \leq \bE\left[\sum_{t=u}^{T-1} \ind(N_k(t)>u, A_{t+1}=k, \bar \cG_t)\right] \\
& = \bE\left[\sum_{t=u}^{T-1} \ind(N_k(t)>u, \bar \cG_t) \times p_k^\pi(t)\right] \\
& = \bE\left[\sum_{t=u}^{T-1} \ind(N_k(t)>u, \bar \cG_t)\times p_1^\pi(t) \times \frac{p_k^\pi(t)}{p_1^\pi(t)}\right] \\
& \leq \bE\left[\sum_{t=u}^{T-1} \ind(\bar \cG_t) \times \frac{p_k^\pi(t)}{p_1^\pi(t)}\times \bE[\ind(A_{t+1}=1)|\cH_{t-1}]\right] \\
& = \bE\left[\sum_{t=u}^{T-1} \ind(A_{t+1}=1, \bar \cG_t) \times \frac{p_k^\pi(t)}{p_1^\pi(t)}\right]
\end{align*}
We then use Assumption~\ref{ass::relaxed_sp} and (A1) of Assumption~\ref{ass::sufficient_prop} to upper bound the ratio of probabilities as follows, for any $\eta>0$ and $n\geq n_\eta$ as formulated in the assumption,
\begin{align*}
\frac{p_k^\pi(t)}{p_1^\pi(t)} &\leq c_{N_1(t)} \exp\left(N_1(t)\frac{D_\pi(F_1(t), \mu^\star(t))+\eta}{1+a_{N_1(t)}}\right), \quad &\text{(\ref{ass::relaxed_sp})} \\
&\leq c_{N_1(t)} \exp\left(N_1(t)\frac{D_\pi(F_1(t), \mu_1-\epsilon)+\eta}{1+a_{N_1(t)}}\right)\;, \quad &\text{(A1) in \ref{ass::sufficient_prop}}
\end{align*}
and we thus obtain an expression that is fully determined by the first $N_1(t)$ samples of arm $1$. We then use a union bound on the number of pulls of this arm, and simplifies the event $\mu^\star(t) \leq \mu_1-\epsilon$ to $\mu_1(t) \leq \mu_1-\epsilon$, so that all terms in the expectation now depend on arm $1$ only. For convenience, we use the notation \[\text{PR}_n(\mu_1-\epsilon)=c_n e^{n\frac{D_\pi(F_{1,n}, \mu_1-\epsilon)+\eta}{1+a_n}}\ind(n\geq n_\eta)+ p_\eta^{-1}\ind(n<n_\eta),\] and obtain 

\begin{align*}
\text{(F3)} & \leq \bE\left[\sum_{t=u}^{T-1} \sum_{n=1}^{T-1} \ind(A_{t+1}=1, N_1(t)=n, \mu_{1,n}< \mu_1-\epsilon) \times \text{PR}_n(\mu_1-\epsilon)\right] \\
& \leq \bE\left[\sum_{n=1}^{T-1}\ind(\mu_{1,n} < \mu_1-\epsilon)\times \text{PR}_n(\mu_1-\epsilon)\times \sum_{t=u}^{T-1} \ind(A_{t+1}=1, N_1(t)=n)\right] \\
& \leq \sum_{n=1}^{T-1}\bE\left[\ind(\mu_{1,n} < \mu_1-\epsilon)\times \text{PR}_n(\mu_1-\epsilon)\right]\\
& \leq \sum_{n=1}^{n_\eta-1} \bP(\mu_{1,n} \leq \mu_1-\epsilon) p_\eta^{-1} + \sum_{n=n_\eta}^{+\infty}  \bE\left[\ind(\mu_{1,n} < \mu_1-\epsilon)\times c_n e^{n\frac{D_\pi(F_{1,n}, \mu_1-\epsilon)+\eta}{1+a_n}}\right] \\ 
&\leq (n_\eta-1) p_\eta^{-1} + \sum_{n=1}^{+\infty}  c_n \bE\left[\ind(\mu_{1,n} < \mu_1-\epsilon)\times e^{n\frac{D_\pi(F_{1,n}, \mu_1-\epsilon)+\eta}{1+a_n}}\right] \;.
\end{align*}
This completes the proof of Lemma \ref{lem::first_bound_app}, and at the same time the proof of Lemma~\ref{th::first_bound} in the main paper with $n_\eta=1$ and $\eta=0$.

\end{proof}
\subsection{Proof of Theorem~\ref{th::main_result}} \label{app::proof_th2}

We start the proof of the theorem from the intermediate bound of Lemma~\ref{lem::first_bound_app}. The remaining steps consist in further upper bounding $B(c, \epsilon)$, assuming that Assumption~\ref{ass::sufficient_prop} holds in full. We recall that

\begin{align*}
	B(c, \epsilon) = &\underbrace{\sum_{n=u}^{+\infty} \bP\left(D_\pi(F_{k,n}, \mu_1-\epsilon)\frac{1+c}{1+a_n} \leq D_\pi(F_k, \mu_1)\right)}_{B_1} \\ 
&+ \underbrace{\sum_{n=1}^T c_n \bE\left[\ind(\mu_{1,n}\leq \mu_1-\epsilon)e^{n \frac{D_\pi(F_{1,n}, \mu_1-\epsilon)+\eta}{1+a_n}}\right]}_{B_2}	
\end{align*}


\paragraph{\rev{Upper bounding $B_1$}} We first re-arrange this term in order to use (A2) in Assumption~\ref{ass::sufficient_prop}. 
\begin{align*}
&\bP\left(D_\pi(F_{k,n}, \mu_1-\epsilon)\frac{1+c}{1+a_n} \leq D_\pi(F_k, \mu_1)\right) \\
&= \bP\left(D_\pi(F_{k,n}, \mu_1-\epsilon) \leq D_\pi(F_{k}, \mu_1-\epsilon) - \left(D_\pi(F_{k}, \mu_1-\epsilon)-D_\pi(F_k, \mu_1) \frac{1+a_n}{1+c} \right)\right)\;. \\ 
\end{align*}
The rest of the argument consist in finding an appropriate tuning of $\epsilon$, as a function of $c$ such that the right-hand term $\delta_{\epsilon, c}\coloneqq D_\pi(F_{k}, \mu_1-\epsilon)-D_\pi(F_k, \mu_1) \frac{1+a_n}{1+c}$ (that will be the $\delta$ used in (A2)) is a positive constant.
We first remove $a_n$ by using the notation $a_\infty = \lim_{n \to \infty} a_n$, and since $c>a_\infty$, we can deduce that there exists a rank $n_c\in \N$ such that, for all $n\geq n_c$, it holds that $a_n\leq \frac{c+a_\infty}{2}= c - \frac{c-a_\infty}{2}$. Hence, we obtain that 
\[\forall n \geq n_c,\; \delta_{\epsilon, c} \geq D_\pi(F_{k}, \mu_1-\epsilon)-D_\pi(F_k, \mu_1) \left(1-\frac{c-a_\infty}{2(1+c)}\right) \;. \]
We now proceed to tune $\epsilon$, first recalling that $D_\pi$ is continuous and non-decreasing in the second argument ((A1) in \ref{ass::sufficient_prop}). Although choosing $\epsilon>0$ is necessary to upper bound $B_2$, if we could consider $\epsilon=0$ we would be able to use \[\delta_{0, c}= D(F_k, \mu_1) \times \frac{c-a_\infty}{2(1+c)} >0. \]
Additionally, it is immediate that for $\epsilon \geq \Delta_k$ we would get $\delta_{\Delta_k, c}\leq 0$. Hence, by the intermediate value theorem we are guaranteed that there exists a value of $\epsilon$, that we denote by $\epsilon_c$, satisfying
\[\epsilon_c >0, \text{and } \delta_c \coloneqq \delta_{\epsilon_c, c} = D(F_k, \mu_1) \times \frac{c-a_\infty}{4(1+c)} >0 \;.  \]
Then, using property (A2) of Assumption~\ref{ass::sufficient_prop} provides that $B_1 \leq n_c + A_{\delta_c} < +\infty$, which also fixes $\epsilon$ to $\epsilon_c$ for the rest of the proof, although we keep writing $\epsilon$ for simplicity.

\paragraph{Upper bounding $B_2$} We directly start by using (A4), which states that there exists a constant $\delta_{\epsilon, \mu_1}$ such that, if the empirical distribution of arm $1$ satisfies $\mu_{1,n}\leq \mu_1-\epsilon$, then 
\[D_\pi(F_{1,n}, \mu_1) \geq D_\pi(F_{1,n}, \mu_1-\epsilon) + \delta_{\epsilon, \mu_1} \;. \]
Then, just as we fixed $\epsilon$ as a function of $c$ in last paragraph, we now fix the parameter $\eta$ of Lemma~\ref{lem::first_bound_app} as a function of $\epsilon \coloneqq \epsilon_c$ (hence, as a function of $c$). Indeed, we choose $\eta=\frac{\delta_{\epsilon, \mu_1}}{2}$, and thus get that, for all $n\in \N$, 
\begin{align*}
\bE\left[\ind(\mu_{1,n}\leq \mu_1-\epsilon)e^{n \frac{D_\pi(F_{1,n}, \mu_1-\epsilon)+\eta}{1+a_n}}\right] & \leq e^{-n \frac{\delta_{\epsilon, \mu_1}}{2(1+a_n)}}\times \underbrace{\bE\left[\ind(\mu_{1,n}\leq \mu_1-\epsilon)e^{n \frac{D_\pi(F_{1,n}, \mu_1)}{1+a_n}}\right]}_{B_2'}\;,
\end{align*}
and we then further upper bound the term denoted by $B_2'$ using (A4), as follows. We first write that 
\[B_2' \leq \bP(\mu_{1,n}\leq \mu_1-\epsilon) + \bE\left[e^{n \frac{D_\pi(F_{1,n}, \mu_1)}{1+a_n}}-1\right]\;, \]
where we dropped the indicator in the expectation. Then, remarking that the random variable inside the expectation is non-negative we can further write that
\begin{align*}
    \bE\left[e^{n \frac{D_\pi(F_{1,n}, \mu_1)}{1+a_n}}-1\right] & = \int_{0}^{+\infty} \bP\left(e^{n \frac{D_\pi(F_{n}, \mu_1)}{1+a_n}}-1 > x\right) \mathrm{d}x \\
	& =  \int_{0}^{+\infty} \bP\left(D_\pi(F_{n}, \mu_1)> \frac{1+a_n}{n}\log\left(1+x)\right)\right) \mathrm{d}x  \\
 \end{align*}
Let us first consider the case $a_n>0$ in the theorem. Then, by using (A4) we simply get that
\[\bE\left[e^{n \frac{D_\pi(F_{1,n}, \mu_1)}{1+a_n}}-1\right] \leq \alpha_n \int_{0}^{+\infty} \frac{ \mathrm{d}x}{(1+x)^{1+a_n}} = \frac{\alpha_n}{a_n} \;, \]
so in that case we obtain that 
\begin{align*}
B_2 &\leq \sum_{n=1}^{+\infty} e^{-n \frac{\delta_{\epsilon, \mu_1}}{2(1+a_n)}}\times \bP(\mu_{1,n}\leq \mu_1-\epsilon) + \sum_{n=1}^{+\infty} \frac{\alpha_n \times c_n}{a_n} e^{-n \frac{\delta_{\epsilon, \mu_1}}{2(1+a_n)}} \\
&\leq \sum_{n=1}^{+\infty} e^{-n \frac{\delta_{\epsilon, \mu_1}}{2(1+a_n)}}+ \sum_{n=1}^{+\infty} \frac{\alpha_n \times c_n}{a_n} e^{-n \frac{\delta_{\epsilon, \mu_1}}{2(1+a_n)}} \\
& \leq \frac{2(1+a_n)}{\delta_{\epsilon, \mu_1}} + \sum_{n=1}^{+\infty} \frac{\alpha_n \times c_n}{a_n} e^{-n \frac{\delta_{\epsilon, \mu_1}}{2(1+a_n)}} \;,
\end{align*}
where we wrote the first line to show that an exponential bound for $\bP(\mu_{1,n}\leq \mu_1-\epsilon)$ can further tighten the result, but is not necessary. We thus obtain that
\begin{equation}\label{eq::bound_B2_delta}
B_2 \leq \frac{2(1+a_n)}{\delta_{\epsilon, \mu_1}} + \sum_{n=1}^{+\infty} \frac{\alpha_n \times c_n}{a_n} e^{-n \frac{\delta_{\epsilon, \mu_1}}{2(1+a_n)}}\;.
\end{equation}
To conclude, it remains to remark that all the assumptions on $(\alpha_n, c_n, a_n^{-1})$ lead to the conclusion that $\frac{\alpha_n \times c_n}{a_n}=o\left(e^{-n \frac{\delta_{\epsilon, \mu_1}}{4(1+a_n)}}\right)$, so the sum converges. This proves that $B_2<+\infty$. 

We now consider the second case, where $a_n=0$ but $D_\pi(F_{1,n}, \mu_1)\leq D_{\mu_1}^+$ for some constant $D_{\mu_1}^+<+\infty$. This time, we first obtain that
\[\bE\left[e^{n \frac{D_\pi(F_{1,n}, \mu_1)}{1+a_n}}-1\right] \leq \alpha_n \int_{0}^{e^{\frac{nD_{\mu_1}^+}{1+a_n}}-1} \frac{ \mathrm{d}x}{1+x} = \alpha_n\times nD_{\mu_1}^+ \;. \]
We then apply the exact same steps as in the previous case and reach the same conclusion, with the term $\frac{\alpha_n\times c_n}{a_n}$ being replaced by $\alpha_n\times c_n\times nD_{\mu_1}^+$. This concludes the proof of Theorem~\ref{th::main_result}.

\qed

\subsection{Mean estimate for heavy-tailed distributions} \label{app::A5_heavy}

We now show that if the mean is expected to concentrate slowly, a robust estimator may be used instead of the empirical mean. As an example, we consider a truncation-based estimator presented for instance in 
\citet{bubeck_heavy}.
We consider a sequence $(u_n)_{n \geq 0}$ and further assume that $\bE[|X|^{1+\delta}]\leq B$ for some $\delta>0, B\geq 0$. Given observations $X_1,\dots, X_n$, the truncation-based estimator returns
 \[ \mu_n =  \frac{1}{n} \sum_{i=1}^n X_i \ind(|X_i|\leq u_n).\] 
Then, denoting by $\widetilde \mu_n$ the expectation of the truncated variable we can follow the proof of \cite{bubeck_heavy} to first show that  $\widetilde \mu_n \geq \mu_1 - \frac{B}{u_n^{\epsilon}}$ (Markov inequality), and then obtain with Bernstein inequality that 

\begin{align*}
	\bP\left(\frac{1}{n} \sum_{i=1}^n X_i \ind(|X_i|\leq u_n) \leq \mu_1-\epsilon\right)  & = 	\bP\left(\frac{1}{n} \sum_{i=1}^n (X_i \ind(|X_i|\leq u_n)-\widetilde \mu_n) \leq (\mu_1-\widetilde \mu_n)-\epsilon \right) \\
 & \leq \exp\left(- \frac{n^2 \frac{\epsilon^2}{8}}{n B u_n^{1-\epsilon} + \frac{1}{3}u_n \epsilon}\right)\;,
\end{align*}
and so, for any $\epsilon>0$,  $\sum_{n=1}^{+\infty} c_n \bP(\mu_n \leq \mu_1-\epsilon)$ converges for instance by taking $u_n = \cO(n)$: $u_n=n$ leads to a bound on the probabilities that is exponentially decreasing in $n^{\epsilon}$. 

Note that this change has no incidence on the proof of Theorem~\ref{th::main_result}, as long as the family $\cF$ and the bandit algorithm satisfy the other assumptions. Indeed, in the proof of the theorem the explicit form of $\mu_n$ is never used.





\section{Detailed results for existing Thompson sampling algorithms}\label{sec::TS}

\subsection{Detailed implementations of the policies considered}\label{subsec::old_ts}

In this section we detail the practical implementation of several instances of \TS{} analyzed in this paper. In each case we introduce the exact definition of the sampler and some potential ``pre-sampling'' step necessary for their analysis.

\paragraph{SPEF-\TS} We start with our implementation of TS for SPEF with a conjugate prior and posterior in Algorithm~\ref{alg::TS_SPEF}. It is known (see for instance \citealp{kaufmann_bayesucb}) that for any choice of conjugate prior/posterior, the expression of the sampling distribution can be expressed in terms of the prior and the $\KL$ divergence of the family (denoted by $\kl$ in the proofs). Furthermore, as our analysis uses some results from \cite{kaufmann_bayesucb} and \cite{jin22expTS} we use a clipping of empirical means, assuming that the true means belong to a bounded range $[\mu_0^-, \mu_0^+]$

\begin{algorithm}[h]
	\SetKwInput{KwData}{Input}
	\SetKwComment{Comment}{\color{purple} $\triangleright$\ }{}
	\SetKwComment{Titleblock}{// }{}
	\KwData{Prior $p_0$, $\kl$ function corresponding to the SPEF $\cF$, range $[\mu_0^-, \mu_0^+]$}
 	For all $k$: draw the arm once, initialize the empirical distribution means $\wh \mu_k$, and $N_k=1$. \\
 	\For{$t> 1$}{
			Set $\cA=\aargmax_{k \in \K}\wh \mu_k$ \Comment*[r]{\color{purple} \small Candidate arm(s) = Best arm(s)}
	\For{$k \notin \cA$}{
 		Set $\wh \mu_k = \max\{\mu_0^-, \min \{\wh \mu_k, \mu_0^+\}\}$ \Comment*[r]{\color{purple} \small Clipping}
   \vspace{2mm}
				Sample $\widetilde \mu_k \sim \pi_k(.|\wh \mu_k, N_k)$, with $\pi_k(\mu|\wh \mu_k) = p_0(\mu)\exp(-N_k \kl(\mu, \wh \mu_k))$ \\
			Add $k$ to $\cA$ if $\widetilde \mu_k \geq \max_{j} \wh \mu_{j}$}
   		Choose $k_t$ at random in $\cA$, collect a reward, update $\wh F_{k_t}$, $\wh \mu_{k_t}$, $N_{k_t}=N_{k_t}+1$
   }
		\SetKwInput{KwResult}{Return}
	\caption{variant of TS with conjugate prior/posterior for SPEF}
	\label{alg::TS_SPEF}
\end{algorithm}

\paragraph{Gaussian TS (G-\TS)} We provide in Algorithm~\ref{alg::TS_G} our variant of the Gaussian Thompson Sampling with inverse-gamma priors, proposed and analyzed in \cite{honda14}. 

\begin{algorithm}[h]
	\SetKwInput{KwData}{Input}
	\SetKwComment{Comment}{\color{purple} $\triangleright$\ }{}
	\SetKwComment{Titleblock}{// }{}
	\KwData{Shape parameter $\alpha\in \R$, lower bound on the means $\mu^-$, upper bound on the variances $V$}
	Pull each arm $N_k=\max\left\{2, 3-\lceil 2\alpha\rceil \right\}$ times, initialize means/variances $(\wh \mu_k, \wh \sigma_k^2)_{k \in [K]}$\\
	\For{$t> 1$}{
    Set $\cA=\aargmax_{k \in \K}\wh \mu_k$\Comment*[r]{\color{purple} \small Candidate arm(s) = Best arm(s)}
		\For{$k \notin \cA$}{
Sample $\widetilde \mu_k$ from posterior $\pi_k(.|(\wh \mu_k, \wh \sigma_k^2, N_k))$ with \[\pi_k(\mu|(\wh \mu_k, \wh \sigma_k^2, N_k))\propto \left(1+\frac{(\mu - \wh \mu_k\vee \mu^-)^2}{\wh \sigma_k^2 \wedge V}\right)^{-\frac{N_k}{2}-\alpha}\;. \] \Comment*[r]{\color{purple} \small Posterior sampling with clipped estimates.}
		Add $k$ to $\cA$ if $\widetilde \mu_k \geq \max_j \wh \mu_j$ \\ }}
		Choose $k_t$ at random in $\cA$, collect a reward, update $\wh \mu_k, \wh \sigma_k^2$, $N_k=N_k+1$
	\SetKwInput{KwResult}{Return}
	\caption{Gaussian Thompson Sampling adapted from \citet{honda14}}
	\label{alg::TS_G}
\end{algorithm}

\paragraph{Non Parametric TS (NP\TS)} Finally, we present in Algorithm~\ref{alg::NPTS} the proposed variant of Non-Parametric Thompson Sampling \citep{RiouHonda20}. The algorithm relies on the Dirichlet distributions $\cD_n \coloneqq \Dir((1,\dots, 1))$ ($n$ ones), which is simply the uniform distribution on the $n$-simplex. 

\begin{algorithm}[hbtp]
	\SetKwInput{KwData}{Input}
	\SetKwComment{Comment}{\color{purple} $\triangleright$\ }{}
	\SetKwComment{Titleblock}{// }{}
	\KwData{Upper bound $B$}
	For all $k$: Initialize the vector of observations for each arm $X_k=(B)$, and $N_k=1$.  \\
	\For{$t> 1$}{
 Set $\cA=\aargmax_{k \in \K}\wh \mu_k$\Comment*[r]{\color{purple} \small Candidate arm(s) = Best arm(s)}
		\For{$k \notin \cA$}{
			Sample $w=(w_1, \dots, w_{N_k}) \sim \cD_{N_k}$ \\
			Set $\widetilde \mu_k = w^T X_k$ \\
                Add $k$ to $\cA$ if $\widetilde \mu_k \geq \max_j \wh \mu_j$
		}
		Choose $k_t$ at random in $\cA$, collect a reward, extend $X_{k_t}$ with the collected reward and update $N_{k_t}=N_{k_t}+1$
	}  
	\SetKwInput{KwResult}{Return}
	\caption{Non-Parametric Thompson Sampling adapted from \citet{RiouHonda20}}
	\label{alg::NPTS}
\end{algorithm}

\subsection{BCP for TS with conjugate posterior for SPEF}\label{app::ts_spef} 
 
\begin{lemma}[BCP for SPEF]\label{lem::bcp_spef}
Let $\cF$ be a SPEF. Assume that each arm $k$ from the bandit problem satisfies $\mu_k \in [\mu_0^-, \mu_0^+]$ for some finite and known constants $\mu_0^-, \mu_0^+$. Then, the BCP of SPEF-\TS{} (Algorithm~\ref{alg::TS_SPEF}) with a conjugate prior/posterior satisfies Equation~\eqref{eq::sampling_prob}, so the algorithm satisfies Assumption~\ref{ass::sampling_prob}. Furthermore, $c_n=\cO(n)$ and $C_n=\cO(\sqrt{n})$.
\end{lemma}

\begin{proof}
We directly obtain the result using Lemma 4 of \citet{kaufmann_bayesucb}, stating that for any $(\mu_{k,n}, \mu) \in [\mu_0^-, \mu_0^+]^2$ (note that in Algorithm~\ref{alg::TS_SPEF} the means are clipped if this is not the case) it holds that 
\begin{equation}\label{eq::bcp_spef}
An^{-1}e^{-n\kl(\mu_{k, n}, \mu)}\leq \bP(\widetilde \mu_{k, n} \geq \mu) \le B \sqrt{n}e^{-n\kl(\mu_{k, n}, \mu)} \;,
\end{equation}
for some positive constants $A, B$.
\end{proof}

The assumption that the means belong to a range $[\mu_0^-, \mu_0^+]$ may be reasonable in practice and avoids difficult cases in the limits of the support, that would require specific arguments for each SPEF. 

The polynomial terms in the upper and lower bounds may be improved, for instance by using that (see \citealp{kaufmann_bayesucb}) that if the prior has density $p$, then for any threshold $\mu_\star$ the BCP satisfies
\begin{equation}\label{eq::int_bcpts}\pi_n(\mu \geq \mu_\star) = \frac{\int_{\mu_\star}^{\mu_0^+} p(x) \exp \left(-n \kl(\mu_n, x)\right) \mathrm{d}x}{\int_{\mu_0^-}^{\mu_0^+} p(x) \exp \left(-n \kl(\mu_n, x)\right) \mathrm{d}x} \;,\end{equation}
%
%
and then using that (see e.g. \citealp{jin22expTS}) for any $\mu\leq \mu'$, $\kl(\mu, \mu')= \int_{\mu'}^\mu \frac{x-\mu}{V(x)}\mathrm{d}x$, and so 

\begin{equation}\label{eq::boundkl_exp}\frac{(\mu'-\mu)^2}{2V} \leq \kl(\mu, \mu') \leq \frac{(\mu'-\mu)^2}{2v}\;. \end{equation}
This quadratic form for the $\kl$ then allows to use results on Gaussian integrals to derive finer upper and lower bounds on the BCP.

\subsection{BCP for Gaussian TS with inverse-gamma prior}\label{app::ts_g} 

For the Gaussian Thompson Sampling with inverse-gamma prior (G-TS) we prove the following result.

\begin{lemma}[Bounds on the BCP for G-\TS] \label{lem::bcp_gts}
G-\TS with shape parameter $\alpha< 0$ satisfies Assumption~\ref{ass::relaxed_sp} with $D_\pi=\kinf^{\cF_G}$ 
and a sequence $(a_n)_{n\in \N}$ satisfying $a_n>\frac{1}{n-1}$ for all $n\geq 2$, as its Boundary Crossing Probability satisfies Equation~\eqref{eq::ub_sp} for the upper bound and Equation~\eqref{eq::lb_sp} for the lower bound. Furthermore, $c_n=\cO(\sqrt{n})$ and $C_n= \cO(\sqrt{n})$.
\end{lemma}

\begin{proof}
Let us fix a mean $\mu_\star$, and consider an arm with expectation $\mu\in \R$, standard deviation $\sigma$, and gap (w.r.t. $\mu_\star$) $\Delta=\mu_\star-\mu$, and denote by $\mu_n$ and $\sigma_n$ respectively the empirical mean and variance of an empirical distributions based on $n$ observations, and define $\Delta_n=\mu -\mu_n$. We prove the result using Lemma 4 of \cite{honda14}: 

\begin{equation}\label{eq::bcp_gts} \frac{1}{c_n} e^{-(n-1+2\alpha) \kinf^\cF(F_n, \mu_\star)} \leq \bP(\widetilde \mu_n \geq \mu_\star) \leq  \frac{1}{\sqrt{n}}\frac{\sigma_n}{\Delta_n} e^{-(n-2(1-\alpha)) \kinf^\cF(F_n, \mu_\star)}\;, \end{equation}
with $c_n=2e^{1/6}\sqrt{\pi\left(\frac{n}{2}+\alpha\right)}$, where $\alpha$ is the parameter corresponding to the inverse-gamma prior. First, Equation~\eqref{eq::lb_sp} is satisfied for a sequence $(a_n)_{n\in \N}$ if \[n-1+2\alpha \leq \frac{n}{1+a_n} \Longleftrightarrow \alpha \leq \frac{1}{2}\left(1- n \frac{a_n}{1+a_n}\right)\] 
for $n$ large enough. Hence, for the choice $a_n >\frac{1}{n-1}$ we obtain that $\alpha<0$ is sufficient to prove the result.

Next we show \eqref{eq::ub_sp} on the upper bound of BCP.\footnote{%
This argument does not allow us to recover the form of Assumption~\ref{ass::sampling_prob}, which is one of the reasons why we introduced Assumption~\ref{ass::relaxed_sp}.}
We can re-write the upper bound \eqref{eq::bcp_gts} as 
\begin{equation}\label{eq::bcp_gts_ub}\bP(\wt \mu_n \geq \mu_\star) \leq \underbrace{\frac{1}{\sqrt{n}}\frac{\sigma_n}{\Delta_n}e^{(1-\alpha)\log\left(1+\frac{\Delta_n^2}{\sigma_n^2}\ind(\sigma_n^2\leq \sqrt{n}, \Delta_n \leq \sqrt{n})\right)}}_{C_n(F_{n})} e^{-n\kinf(F_n, \mu_\star)}. \end{equation}
We then re-introduce the index of the arm $k$ considered in the notation. If $F_k$ has parameters $(\mu_k, \sigma_k^2)$ and sub-optimality gap $\Delta_k$, we can simply consider the set \[\cB(F_k)=\left\{F_{k,n} \in \cP: \; |\mu_{k,n} - \mu_k| \leq \frac{\Delta_k}{4} \;, \; \frac{\sigma_{k,n}}{\sigma_k} \in \left[\frac{1}{2}, \frac{3}{2}\right] \right\}\;.\]
We prove that this set satisfies the conditions of Assumption~\ref{ass::relaxed_sp}. Let us consider $\mu_\star \geq \mu_k+\frac{\Delta_k}{2}$. In that case, $\Delta_n \geq \frac{\Delta_k}{2}-\frac{\Delta_k}{4}=\frac{\Delta_k}{4}$, hence $\frac{\sigma_n}{\Delta_n}\leq \frac{3}{2}\times 4 \times \frac{\sigma_k}{\Delta_k} = 6$, which leads to
\[ F_{k,n} \in \cB(F_{k,n}) \Rightarrow C_n(F_{k,n})\leq 6\frac{\sigma_k}{\Delta_k} \frac{\left(1+4n\sigma_k^{-2}\right)^{1-\alpha}}{\sqrt{n}} \coloneqq C_n(F_k)  \;,\]
where $C_n(F_k)$ is clearly a constant that depends only of $n$ and $F_k$.
We thus conclude that Equation~\eqref{eq::ub_sp} is satisfied, for constants that satisfy the conditions of the assumption. This concludes the proof.

\end{proof}

Interestingly our result recovers the same restrictions on the prior as \cite{honda14}, that proved (in their Theorem 2) that the G-TS algorithm cannot reach logarithmic regret for $\alpha\geq 0$.


\subsection{BCP for NPTS for bounded distributions}\label{app::ts_b} 

We use the following result to analyze NP\TS{}, implemented according to Algorithm~\ref{alg::NPTS}.

\begin{lemma}[Bounds on the BCP for NP\TS]\label{lem::bcp_npts_B}
	NP\TS{} satisfies Assumption~\ref{ass::relaxed_sp} with $D_\pi=\kinf^\cF$, as its Boundary Crossing Probability satisfy
 Equation~\eqref{eq::ub_sp} (upper bound) and \eqref{eq::lb_sp} (lower bound) for $(c_n)_{n\in \N}$ that is polynomial in the parameter $\eta$ considered in the equation, and $C_n=\cO(n)$.
\end{lemma}

\begin{proof}
Our proof is based on Lemmas 15 (upper bound) and 14 (lower bound) in \citet{RiouHonda20}. In their Lemma 15, the authors prove that the BCP satisfies for any $\eta\in (0,1)$ and $\mu\geq \mu_{k,n}$

\[ \bP(\widetilde \mu_{k,n} \geq \mu) \leq \frac{1}{\eta} \exp\left(-n \left(\kinf^\cF( F_{k,n}, \mu) - \eta \frac{\mu}{B-\mu}\right)\right)\;. \]
Hence, we can use this result directly by choosing $\eta = 1/n$.  We then obtain 
\[ \bP(\widetilde \mu_{k,n} \geq \mu) \leq n e^{\frac{\mu}{B-\mu}} \exp\left(-n \kinf^\cF(F_{k,n}, \mu) \right)\;, \]
which is sufficient to prove that Equation~\eqref{eq::ub_sp} holds for any $n\in \N$.

Slightly adapting the lower bound 
presented in Lemma 14 of \cite{RiouHonda20}, we obtain that for $n$ large enough and $\mu\geq \mu_{k,n}$
\[ \bP(\widetilde \mu_{k,n} \geq \mu) \geq C n^{-\frac{M+2}{2}} \exp\left(-n \kinf^\cF(\widetilde F_{k,n}, \mu)\right)\;, \]
where $\widetilde F_{k,n}$ is a \emph{discretized} version of the empirical distribution $F_{k,n}$, considering the truncated observations $\widetilde X_{k, i}= \frac{1}{M} \left\lfloor MX_{k,i}  \right \rfloor$ for some integer $M\in \N$, and for some constant $C=\left(\frac{1}{\sqrt{2\pi}}\right)^M e^{-(M+1)/12}$.

Fixing $M$, we can control the deviation of $\kinf^\cF(\widetilde F_{k,n}, \mu)$ from $\kinf^\cF(\widetilde F_{k,n}, \mu)$. First, we remark that in the worst case $1/M$ is subtracted to all observations. In that case, reaching the mean $\mu$ is equivalent to reaching the mean $\mu+1/M$ with the initial (non-truncated) distribution and upper bound $B+1/M$. We formalize this by denoting by $\cF_B$ the family of distributions bounded by some constant $B$, and writing that 

\[ \kinf^{\cF_B}(\widetilde F_{k,n}, \mu) \leq \kinf^{\cF_{B+1/M}}(F_{k,n}, \mu+1/M) \leq \kinf^{\cF_{B}}(F_{k,n}, \mu+1/M)\;. \]

We conclude by using Lemma 5 in \citet{KL_UCB}:
\[\kinf^{\cF_{B}}(F_{k,n}, \mu+1/M) \leq \kinf^{\cF_{B}}(F_{k,n}, \mu) + \frac{1}{M(B-\mu)} \;. \]

Finally, to obtain the desired form of Equation~\eqref{eq::lb_sp} for any $\eta>0$ (different from the $\eta$ above) we simply choose $M = \frac{B-\mu}{\eta}$, which provides 

\[ \bP(\widetilde \mu_{k,n} \geq \mu) \geq \left(\frac{1}{\sqrt{2\pi}}\right)^\frac{B-\mu}{\eta} e^{-\left(\frac{B-\mu}{\eta}+1\right)/12} n^{-\left(\frac{B-\mu}{2\eta}+1\right)} \exp\left(-n (\kinf^\cF(F_{k,n}, \mu)+\eta) \right)\;. \]

This is sufficient to conclude that Equation~\eqref{eq::lb_sp} holds, and hence to verify Assumption~\ref{ass::relaxed_sp}. The polynomial $(c_n)$ depends on $\eta$, which is not a problem since a fixed value of $\eta>0$ is considered in the proofs where this lower bound is used. We can remark that we could also prove the lower bound of Assumption~\ref{ass::sampling_prob} directly by choosing $\eta=\cO(\sqrt{n})$.
\end{proof}

\section{Proof that Assumption~\ref{ass::sufficient_prop} holds for various families of distributions}\label{app::proofs_A1_A5}

In this appendix we prove all the results presented in Section~\ref{sec::applications} whose proofs were deferred to the appendix. More specifically, we prove that Assumption~\ref{ass::sufficient_prop} holds for a divergence close to $\kinf^\cF$ when $\cF$ belongs to one of the following families of distributions: 
\begin{itemize}
	\item Single Parameter Exponential Families (Appendix~\ref{subsec::SPEF}).
	\item Gaussian distributions (Appendix~\ref{app::gauss_p15}).
	\item Bounded distributions with a known upper bound $B$ (Appendix~\ref{app::bounded_P15}).
	\item Distributions with a (un-centered) $h$-moment condition (Appendix~\ref{app::h_P15}).
	\item Distributions with a centered $h$-moment condition (Appendix~\ref{app::cent_h_P15}).
\end{itemize}
Only in Appendix~\ref{app::gauss_p15}, we do not prove Assumption~\ref{ass::sufficient_prop} directly but a slight variant, and detail the arguments needed to adapt the result of Theorem~\ref{th::main_result}.

\subsection{Single-Parameter Exponential Families (SPEF)}\label{subsec::SPEF}

In Section~\ref{subsec::spef}, we introduced the definition and some elementary properties of this family of distributions. In particular, we defined the KL divergence, that we denote by $\kl$ when it takes as arguments the expectation of the two distributions. We now prove Lemma~\ref{lemm::cond_spef}.

\paragraph{(A1): } \hspace*{-\baselineskip} $\kl$ is known to be continuous in its second argument, non-decreasing (direct with \eqref{eq::kl_spef}), and $\kl(\mu, \mu')=0$ if $\mu\geq \mu'$ by definition.  

\paragraph{(A2): }\hspace*{-\baselineskip} we use the continuity of the $\kl$ in the first argument: for any $\delta>0$, there exists a constant $\delta'>0$ for which $\{x: \kl(x, \mu_1) \leq \kl(\mu_k, \mu_1) - \delta \} \subset \{x: x \geq \mu_k + \delta' \}$. Then, using \eqref{eq::chernoff} we obtain:
	\begin{align*}
		\bP(\kl(\mu_{k,n}, \mu_1) \leq \kl(\mu_k, \mu_1) -\delta) &\leq \bP(\mu_{k, n} \geq \mu_k + \delta') \\
		&\leq \exp\left(-n \kl(\mu_k + \delta', \mu_k)\right)\;, \\
\end{align*}
which is enough to prove (A2).

\paragraph{(A3) }\hspace*{-\baselineskip} holds with $\delta_{\epsilon, \mu}=\kl(\mu_1-\epsilon, \mu_1)$. Indeed, with \eqref{eq::kl_spef} for $\mu_n \leq \mu_1-\epsilon$ we can write that 

\begin{align*}
	\kl(\mu_n, \mu_1)- \kl(\mu_n, \mu_1-\epsilon) &= \int_{\mu_1-\epsilon}^{\mu_1} \frac{x-\mu_n}{V(x)} \mathrm{d}x \\
	& = \int_{\mu_1-\epsilon}^{\mu_1} \frac{x-(\mu_1-\epsilon)}{V(x)} \mathrm{d}x + \int_{\mu_1-\epsilon}^{\mu_1} \frac{\mu_1-\epsilon-\mu_n}{V(x)} \mathrm{d}x\\
	& = \kl(\mu_1-\epsilon, \mu_1) + (\mu_1-\epsilon-\mu_n) \int_{\mu_1-\epsilon}^{\mu_1} \frac{1}{V(x)} \mathrm{d}x \\
	& \geq \kl(\mu_1-\epsilon, \mu_1) \;.
\end{align*}

Finally, \textbf{(A4)} also holds directly with $\alpha_{n}=1$ using Chernoff inequality \eqref{eq::chernoff}, by mapping the level sets of $\kl$ with level sets for the mean.

\subsection{Gaussian distributions with unknown variances}\label{app::gauss_p15}

In this section we prove Lemma~\ref{lemm:cond_gauss} and Corollary~\ref{cor::G}. This case is interesting because the family is still parametric but the fact that there are two parameters make the proofs less direct than for SPEF. We first prove the lemma, and then detail how the proof scheme of Theorem~\ref{th::first_bound} can still be used to prove Corollary~\ref{cor::G}.

\paragraph{Proof of Lemma~\ref{lemm:cond_gauss}} Let us consider a distribution $F_n$, that is the empirical distribution of $n$ observations, and denote by $\mu_n$ and $\sigma_n^2$ respectively its expectation and variance, clipped to the range $[\mu_0^-, +\infty)$ and $(0, V]$. Then, for any $\mu \in \R$ we recall that
\[\kinf^\cF(F_n, \mu) = \frac{1}{2}\log\left(1+\frac{(\mu-\mu_n)^2}{\sigma_n^2}\right)\ind(\mu\geq \mu_n)\;, \]
which is obtained by computing the KL-divergence between the distribution $\cN(\mu_n, \sigma_n^2)$ and the distribution $\cN(\mu, \sigma_n^2+(\mu-\mu_n)^2)$. We recall that we consider $D_\pi=\kinf^\cF$. First, the function is clearly continuous and non-decreasing in $\mu\geq \mu_n$, and $\kinf^\cF(F_n, \mu_n)=0$ by definition, so \textbf{(A1)} holds. 

\paragraph{(A2) } \hspace*{-\baselineskip} does not require a tight concentration inequality. We use a similar argument as in Section~\ref{subsec::SPEF}: for a fixed $\delta>0$, if $D_\pi(F_{k,n}, \mu_1)\leq D_\pi(F_k, \mu_1) - \delta$ then there exists two constants $\delta_1>0$ and $\delta_2>0$, such that either $|\mu_{k,n}-\mu_k|>\delta_1$ or $\sigma_{k,n}/\sigma_k\geq 1+\delta_2$. Essentially, at least one of the two empirical parameters is deviating from its true value. Hence, using Chernoff inequality for the means and Lemma 7 of \citet{chan2020multi} for the standard deviation we obtain that
\begin{align*}
	\bP\left(D_\pi(F_{k,n}, \mu_1) \leq D_\pi(F_{k}, \mu_1) -\delta\right) &\leq \bP(\mu_{k,n}\geq \mu_k+\delta_1 ) + \bP\left(\frac{\sigma_{k, n}}{\sigma_k} \geq 1+\delta_2 \right)\\
	& \leq e^{-n \frac{\delta_1^2}{2\sigma_k^2}} + e^{\frac{n-1}{2}(\log(1+\delta_2)-\delta_2) }, \\
\end{align*}
which proves that (A2) holds, since $\log(1+\delta_2)-\delta_2<0$ (it is furthermore equivalent to $\frac{\delta_2^2}{2}$ for small $\delta_2$). Note that for this property the arguments are unaffected by clipping.

\paragraph{(A3) } \hspace*{-\baselineskip} is the condition that requires clipping the empirical parameters. Indeed, for any $\mu\geq \mu_n$, using the notation $\Delta_n=\mu-\mu_n$ and that $\Delta_n\geq \epsilon$ we can first obtain the following lower bound:
\begin{align*}\kinf^\cF(F_{n}, \mu_1)- \kinf^\cF(F_{n}, \mu_1-\epsilon) 	
	&= -\frac{1}{2}\log\left(1+\frac{\epsilon^2-2\Delta_n\epsilon}{\Delta_n^2+\sigma_n^2}\right)\\
& \geq \frac{1}{2}\times \frac{2\Delta_n \epsilon-\epsilon^2}{\Delta_n^2 + \sigma_n^2} \;.
 \end{align*}
Unfortunately this term can be arbitrarily small if $\Delta_n^2$ or $\sigma_n^2$ are too large. With clipped estimates, we can however use that $\sigma_n^2\leq V$ and $\Delta_n \leq R_\mu\coloneqq \mu-\mu_0^-$. To simplify the analysis, we consider two case. First, if $\Delta_n \geq \sigma_n$ we obtain that 
\begin{align*}
\frac{1}{2}\times \frac{2\Delta_n \epsilon-\epsilon^2}{\Delta_n^2 + \sigma_n^2}& \geq \frac{1}{4}\times \frac{2\Delta_n \epsilon-\epsilon^2}{\Delta_n^2} \\
& \geq \frac{1}{4}\times \left(2\frac{\epsilon}{\Delta_n}-\left(\frac{\epsilon}{\Delta_n}\right)^2\right) \\ 
& \geq \frac{1}{4}\left(\frac{\epsilon}{R} + \frac{\epsilon}{\Delta_n}\left(1-\frac{\epsilon}{\Delta_n}\right) \right) \geq \frac{\epsilon}{4R} \;.\\
\end{align*}
This gives a first part of the result. For the second part, we consider $\Delta_n\leq \sigma_n$, and simply obtain that 
\[\frac{1}{2}\times \frac{2\Delta_n \epsilon-\epsilon^2}{\Delta_n^2 + \sigma_n^2} \geq \frac{1}{4}\frac{\epsilon^2}{\sigma_n^2} \geq \frac{\epsilon^2}{4V} \;, \]
which proves (A3) with $\delta_{\epsilon, \mu} = \frac{\epsilon^2}{4\left\{R\vee V\right\}}$.

\paragraph{(A4): } \hspace*{-\baselineskip} Contrarily to (A2), we need this time a tight concentration inequality for $\kinf^\cF(F_n, \mu_1)$, and hence cannot handle separately the mean and variance. To do that, we directly express the distribution of $\kinf^\cF$ according to the Student distribution. We present our result in Lemma~\ref{lem::conc_gaussian} below, which is (to the best of our knowledge) a novel result of independent interest.

\begin{lemma}[Concentration of $\kinf^\cF(F_{n}, \mu_1)$ for Gaussian distributions]\label{lem::conc_gaussian}
If $\cF$ is the family of gaussian distributions, $F\in \cF$ has mean $\mu_1$ and $F_n$ is the empirical distribution computed from $n$ i.i.d. observations from $F$, then for any $n\geq 2$, $x_0 > 0$ and $x\geq x_0$ it holds that 
\[\bP\left(\kinf^\cF(F_n, \mu_1) \geq x \right) \leq \sqrt{\frac{8}{\pi(1-e^{-2x_0})}} e^{-(n-1)x} \;. \]
\end{lemma}

\begin{proof} We know that $S_n = \sqrt{n} \frac{\Delta_n}{\sigma_n}$ follows a Student distribution with $n-1$ degrees of freedom. We write the probability in (A3) in terms of $S_n$, and obtain 

\[\bP(\kinf^\cF(F_{n}, \mu_1) \geq x) = \bP\left(S_n \geq \sqrt{n}\sqrt{e^{2x}-1}\right)\;.\]

We then compute the density of $\kinf^\cF(F_n, \mu_1)$ according to the density of the Student distribution with $n-1$ degrees of freedom, denoted by $f_n$, using the formula for monotonous transformations of continuous random variables: 

\begin{align*}
	\frac{\mathrm{d}\bP(\kinf^\cF(F_{n}, \mu_1) \leq x)}{\mathrm{d}x} = \sqrt{n}\frac{e^{2x}}{\sqrt{e^{2x}-1}}\times f_n(\sqrt{n(e^{2x}-1)}) \;.
\end{align*}

Now we use that $f_n(y) = C_n e^{-\frac{n}{2}\log\left(1+\frac{y^2}{n-1}\right)}$, with $C_n = \frac{\Gamma(n/2)}{\sqrt{(n-1)\pi}\Gamma((n-1)/2)}$, and obtain
\begin{equation}\label{eq::student_bound}
f_n(\sqrt{n(e^{2x}-1)}) = C_n e^{-\frac{n}{2}\log\left(1+\frac{n(e^{2x}-1)}{n-1}\right)} \leq C_n e^{-nx} \;.
\end{equation}
We then consider $x\geq x_0>0$, and obtain
\begin{align*}
	\frac{\mathrm{d}\bP(\kinf^\cF(F_{n}, \mu_1) \leq x)}{\mathrm{d}x} \leq \sqrt{\frac{n}{1-e^{-2x}}} C_n e^{-(n-1) x} \leq  \sqrt{\frac{n}{1-e^{-2x_0}}} C_n e^{-(n-1) x} \;.
\end{align*}
By integrating this bound (since all the terms are positive), we obtain that for $x\geq x_0$, 
\begin{align*}
  \bP(\kinf^\cF(F_n, \mu_1)\geq x)&=\int_{x}^{+\infty}  \mathrm{d}\bP(\kinf^\cF(F_{n}, \mu_1) \leq y) \\
  & = \int_{x}^{+\infty}  \frac{\mathrm{d}\bP(\kinf^\cF(F_{n}, \mu_1) \leq y)}{\mathrm{d}y} \mathrm{d}y  \\
  & \leq \int_{x}^{+\infty}  
 \sqrt{\frac{n}{1-e^{-2x_0}}} C_n e^{-(n-1) y} \mathrm{d}y \\
 &=\sqrt{\frac{n}{1-e^{-2x_0}}} \frac{C_n}{n-1}e^{-(n-1) x} \;.
\end{align*}
We now analyze more precisely $C_n$. Using a basic property of the Gamma function, we have that $\Gamma(n/2) = (n/2) \Gamma((n-2)/2)$, so $C_n \leq \frac{n}{2\sqrt{(n-1)\pi}}$. Hence, for $n\geq 2$ it holds that
\[\sqrt{n} \frac{C_n}{n-1} \leq \frac{n^{3/2}}{\sqrt{\pi} (n-1)^{3/2}} \leq \frac{1}{\sqrt{\pi}} \left(\frac{n}{n-1}\right)^\frac{3}{2} \leq \sqrt{\frac{8}{\pi}} \;, \]
which gives the final result. 
\end{proof}

We remark that the result does not directly provide (A4) as the bound decreases exponentially with $n-1$ instead of $n$, and a multiplicative sequence that depend on the choice of a constant $x_0>0$. This completes the proof of Lemma~\ref{lemm:cond_gauss}, and we now discuss the adaptation of the proof of Theorem~\ref{th::main_result}. 

\paragraph{Proof of Corollary~\ref{cor::G}} We prove with the following Lemma~\ref{lem::A4_gauss} that this bound is sufficient to obtain the same result as when (A4) directly holds in the regret analysis, with the only requirement that $a_n > \frac{1}{n-1}$. 

\begin{lemma}[Upper bounding $B_2$ (proof of Th.~\ref{th::main_result}) under a slight relaxation of (A4)] \label{lem::A4_gauss}
Let us assume that there exists some constant $b>0$ such that, for any $x_0>0$, there exists a sequence $(\alpha_n(x_0))_{n\leq x_0}$ such that for $n$ large enough it holds that  \begin{equation}\label{eq::alternative_A4}
    \forall x \geq x_0, \;\bP(D_\pi(F_n, \mu_1) \geq x ) \leq \alpha_n(x_0) e^{-(n-b)x}
\end{equation} 
Then, for $n>b$ such that \eqref{eq::alternative_A4} holds, if $a_n>\frac{b}{n-b}$ it holds that 
\[\bE\left[e^{n \frac{D_\pi(F_{1,n}, \mu_1)}{1+a_n}}-1\right] \leq e^{\frac{nx_0}{1+a_n}} + \frac{\alpha_n(x_0)}{a_n'}, \; \]
with $a_n' = a_n\left(1-\frac{b}{n}\right)-\frac{b}{n}$
\end{lemma}

\begin{proof} We start with the same re-formulation as in the proof of Theorem~\ref{th::main_result},
\begin{align*}
    B_{2,n} \coloneqq \bE\left[e^{n \frac{D_\pi(F_{1,n}, \mu_1)}{1+a_n}}-1\right] & = \int_{0}^{+\infty} \bP\left(e^{n \frac{D_\pi(F_{n}, \mu_1)}{1+a_n}}-1 > x\right) \mathrm{d}x \\
	& =  \int_{0}^{+\infty} \bP\left(D_\pi(F_{n}, \mu_1)> \frac{1+a_n}{n}\log\left(1+x)\right)\right) \mathrm{d}x  \\
 \end{align*}
Let us now consider some $x_0>0$, that we will tune later. We then simply write that
\begin{align*}
B_{2,n} & \leq e^{\frac{nx_0}{1+a_n}}-1+ \int_{e^{\frac{nx_0}{1+a_n}}-1}^{+\infty} \bP\left(D_\pi(F_{n}, \mu_1)> \frac{1+a_n}{n}\log\left(1+x)\right)\right) \mathrm{d}x \\
& \leq e^{\frac{nx_0}{1+a_n}} + \int_{0}^{+\infty} \alpha_n(x_0) \frac{1}{(1+x)^{1+a_n\left(1-\frac{b}{n}\right)-\frac{b}{n}}} \mathrm{d}x \;. \\ 
\end{align*}
Then, we define $a_n'=a_n\left(1-\frac{b}{n}\right)-\frac{b}{n}$, and obtain that the integral is finite and bounded by $\frac{\alpha_n(x_0)}{a_n'}$ if $a_n'>0$. We can verify that this is the case if $a_n>\frac{b}{n-b}$, as stated in the lemma. In that case, it holds that 
\begin{align*}
    B_{2,n} & \leq e^{\frac{nx_0}{1+a_n}} + \frac{\alpha_n(x_0)}{a_n'} \;.
\end{align*}
This concludes the proof of the lemma.
\end{proof}

To prove Corollary~\ref{cor::G}, we can then plug this result back into the analysis of Theorem~\ref{th::main_result}, remarking that the term $B_2$ is upper bounded by a constant if $\sum_{n=1}^{+\infty} c_n e^{-n \frac{\delta_{\epsilon, \mu_1}}{2(1+a_n)}}B_{2,n}$ is finite (Eq.~\eqref{eq::bound_B2_delta}). This is the case if we choose $x_0=\frac{\delta_{\epsilon, \mu_1}}{4}$. Hence, Lemma~\ref{lem::A4_gauss} is sufficient to prove the asymptotic optimality of the policies satisfying the slightly relaxed condition (A4).

\subsection{Bounded distributions with a known upper bound $B$}\label{app::bounded_P15}

For this family of distributions we can directly use several results from the literature in order to prove Lemma~\ref{lemm::bounded_div}, establishing that Assumption~\ref{ass::sufficient_prop} holds for the optimal divergence on $\cF_B$.

\paragraph{(A1) } \hspace*{-\baselineskip} is proved by Theorem 7 in \citet{HondaTakemura10}.
\paragraph{(A2): } \hspace*{-\baselineskip} we propose as in previous section to use continuity arguments (in the first argument of $\kinf^\cF$, see Theorem 7 in \citealp{HondaTakemura10}). For a fixed $\delta>0$, if $\kinf^\cF(F_{k,n}, \mu) \leq \kinf^\cF(F_{k}, \mu)-\delta$ then it is necessary that $||F_{k, n}-F_k||_\infty \coloneqq \sup_{x}|F_{k, n}(x)-F_k(x)|\geq \delta_1$ for some $\delta_1>0$. Hence, using the DKW inequality (see e.g \cite{massart1990}) we use that \[\bP(\kinf^\cF(F_{k,n}, \mu) \leq \kinf^\cF(F_{k}, \mu)-\delta) \leq \bP\left(||F_{k, n}-F_k||_\infty \geq \delta_1\right) \leq 2 e^{-2n\delta_1^2}\;,\] 
which proves that (A2) holds.  
\paragraph{(A3) } \hspace*{-\baselineskip} is proved by Lemma 13 in \citet{HondaTakemura10}, with $\delta_{\epsilon, \mu} = \frac{\epsilon^2}{2(\mu-\epsilon-b)(B-\mu)}$ if the support is $[b, B]$.
\paragraph{(A4) } \hspace*{-\baselineskip} is proved by Lemma 6 in \citet{KL_UCB}, with $\alpha_{n}=e(n+2)$.

\subsection{Distributions satisfying an uncentered $h$-moment condition} \label{app::h_P15}

Consider a positive convex function $h$ satisfying $h(0)=0$ and $\frac{h(|x|)}{|x|^{1+\eta}} \rightarrow + \infty$ for some $\eta>0$. Then, a family $\cF$ of distributions satisfying an uncentered $h$-moment condition, that we recall is defined as follows
 \begin{equation*}\cF = \left\{F \in \cP:\; \bE\left[h(|X|)\right] \leq B \;, \; \bE[h(|X|)^{2+\epsilon_F}]<+\infty \text{ for some } \epsilon_F>0 \right\} \;,\end{equation*}
for a known constant $B>0$. We also recall that we consider 
\begin{align*}D_\pi(F, \mu) \coloneqq \kinf^\cF(F, \mu)\times \ind(F \in \cF)\;,\end{align*}

As this kind of family of distribution is not much studied in the literature yet, we recall for the unfamiliar reader some important properties from \citet{agrawal2020optimal}, that we will use to prove our results. We denote by $h^{-1}$ the inverse of $h$ on $\R^+$.

\begin{theorem}[Theorem 12 and Lemma 4 in \citealp{agrawal2020optimal}]\label{th::theorem_12_agrawal}
	For any $F \in \cP$ and $\mu$ such that $\bE_F[X] \leq \mu$ and $h(|\mu|)\leq B$, it holds that 
	\[\kinf^\cF(F, \mu) = \max_{(\lambda_1, \lambda_2) \in \cR_2} \bE_{F}\left[\log\left(1-(X-\mu)\lambda_1 - (B-h(|X|))\lambda_2\right)\right]\;, \]
	where $\cR_2= \left\{(\lambda_1, \lambda_2) \in (\R^2)^+: \forall y \in \R, 1-(y-\mu) \lambda_1 - (B-h(|y|))\lambda_2 \geq 0 \right\}$. Then, the following properties hold,
	\begin{itemize}
		\item The maximum is attained at a unique $(\lambda_1^\star, \lambda_2^\star)\in \cR_2$.
		\item For $\mu \geq \bE_F[X]$ and $F\in \cF$ or for $\mu=\bE_F[X]$ and $F \notin \cF$ any distribution $G^\star$ achieving infimum in the primal problem satisfies $\bE_{G^\star}[X]=\mu$ and $\bE_{G^\star}[h(|x|)]=B$. Furthermore, \[\frac{dF}{dG^\star}=\frac{1}{1-(X-\mu)\lambda_1^\star - (B-h(|\mu|))\lambda_2^\star}\;. \]
		\item If $\int_{y \in \text{Supp}(F)} dG^\star <1$, then $\text{Supp}(G^\star)= \text{Supp}(F) \cup \{y_0\}$, where \[1-(y_0-\mu) \lambda_1^\star - (B-h(|y_0|))\lambda_2^\star = 0
		\;.\]
		\item (Lemma 4) The set $\cF$ is uniformly integrable and compact in the Wasserstein metric. For $\mu\leq h^{-1}(B)$ and $F\in \cF$ $\kinf^\cF(F,\mu)$ is increasing, the mapping $\mu \mapsto \kinf^\cF(F,\mu)$ is continuous, convex, and twice differentiable, and the mapping $F \mapsto \kinf^\cF(F,\mu)$ is continuous and convex (under the Wasserstein metric). Furthermore, for $F\in \cF$, $\kinf^\cF(F, \mu) \leq \cK_\cF^+\coloneqq \log\left(\frac{h^{-1}(B)}{2(h^{-1}(B)-\mu)}\right)$ (achieved for the Dirac distribution in $-h^{-1}(B)$), $\kinf^\cF(F, \bE_F[X])=0$ and $\frac{\partial \kinf^\cF(F, \bE_F[X])}{\partial \mu}=0$.
	\end{itemize}
\end{theorem}

We now provide another important result for our proof, that motivates the second moment condition in the definition of $\cF$.

\begin{lemma}[part of Theorem 2 in \citealp{fournier_wasserstein}] \label{lem::wasserstein}
Let $X_1,\dots, X_n$ be a sequence of i.i.d. random variables drawn from a distribution $F$ satisfying $\bE[|X|^a]<+\infty$ for some $a>2$. Denote by $F_n$ the empirical distribution corresponding to $X_1,\dots, X_n$. Then the Wasserstein distance between $F_n$ and $F$, denoted by $W(F_n, F)$, satisfies for any $x \geq 0$
\[\bP\left(W(F_n, F) \geq x\right) = \cO(n^{-(a-1)}) \;. \] 
\end{lemma}

For any $\delta>0$, we know that $W(F_n, F)\leq \delta \Rightarrow |\bE_{F_n}[X]- \bE_F[X]| \leq \delta$. Hence, Lemma~\ref{lem::wasserstein} can be used directly to get under our assumptions that $\sum_{n=1}^{+\infty}\bP(F_{k,n}\notin \cF)= \cO(1)$ for any arm $k$. Furthermore, since $h(|X|)^{2+\epsilon_F}\geq |X|^{2+2\eta+\eta \epsilon_F}$ for some $\eta>0$ then $\sum \bP(\mu_{1,n}\leq \mu_1-\epsilon)$ converges for any $\epsilon>0$. 

We can now prove that (A1)--(A4) hold for the family $\cF$ defined in \eqref{eq::uncentered_main}.

\paragraph{(A1) } \hspace*{-\baselineskip} is proven by Theorem~\ref{th::theorem_12_agrawal}.

\paragraph{(A2): } \hspace*{-\baselineskip} We already proved that $\sum_{n=1}^{+\infty}\bP(F_{k,n}\notin \cF)= \cO(1)$ for any arm $k$. We now consider possible deviations of the empirical $\kinf^\cF$ by some $\delta>0$. As usual, we resort to the continuity of $\kinf^\cF$ in its first argument: for any $\mu>\mu_k, \delta >0$, if $\kinf^\cF(F_{k, n}, \mu)\leq \kinf^\cF(F_k, \mu)-\delta$ then there exists $\delta_1>0$ such that $W(F_{k, n}, F_k)\geq \delta_1$. Using Lemma~\ref{lem::wasserstein}, we obtain that $\sum_{n=1}^{+\infty}\bP(W(F_{k, n}, F_k)\geq \delta_1) = \cO(1)$, which concludes the proof that (A2) holds.

\paragraph{(A3): } \hspace*{-\baselineskip} First, using the convexity of $\kinf^\cF$ in its second argument we obtain that for any $F\in\cF$,
\[\kinf^\cF(F, \mu_1)-\kinf^\cF(F, \mu_1-\epsilon)\geq \kinf^\cF(F, \mu_F+\epsilon) -\kinf^\cF(F, \mu_F) = \kinf^\cF(F, \mu_F+\epsilon)  >0\;.\]
Then, we simply use the compactness of the family $\cF$ (see Theorem~\ref{th::theorem_12_agrawal}) to prove (A3) with $\delta_{\epsilon, \mu_F}=\min_{F\in \cF: \bE_F[X]\leq \mu_1-\epsilon}\kinf^\cF(F, \mu_F+\epsilon)>0$ (which is achieved by some distribution $F$).

\paragraph{(A4) } \hspace*{-\baselineskip} is proven by Theorem 11 in \cite{agrawal2020optimal}, with $\alpha_{n}\coloneqq (n+1)^2 \alpha$ for some constant $\alpha>0$.

\subsection{Distribution satisfying a centered $h$-moment condition}\label{app::cent_h_P15}

We now prove Lemma~\ref{lemm::cond_h_cent}, on the families of distributions satisfying a centered $h$-moment condition. 
As in the uncentered case, we consider a positive convex function $h$ satisfying $h(0)=0$ and $\frac{h(|x|)}{|x|^{1+\eta}} \rightarrow + \infty$ for some $\eta>0$. Then, a family $\cF$ of distributions satisfying a \emph{centered} $h$-moment condition is defined as follows
\begin{equation*}
\begin{aligned}
\cF = \{F \in \cP:\; &\bE_F\left[h(|X-\bE_F[X]|)\right] < B \;, \bE_F[X] \geq \mu_0^- \;,\\
&\bE_F[h(|X-\bE_F[X]|)^{2+\epsilon_F}]<+\infty \text{ for some } \epsilon_F>0 \} \;,
\end{aligned}
\end{equation*}
for known constants $B>0, \mu_0^-\in \R$. Compared with the previous section (uncentered case), we add the requirement that the means are lower bounded by a known constant $\mu_0^-$, in order use arguments based on compactness. For this problem, the lower bound on the regret is defined by the following function,
\[\kinf^\cF : (F, \mu) \mapsto \inf_{G\in \cP:\; \bE_G [X]\geq \mu\;,\; \bE_F[h(|X-\bE_G[X])|)]< B} \KL(F, G) \;. \]
This setting is more difficult than the uncentered case, because the optimization problem is in general not convex anymore, and working directly on $\kinf^\cF$ is intricate. However, we can express in a form that is closer to the $\kinf$ function for uncentered conditions,
\[
\kinf^\cF(F, \mu_1) = \inf_{\mu \in [\mu_1, +\infty)} \cK(F, \mu)\;, \text{with} \; \cK(F, \mu)=\inf_{G \in \cF} \KL(F, G) \quad \text{s.t } \begin{array}{ll} \bE_G[X] = \mu \\ \bE_G[h(|X-\mu|)] \leq B \end{array} \;.\]
The proof is direct by considering the distribution $G^\star\in \cF$ for which $\kinf^\cF(F, \mu_1)=\KL(F, G^\star)$. The main interest of this transformation is to fix $\bE_F[X]$ inside the $h$-moment condition, obtaining something similar to the uncentered case. We consider the properties of this quantity for a distribution $F \in \cF$, with $\bE_{F}[X]\coloneqq \mu_F \leq \mu_1$. First, $\cK$ can be expressed very similarly to the $\kinf^\cF$ function for the uncentered $h$-moment condition, with 
\[\cK(F_n, \mu) = \max_{(\lambda_1,\lambda_2)\in \cR} \bE_{F_n}\left[\log(1-\lambda_1 (X-\mu)-\lambda_2(B-h(|X-\mu|)\right] \;, \]
where $\cR = \{(\lambda_1,\lambda_2) \in \R \times \R^+ \;: \;\forall x \in \R \;,  1-\lambda_1 (x-\mu)-\lambda_2(B-h(|x-\mu|))\geq 0 \}$. Hence, for fixed $(F_n, \mu)$ all the properties that were derived for the uncentered conditions can be easily translated to the function $\cK$. Then, the natural question it to ask whether the minimum over $\mu$ is achieved in $\mu_1$ or not. We prove that this is the case by writing $\cK$ in terms of the $\kinf$ function for the family satisfying the \emph{un-centered} $h$-moment condition. First, we denote by $F_{-\mu}$ the cdf of the translation by $-\mu$ of the distribution $\nu_F$ of cdf $F$, $F_{-\mu}: x \mapsto F(x+\mu)$ (e.g. if $F$ is the empirical cdf of $X_1,\dots, X_n$ then $F_{-\mu}$ is the empirical cdf of $X_1-\mu, \dots, X_n-\mu$). Now, let us denote by $G^\star$ the distribution satisfying $\cK(F,\mu) = \KL(F, G^\star)$. First, it directly holds that $\KL(F, G^\star) = \KL(F_{-\mu}, G_{-\mu}^\star)$. Then, by definition $G^\star$ satisfies $\bE_{G^\star}[X] = \mu$ , so $\bE_{G_{-\mu}^\star}[X]=0$, and $\bE_{G^\star}[h(|X-\mu|)] \leq B$ so $\bE_{G_{-\mu}^\star}[h(|X|)] \leq B$. Hence, if we denote by $\cF_{\text{uc}}$ the family defined by the uncentered $h$-moment condition corresponding to same function $h$ and constant $B$ as $\cF$ it holds that $G_{-\mu}^\star\in \cF_{\text{uc}}$. Then, using from Theorem~\ref{th::theorem_12_agrawal} that the constraints are binding in the solutions of $\kinf^{\cF_{\text{uc}}}$ we finally obtain that for any cdf $F$ and threshold $\mu$, 
\[\cK(F, \mu) \coloneqq \KL(F, G^\star)= \KL(F_{-\mu}, G_{-\mu}^\star) =  \kinf^{\cF_{\text{uc}}}(F_{-\mu}, 0)  \;, \]
and with this notation $\kinf^{\cF_{\text{uc}}}(F_{-\mu_F}, 0) = 0$.
With this re-writing and because $F\in \cF$, it now becomes clear that $\cK(F, \mu)$ is increasing in $\mu$ for $\mu\geq \mu_F$, since the objective remains the same and the distribution shifts further away from it.
Hence, we obtain that 
\[\kinf^{\cF}(F, \mu) =  \cK(F, \mu) =  \kinf^{\cF_{\text{uc}}}(F_{-\mu}, 0)  \;.\]

As a consequence, we can use several of the continuity and concentration properties from the previous section to prove that Assumption~\ref{ass::sufficient_prop} holds for
\[D_\pi(F, \mu) = \cK(F, \mu) \ind(\mu\geq \mu_F, F \in \cF) \;. \]

\paragraph{(A1) } \hspace*{-\baselineskip} holds because of the continuity of $\kinf^{\cF^{uc}}$ with respect to its first argument (Theorem~\ref{th::theorem_12_agrawal}) under the Wasserstein metric $W$. Indeed,  $\mu' \rightarrow \mu \Rightarrow W(F_{-\mu'}, F_{-\mu}) \rightarrow 0$. It is also non-decreasing, and $D_\pi(F, \mu) = 0$ if $\mu_F\geq \mu$.

\paragraph{(A2)} \hspace*{-\baselineskip} holds thanks to the same arguments as for the uncentered case, using upper bounds on $\bP(W(F_n, F)\geq \delta)$ for $\delta>0$, provided by Lemma~\ref{lem::wasserstein} and the definition of the family. 

\paragraph{(A4): } \hspace*{-\baselineskip} The concentration inequality of Theorem 11 in \cite{agrawal2020optimal} can still be used, so as in the uncentered case (A4) holds with $\alpha_n=\cO((n+1)^2)$.




\paragraph{(A3): } \hspace*{-\baselineskip} Using the same techniques as for the proof of Theorem 12 in \cite{agrawal2020optimal} for $\cF^{\text{uc}}$ we can prove that the intersection of $\cF$ and the set of distributions with mean bounded by $[\mu_0^-, \mu_1-\epsilon]$ is uniformly integrable and compact. Hence, with the same compactness argument as for $\cF^{\text{uc}}$, we conclude (thanks to the assumptions that all means are larger than $\mu_0^-$) on the existence of \[\delta_{\epsilon, \mu} = \min_{F \in \cF: \bE_F[X]\leq \mu_1-\epsilon} \kinf^\cF(F, \mu)-\kinf^\cF(F, \mu-\epsilon) > 0  \;. \]

This completes the proof that (A1)--(A5) hold for the centered $h$-moment condition. We see that the main difficulty of this setting was to express the optimization problem giving the $\kinf^\cF$ function for this setting into the more convenient form provided by uncentered conditions.


\subsection{A variant of the sub-gaussian case as a centered $h$-moment condition}\label{app::SG}

In this appendix we explain that the centered $h$-moment condition with $h(x) = e^{\frac{x^2}{s}}$ for some $s>0$ and a constant $B>0$ can be viewed as an alternative characterization of subgaussian families of distributions.
It is known that there are many characterization of subgaussian tails (see Proposition 2.5.2 of \citealp{vershynin2018high}).
Among them, the most typical definition of subgaussian distributions in the bandit literature is 
\begin{align*}
\sgone(\sigma)&=
\left\{
F: \forall t\in \R,\,\E_F[\e^{tX}]\le \e^{t\mu_F+\frac{\sigma^2t^2}{2}}
\right\} \;,
\end{align*}
where $\mu_F$ denotes the expectation of the distribution of cdf $F$. 

Interestingly, when asymptotic optimality in the sense of \eqref{eq::asymp_opt_def} is not targeted this definition is very convenient since the Chernoff method easily provides some concentration inequalities for the empirical means, which is sufficient for the analysis of the algorithm tackling subgaussian distributions in the literature. However, as the family is defined by infinitely many constraints (due to the parameter $t$) the function $\kinf^{\text{SG}_1(\sigma)}$ has a complicated form. To target asymptotic optimality, it may be easier to work with an alternative characterization, using that sub-Gaussianity implies the following $h$-moment condition:
\[ F\in \sgone(\sigma) \Rightarrow \E_F\left[\e^{\frac{(X-\mu_F)^2}{8\sigma^2}}\right]\le 2 \;. \]
Hence, we can consider a subgaussian model through the following family of distribution
\begin{align*}
    \sgthree(\sigma)&=
\left\{F: \E_F\left[\e^{\frac{(X-\mu_F)^2}{8\sigma^2}}\right]\le 2, \; \E_F\left[\e^{\frac{(X-\mu_F)^2}{3\sigma^2}}\right]< +\infty \right\}\n \;,
\end{align*}
that fits into the scope of our analysis if we intersect it with the set $\{F\in \cP: \; \bE_F[X]\geq \mu^-\}$ for any constant $\mu^-\in \R$.
We can easily verify that $F \in \sgone(\sigma)$ satisfies the above second moment condition since
$\bP(|X|\geq t)\leq 2e^{-\frac{t^2}{2\sigma^2}}$
for $X \sim F$ and all $t>0$,
and thus 
\[\E_F\left[\e^{\frac{(X-\mu_F)^2}{3\sigma^2}}\right]\leq 1+\int_1^{+\infty} \bP(|X-\bE_F[X]|>\sqrt{3\sigma^2 \log(t)}) \mathrm{d}t \leq 1+ 2 \int_1^{+\infty} t^{-3/2} \mathrm{d}t <+\infty \;. \]
Then, we can easily verify that \[ \sgone(\sigma)\subset \sgthree(\sigma)\subset \sgone(\sqrt{24}\sigma)\;, \]
and so even though $\sgone$ and $\sgtwo$ do not exactly match, we are sure that all distributions in $\sgtwo(\sigma)$ are still subgaussian (for a potentially larger parameter).
Similar relations can also be seen in many other characterizations of subgaussian distributions (see \citealp{vershynin2018high}).
\section{NPTS for centered $h$-moment conditions}

In this appendix we prove Lemma~\ref{lem::bounds_bcp}, with the proof for the upper bound on the BCP in Appendix~\ref{app::ub_npts} and the proof for the lower bounds in Appendix~\ref{app::lb_npts}. In Appendix~\ref{app::lb_bcp_bounded} we provide a novel lower bound on the BCP of NPTS, for bounded distributions, that showcases the general intuition of the proof of Appendix~\ref{app::lb_npts}. 

\subsection{Upper bound}\label{app::ub_npts}

We start by proving the upper bound presented in Lemma~\ref{lem::bounds_bcp}. To do that, we use the notation $Z_i= h(|X_i-\mu|)$ for $i \in [n]$ and $Z_{n+1}=h(|X_{n+1}-\mu|)-\gamma$, and consider \[\text{(P)}\coloneqq \bP\left(\exists X_{n+1}: \sum_{i=1}^{n+1} w_i X_i \geq \mu ,\sum_{i=1}^{n+1} w_i Z_i \leq B \right)\;.\] 
The general sketch of the proof is the following: (1) we prove an upper bound for a fixed value of $X_{n+1}$ using Chernoff method, (2) we use a union bound on a partition of the possible values for $X_{n+1}$, and (3) we show that a finite number of terms in this union bound are significant. For any cdf $F$, threshold $\mu$ and random variable $Z_X$ determined by $X$ we use the notation 
\[\Lambda_{\alpha, \beta}(F, \mu) = \bE_{F}\left[\log\left(1-(X-\mu)\alpha - (B-Z_X)\beta\right)\right]\;, \]
where $\cR_2 \coloneqq (\alpha, \beta) \in \left\{(\alpha, \beta) \in (\R^+)^2: \forall y \in \R, 1-(y-\mu) \alpha - (B-Z_y)\beta \geq 0 \right\}$. Our first result is the following upper bound for a given value of $X_{n+1}$.

\begin{lemma}[Upper bound on the BCP] \label{lem::ub_bcp_single}
	Consider fixed observations $X_1,\dots, X_n, X_{n+1}$, and denote by $F_{n+1}$ their empirical cdf. Then, for any $(\alpha, \beta)\in \cR_2$ it holds that  
	\[\bP\left(\sum_{i=1}^{n+1} w_i X_i \geq \mu,\sum_{i=1}^{n+1} w_i Z_{X_i} \leq B \right) \leq \exp(-(n+1)\Lambda_{\alpha, \beta}(F_{n+1}, \mu))\;.  \]
\end{lemma}

\begin{proof}
We use the notation $\widetilde X=(X_1,\dots, X_{n+1})$ and $\widetilde Z= (Z_{X_1}, \dots, Z_{X_{n+1}})$ for simplicity, and denote by $R=(R_1, \dots, R_{n+1})$ a vector of $n+1$ i.i.d. observations drawn from an exponential distribution $\epsilon(1)$. Using usual properties of the Dirichlet distribution (see Chapter 23 of \citet{inference_mckay}) and Chernoff method, for any $\alpha>0$, $\beta>0$ it holds that

\begin{align}
	\bP_{w \sim \cD_n}(w^T \widetilde X \geq \mu, w^T \widetilde Z \leq B) 	& = \bP_{R \sim \cE(1)^n}(R^T (\widetilde X-\mu) \geq 0, R^T (B-\widetilde Z) \geq 0)\nonumber\\
	& = \bP_{R \sim \cE(1)^n}\left(e^{\alpha(R^T (\widetilde X-\mu)} \geq 1, e^{\beta (R^T (B-\widetilde Z))} \geq 1\right) \nonumber\\
	& \leq \bP_{R \sim \cE(1)^n}\left(e^{\alpha(R^T (\widetilde X-\mu) + \beta (R^T (B-\widetilde Z))} \geq 1\right) \nonumber\\
	& \leq \bE_{R \sim \cE(1)^n}\left[e^{\alpha(R^T (\widetilde X-\mu) + \beta (R^T (B-\widetilde Z))}\right] \;.
 \nonumber
\end{align}
Then, if $(\alpha, \beta) \in \cR_2$, it holds that 
\begin{align*}
	\bP_{w \sim \cD_n}(w^T \widetilde X \geq \mu, w^T \widetilde Z \leq B) & \leq \prod_{i=1}^{n+1} \bE_{R_i \sim \cE(1)}\left[e^{\alpha(R_i (X_i-\mu) + \beta (R_i (B-Z_{X_i}))}\right] \\
	& = \prod_{i=1}^{n+1} \frac{1}{1- \left(\alpha(X_i-\mu) + \beta(B-Z_{X_i})\right)} \\
	& = \exp\left(-(n+1) \bE_{F_{n+1}} \left[\log\left(1-\left(\alpha(X-\mu) + \beta(B-Z_{X_i})\right)\right)\right]\right)\\
	& = \exp\left(-(n+1) \Lambda_{\alpha, \beta}(F_{n+1}, \mu)\right) \;.
\end{align*}
We use this result by partitioning the set of possible values for $X_{n+1}$. Consider a step $\eta>0$, and the sequence $(x_j)_{j \in \N}$ defined by $\forall j\in \N$, $x_j=\mu + j \eta$.

\begin{align*}
	\text{(P)} &\leq  \sum_{j=0}^{+\infty} \bP\left(\exists X_{n+1} \in [x_j, x_{j+1}]: \sum_{i=1}^{n+1} w_i X_i \geq \mu ,\sum_{i=1}^{n+1} w_iZ_{X_i} \leq B \right) \\
	& \leq \sum_{j=0}^{+\infty} \bP\left(\sum_{i=1}^{n} w_i X_i + w_{n+1} x_{j+1} \geq \mu ,\sum_{i=1}^{n} w_i Z_{X_i}+ w_{n+1} Z_{x_j} \leq B \right) \\
	& \leq \sum_{j=0}^{+\infty} \underbrace{\bP\left(\sum_{i=1}^{n} w_i X_i + w_{n+1} (x_j+\eta) \geq \mu ,\sum_{i=1}^{n} w_i Z_{X_i} + w_{n+1} Z_{x_j} \leq B \right)}_{P_j} \;. \\
\end{align*}

At this step we can use Lemma~\ref{lem::ub_bcp_single}, and isolating the term depending on $x_j$ from the others we obtain 
\begin{equation}\label{eq::union_proof_ub}
\sum_{j=0}^{+\infty} P_j \leq \sum_{j=0}^{+\infty} \frac{\exp\left(-n\Lambda_{\alpha(j), \beta(j)}(F_n, \mu-\eta)\right)}{1-\alpha(j)(x_j+\eta-\mu) - \beta(j)(B+\gamma-h(x_j-\mu))} \;, \end{equation} for any choice of sequence $(\alpha(j), \beta(j))_{j \in \N}$. In the rest of the proof we use that when $x_j$ is large enough the denominator becomes small, and so only a finite number of terms are dominant in the sum.

\paragraph{Upper bound for large $x_j$} We consider $(\alpha_n^\star, \beta_n^\star) = \aargmax_{\alpha, \beta} \Lambda_{\alpha, \beta}(F_n, \mu-\eta)$. We first use that $\beta_n^\star>0$, since $\beta_n^\star \in \cR_2$, so when $x_j$ is large enough the denominator scales with $\beta_n^\star h(x_j-\mu)$, which is larger than the other terms since $x^{1+\gamma}=o(h(x))$ for some $\gamma>0$. Hence, there exists a constant $y_n$ such that for any $x_j\geq y_n$ it holds that 

\[ 1-\alpha_n^\star(x_j+\eta-\mu) - \beta_n^\star(B+\gamma-h(x_j-\mu)) \geq \frac{\beta_n^\star}{2}h(x_j-\mu)  \;, \]
and so the part of the sum corresponding to all $x_j \geq y_n$ satisfies 
\begin{align*}\sum_{j\geq 1: x_j \geq y_n}^{+\infty} P_j &\leq \frac{2}{\beta_n^\star}  \left(\sum_{j\geq 1: x_j \geq y_n} h(j\eta)^{-1}\right) e^{-n\Lambda_{\alpha_n^\star, \beta_n^\star}(F_n, \mu-\eta)}
  \;, \end{align*}
where we know that $\sum h(j\eta)^{-1}$ converges since $x^{1+\eta_0}=o(h(x))$ for some $\eta_0>0$.

\paragraph{Upper bound for $x_j\leq y_n$} 
In this case, we define $(\alpha_n^\star, \beta_n^\star)$ as 
\[(\alpha_n^\star, \beta_n^\star) = \max_{(\alpha, \beta) \in \cR_2^{\gamma, \eta}} \Lambda_{\alpha, \beta}(F, \mu)\;, \]
with $\cR_2^{\gamma, \eta} = \{(\alpha, \beta) \in (\R^+)^2: \forall x\in \R, \; \; 1-\alpha(x+\eta-\mu)-\beta(B+\gamma-h(|x-\mu|) ) \geq 0\}$. 
We further define $x^\star = \text{argmin}_{x \in \R}\left\{1-\alpha_n^\star (x+\eta-\mu)-\beta_n^\star(B+\gamma-h(|x-\mu|)) \right\}$. $x^\star$ depends on $F_n$, but we omit reference to $F_n$ for lighter notation. We use again the upper bound of $\sum P_j$ provided by \eqref{eq::union_proof_ub}. The main task, that we tackle in the following, is to lower bound the denominator by a constant depending on the empirical distribution $F_n$ and $\eta$.

Since the partition $(x_j)_{j \in \N}$ defined at the beginning of the proof is arbitrary (and only used in the proof but not by the algorithm), we can slightly modify it to avoid arbitrarily large values when $x_j$ is close to $x^\star$. Let $j^\star\in \N$ be the index satisfying $x^\star \in [x_{j^\star}, x_{j^\star+1}]$. In order to introduce some distance between $x^\star$ and $x_{j^\star}$, $x_{j^\star+1}$, we simply redefine $x_{j^\star}=x^\star-\frac{\eta}{2}$ and $x_{j^\star+1}=x^\star-\frac{\eta}{2}$.
We furthermore use that $d(x_j)\coloneqq 1-\alpha_n^\star(x^\star+\eta-\mu) - \beta_n^\star(B-h(x^\star-\mu)) \geq 0$ and is minimized in $x^\star$. By convexity of $h$, this quantity
is minimized either in $x_{j^\star}$ or $x_{j^\star+1}$. Thanks to the change made to the discretization scheme, this provides 
\[ \inf_{j \in \N} d(x_j)\geq \min\left\{d\left(x^\star+\frac{\eta}{2}\right), d\left((x^\star -\frac{\eta}{2}\right) \right\}>0\;. \]
Furthermore, the right-hand term only depends on $\eta$ and on the empirical distribution $F_n$ (through $x^\star$).
Hence, for $x_j \leq y_n$ we can state that there exists a non-negative constant $D(F_n, \eta)$ such that 
\[ \frac{1}{1-\alpha_n^\star(x_j-\mu) - \beta_n^\star(B-h(x_j-\mu))} \leq D(F_n, \eta)\;,\]
where $D(F_n, \eta)$ 
is continuous in $F_n$ w.r.t. the Wasserstein metric since both $(\beta_n^\star, x^\star)$ are continuous under this metric (as a consequence of Berge's maximum theorem, \citealp{berge1997topological}). Finally, we remark that there are no more than $y_n/\eta$ terms satisfying $x_j\leq y_n$ and obtain 
\begin{align*} \sum_{j=0}^{k\geq  1: x_j < y_n} P_j & \leq  
\sum_{j=0}^{k\geq  1: x_j < y_n} D(F_n, \eta) e^{-n\Lambda_{\alpha_n^\star,\beta_n^\star}(F_n, \mu)} \\
&\leq y_n \frac{D(F_n, \eta)}{\eta} e^{-n\Lambda_{\alpha_n^\star, \beta_n^\star}(F_n, \mu)}.
\end{align*}

\paragraph{Summary} Combining our two results, we obtain that  

\begin{equation}\label{eq::upper_bound_bcp} \text{(P)} \leq C(\alpha_n^\star, \beta_n^\star, \eta) \exp\left(-n\Lambda_{\alpha_n^\star, \beta_n^\star} (F_n, \mu)\right) \;, \end{equation}
where \[C(\alpha_n^\star, \beta_n^\star, \eta)= y_n \frac{D(F_n, \eta)}{\eta} +\frac{2}{\beta_n^\star} \sum_{k=\left\lceil\frac{y_n-\mu}{\eta} \right\rceil}^{+\infty} \frac{1}{h(\eta k)}\;.\] 

As stated above the constant $C(\alpha_n^\star, \beta_n^\star, \eta)$ depend on $F_n$ through the optimizers $(\alpha_n^\star, \beta_n^\star)$ (as they also determine $y_n$ and $x^\star$), that depends on $\gamma$ and $\eta$ since they belong to $\cR_2^{\gamma, \eta}$. Finally, with the same arguments as for the continuity of $D$, it is clear that $C$ is continuous in all parameters, which concludes the proof of the result.
\end{proof}

%

\subsection{Lower bound: Bounded case}\label{app::lb_bcp_bounded}

In this section we present a novel lower bound for the BCP of NPTS in the bounded case, that can also serve as a warm-up to present the idea for the slightly more complicated $h$-moment condition setting. We assume that all observations are bounded by some constant $B>0$, and that the algorithm is using the additional support $X_{n+1}= B+\gamma$ for some $\gamma>0$.

\begin{lemma}[Novel lower bound for the BCP of NPTS]\label{lem::novel_lb_bcp_npts}
If NPTS is implemented with exploration bonus $B+\gamma$ for some $\gamma>0$, then for any $\delta>0$ there exists a constant $n_{\delta, \gamma} \in \N$ such that
    \[\text{\emph{[BCP]}}\geq 
\frac{1-e^{-\frac{\delta}{4}}}{2e^{1+\delta}} \times e^{-n\left(\Lambda_{\alpha_n^\star}(F_n, \mu^\star)+\delta\right)}\; \quad \mbox{if} \quad n \geq n_{\delta, \gamma}\;,\]
and $\alpha_n^\star = \aargmax_{\alpha \in \left[0, \frac{1}{B+\gamma-\mu}\right]}\Lambda_{\alpha_n^\star}(F_n, \mu^\star)$.
\end{lemma}

\begin{proof} Recall that
\[\text{[BCP]} \geq \bP_{w \sim \cD_{n+1}}\left(\sum_{i=1}^{n+1} w_i X_i\geq \mu^\star \right) \;,\]
with $X_{n+1} = B+\gamma$. For any cdf $F_{n+1}: x \mapsto  \frac{1}{n+1}\sum_{i=1}^{n+1}\ind(X_i\leq x)$ we define

\[\Lambda_{\alpha}(F_n, \mu) = \bE_{F_{n+1}}\left[\log\left(1-\alpha(X-\mu)\right)\right] \;. \]

We recall that any vector $w=(w_1,\dots, w_{n+1}) \sim \cD_{n+1}$ can be re-written as a normalized sum of $\epsilon(1)$ exponential random variables. Hence, we can write the BCP as follows:
$$ \text{[BCP]} = \bP_{R_1\sim \epsilon(1),\dots, R_{n+1} \sim \epsilon(1)}\left(\sum_{i=1}^{n+1} R_i(X_i-\mu^\star) \geq 0\right) \;,$$
where $(R_i)_{i\geq 1}$ are i.i.d. and follow an exponential distribution $\cE(1)$. To lower bound this quantity, we will perform a change of measure for each variable $R_i$. In particular, we propose to consider a variable that has expectation $G_i$, for some $(G_i)_{i=1,\dots, n+1}$. We can do that by choosing exponential random variables $\epsilon\left(\frac{1}{G_i}\right)$: this technique is known as exponential tilting. If $\widetilde f_i$ denotes the density of this random variable and $f$ denotes the density of the distribution $\epsilon(1)$, it holds that 

\[\forall x \geq 0 : \; \frac{f(x)}{\widetilde f_i(x)} = G_i e^{- \left(1-\frac{1}{G_i} \right) x}
\]

We further denote by and $\bP$ and $\widetilde \bP$ their respective product. Using $R^T U = \sum_{i=1}^{n+1} R_i (X_i-\mu)$, we now express the BCP using the distributions of densities $\widetilde f_i$, 

\begin{align*}
	\text{[BCP]} =& \bE_{\bP}\left[\ind\left(R^T U \geq 0\right)\right] \\
	= & \bE_{\widetilde \bP}\left[\ind\left(R^T U \geq 0\right) \prod_{i=1}^{n+1} \frac{f(R_i)}{\widetilde f_i(R_i)}\right] \\
	= & \prod_{i=1}^{n+1} G_i \times \bE_{\widetilde \bP}\left[\ind\left(R^T U \geq 0\right) e^{-\sum_{i=1}^{n+1} \left(1-\frac{1}{G_i}\right)R_i}\right] \\
	= & e^{- \sum_{i=1}^{n+1} \log\left(\frac{1}{G_i}\right)} \bE_{\widetilde \bP}\left[\ind\left(R^T U \geq 0\right) e^{-\sum_{i=1}^{n+1} \left(1-\frac{1}{G_i}\right)R_i}\right] \;.\\
\end{align*}
Then, our proof will consists in: first, deriving the term $\max_{\alpha}\Lambda_{\alpha}(F_n, \mu)+\delta$ from the exponential term, and then showing that the expectation can be lower bounded by a constant. To do that, for any arbitrary constant $\eta_1>0$ and 
for $i\leq n+1$ we consider 
\[ G_i =  \frac{1}{1-\alpha(X_i-\mu-\eta_1)} \quad \text{ for some } \alpha>0\;, \] 
so that $\frac{1}{n}\sum_{i=1}^{n} \log\left(\frac{1}{G_i}\right)=\Lambda_{\alpha}(F_n, \mu+\eta_1)$. Now, let us choose \[\alpha= \alpha^\star \coloneqq \aargmax\, \Lambda_{\alpha}(F_{n+1}, \mu+\eta_1) \;, \] 
and then $\Lambda_{\alpha^\star}(F_{n+1}, \mu+\eta_1)= \kinf^\cF(F_{n+1}, \mu+\eta_1)$, where $\cF$ is the family of distributions bounded by $B+\gamma$, and the weights $(G_i)_{i \in [n+1]}$ associated with support $(X_1,\dots, X_n)$ define a probability distribution of mean equal to $\mu+\eta_1$. Furthermore, it holds that
\begin{align}
 \alpha \leq \alpha_M \coloneqq \frac{1}{B+\gamma - \mu-\eta_1},\label{def_alpha_M}   
\end{align}
but also that 
\[\forall i \in [n]\;, \;  X_i\leq B \Rightarrow  \frac{1}{1+\alpha_M(\mu+\eta_1)}\leq G_i \leq \frac{B+\gamma-\mu-\eta_1}{\gamma}\;.\] 

This property is crucial for the rest of our proof, and motivates the choice of $X_{n+1}$ as $B+\gamma$ instead of $B$: it ensures that all the exponential variables $R_1,\dots, R_n$ have a reasonable expectation, that does not depend on $n$: only $G_{n+1}$ can potentially be large (in the case when the empirical distribution is very bad, e.g. $X_1= X_2= \dots= X_n=0$).

We now write that 

\[ \text{[BCP]} \geq   e^{-(n+1) \Lambda_{\alpha}(F_{n+1} \mu+\eta_1)} \times \bE_{\widetilde \bP}\left[\ind\left(R^T U \geq 0\right) e^{- \alpha R^T \widetilde U}\right] \;,
\]
with $\widetilde U = (X_1-\mu-\eta_1, \dots, X_{n+1}-\mu-\eta_1)$, and after remarking that for all $i$,
\[1- \frac{1}{G_i}= \alpha (X_i-\mu-\eta_1) .\]

We now use that $\alpha\leq \alpha_M \coloneqq \frac{1}{B+\gamma-\mu-\eta_1}$ to handle the exponential term inside the expectation. Then, for any $\delta >0$ we write that 
\begin{equation} \label{eq::bcp_palph}	\text{[BCP]} \geq  e^{-(n+1) \left(\Lambda_{\alpha, \beta}(F_{n+1}, \mu+\eta_1)+\frac{\delta}{2}\right)} \times P_{\alpha} \;,
\end{equation}
where we defined 
\begin{equation}\label{eq::P_alpha}P_{\alpha} = \widetilde \bP\left(R^T U \geq 0, R^T \widetilde U \leq (n+1)\frac{\delta}{2\alpha_M}\right) \;.
\end{equation}

In some cases, it may happen that $G_{n+1}$ is large (at most $n+1$), and in that case the exponential variable as variance $n^2$. Hence, we would like to control very precisely the deviations of $R_{n+1}$. For this reason, for some constant $c>0$ we consider the event 

\[E_{n+1} \coloneqq \left\{ R_{n+1} \in \left[G_{n+1},\; (1+c)G_{n+1} \right]\right\} \;. \]

Using the cdf of exponential variables, this event has exactly a probability of 

\[\bP_\alpha(R_{n+1} \in \left[G_{n+1},\; (1+c)G_{n+1} \right]) = e^{-\frac{G_{n+1}}{G_{n+1}}}-e^{-\frac{G_{n+1}(1+c)}{G_{n+1}}} = e^{-1}(1-e^{-c})\;. \]

We now re-write $P_{\alpha}$ as

\begin{align*}
	P_{\alpha} &\geq \widetilde \bP\left(R^T U \geq 0, R^T \widetilde U \leq (n+1)\frac{\delta}{2\alpha_M}, E_{n+1}\right) \\
	& \geq \bP(E_{n+1}) \bP\left(R^T U \geq 0, R^T \widetilde U \leq (n+1)\frac{\delta}{2\alpha_M}| E_{n+1}\right)\\	
	& = e^{-1}(1-e^{-c}) \left(1- \underbrace{\bP_{\alpha}(R^T U \leq 0, R^T \widetilde U \geq n\frac{\delta}{2\alpha_M}|E_{n+1})}_{p_1}\right) \;.
\end{align*}

The rest of the proof consists in showing that for $n$ large enough the terms $p_1$ and $p_2$ are small. We consider the first event in $p_1$. Using that $\bE_\alpha[R^TU] = (n+1) \eta$

\begin{align*}
	 \left\{R^T U \leq 0 \right\}& \subset \left\{\sum_{i=1}^n R_i (X_i - \mu) + G_{n+1} (X_{n+1}- \mu) \leq 0\right\} \\
	& = \left\{\sum_{i=1}^n R_i (X_i - \mu) - \bE\left[\sum_{i=1}^n R_i (X_i-\mu)\right] \leq -(n+1)\eta \right\} \;,
\end{align*}
and so under $E_{n+1}$ the first event implies a deviation of $\sum_{i=1}^n R_i (X_i - \mu)$ from its mean. Next, considering the second event it holds that 
\begin{align*}
	&\left\{R^T \widetilde U \geq (n+1)\frac{\delta}{2\alpha_M} \right\}\subset \left\{\sum_{i=1}^n R_i (X_i - \mu) + (1+c)G_{n+1} (X_{n+1}- \mu) \geq (n+1)\frac{\delta}{2\alpha_M}\right\} \\
	& = \left\{\sum_{i=1}^n R_i (X_i - \mu) - \bE\left[\sum_{i=1}^n R_i (X_i-\mu)\right] \leq (n+1)\frac{\delta}{2\alpha_M} - c G_{n+1}(B+\gamma-\mu) \right\} \;.
\end{align*}
Then, we use that $G_{n+1} \leq n+1$ to tune $c$ in order to obtain something similar to the first event.

For the above goal, let us consider the case of
\begin{align}
c= \frac{\delta}{4\alpha_M (B+\gamma-\mu)}.\label{assump_c}
\end{align}
In this case we obtain that 
\begin{align*}
	\left\{R^T \widetilde U \geq (n+1)\frac{\delta}{2\alpha_M} \right\}\subset\left\{\sum_{i=1}^n R_i (X_i - \mu) - \bE\left[\sum_{i=1}^n R_i (X_i-\mu)\right] \leq (n+1)\frac{\delta}{4\alpha_M}\right\} \;.
\end{align*}
Then, using the notation $\eta' = \min\left\{\eta, \frac{\delta}{4\alpha_M} \right\}$ we obtain with the Tchebychev inequality that 
\begin{align}
	p_1 &\leq \bP_\alpha\left(\left|\sum_{i=1}^n R_i (X_i-\mu) - \sum_{i=1}^n G_i (X_i-\mu) \right| \geq (n+1) \eta' \right)\nn \\
	& \leq \frac{1}{(n+1)^2 \eta'^2} \times \sum_{i=1}^n G_i^2 (X_i-\mu)^2\nn \nn\\
	& \leq \frac{1}{(n+1)^2 \eta'^2} (B-\mu)^2 \max_{i=1,\dots, n} G_i \times \sum_{i=1}^n G_i  \nn\\
	& \leq \frac{(B-\mu)^2}{(n+1) \eta'^2}  \frac{B+\gamma -\mu - \eta_1}{\gamma}  \nn\\
	& \leq \frac{1}{n+1} \times \frac{(B-\mu)^2}{\eta'^2} \left(1+\frac{B -\mu}{\gamma} \right) \;.\label{result_under_c}
\end{align}
As this bound is independent of $F_n$, we obtain that for any $\eta>0$, there exists $n_\eta \in \N$ satisfying $p_1 \leq \frac{1}{2}$ for $n\geq n_\eta$, under the case of \eqref{assump_c}.

Here note that $c$ in \eqref{assump_c} goes to $\frac{\delta}{4}$ as $\eta\to 0$
since $\alpha_M (B+\gamma- \mu)\to 1$ holds from \eqref{def_alpha_M}.
For this reason,
we now take
$c=\delta/4$ instead of \eqref{assump_c}, under which a result similar to \eqref{result_under_c} holds and there still exists $n_{\eta}\in\mathbb{N}$ satisfying $p_1 \leq \frac{1}{2}$ for $n\geq n_\eta$ if we choose $\eta$ small enough.
For such $n\ge n_{\eta}$,
it holds that $P_\alpha\geq \frac{e^{-1}(1-e^{-c})}{2}=\frac{e^{-1}(1-e^{-\frac{\delta}{4}})}{2}$ (plugging our result in Equation~\eqref{eq::P_alpha}) and so using \eqref{eq::bcp_palph} we obtain that
\[ \text{[BCP]} \geq  \frac{e^{-1}(1-e^{-\frac{\delta}{4}})}{2}  e^{-(n+1) \left(\Lambda_{\alpha}(F_{n+1}, \mu+\eta)+\frac{\delta}{2}\right)}  \;.
\]

We now tune $\eta$ as 
\[\eta = \max \left\{\eta >0: \; \forall F_n \in \cF\;, \; \Lambda_{\alpha^\star}(F_{n+1}, \mu+\eta)\leq \kinf^{\cF^\gamma}(F_{n+1}, \mu) + \frac{\delta}{2} \right\}\;. \]
Due to the compactness of $\cF$, this quantity exists, is strictly positive, and depends on $\delta$, and goes to $0$ as $\delta$ goes to $0$. Finally, we remark that since $X_{n+1}$ has a negative contribution to $\kinf(F_{n+1}, \mu)$ (expressed as a sum of $n+1$ logarithmic terms, the $(n+1)$-th one being the minimum), 
then it holds that 
\[(n+1)\kinf^{\cF^\gamma}(F_{n+1}, \mu) 
\leq n \kinf^{\cF^\gamma}(F_n, \mu)\;\]
and consequently, we conclude that for any $\delta>0$, there exists a constant $n_\delta$ such that for all $n\geq n_\delta$ it holds that  
\[ \text{[BCP]} \geq  \frac{1-e^{-\frac{\delta}{4}}}{2e^{1+\delta}}  e^{-n \left(\kinf^{\cF^\gamma}(F_n, \mu)+\delta\right)}  \;,
\]
which concludes our proof for this result.
\end{proof}

\begin{remark} Chernoff method would lead to a tighter upper bound for $p_1$, however we choose Tchebychev inequality for simplicity since knowing the right value of $n_\delta$ is not necessary for our algorithm.
\end{remark}


\subsection{Lower bound: $h$-moment condition}\label{app::lb_npts}

In this section we prove the lower bounds on the BCP in Lemma~\ref{lem::bounds_bcp}. The proof follows the exact same steps as the proof of Lemma~\ref{lem::novel_lb_bcp_npts} presented in Section~\ref{app::lb_bcp_bounded}, with some additional technicality due to the slightly more complicated definition of the family $\cF$. We still write all the details for completeness.
We first recall that for observations $X_1,\dots, X_n$ the BCP is expressed as 
\begin{align}
\text{[BCP]}
&\geq \bP_{w \sim \cD_{n+1}}\left(\exists X_{n+1} \geq \mu^\star: \; \sum_{i=1}^{n+1} w_i X_i\geq \mu^\star, \sum_{i=1}^{n+1} w_i Z_{X_i}\leq B \right)\nn\\
&\geq \bP_{w \sim \cD_{n+1}}\left(\sum_{i=1}^{n+1} w_i X_i\geq \mu^\star, \sum_{i=1}^{n+1} w_i Z_{X_i}\leq B \right)\quad\mbox{for arbitrary $X_{n+1}\ge \mu^\star$}\nn
\;,
\end{align}
with $Z_{X_i} = h(|X_i-\mu|)$ for $i \in [n]$ and $Z_{X_{n+1}} = h(X_{n+1}-\mu)-\gamma$.
To provide a lower bound on this quantity we will choose $X_{n+1}$ so that the bound is expressed in terms of the quantity 
\[\Lambda_{\alpha, \beta}(F, \mu, B) = \bE_{F}\left[\log\left(1-\alpha(X-\mu)-\beta(B-Z_X)\right)\right] \;, \]
given some parameters $\alpha, \beta, \mu, B$ and a cdf $F$.


We recall that any vector $w=(w_1,\dots, w_{n+1}) \sim \cD_{n+1}$ can be re-written as a normalized sum of $\epsilon(1)$ exponential random variables.
Hence, we can write the BCP as follows:
$$ \text{[BCP]} \geq \bP_{R_1\sim \epsilon(1),\dots, R_{n+1} \sim \epsilon(1)}\left(\sum_{i=1}^{n+1} R_i(X_i-\mu) \geq 0, \sum_{i=1}^{n+1} R_i(B-Z_{X_i}) \geq 0 \right) \;,$$
where $(R_i)_{i\geq 1}$ are i.i.d. and follow an exponential distribution $\cE(1)$. We perform an exponential tilting of the distributions, considering the exponential variables with expectation $\bE[R_i] = G_i$, for some sequence $(G_i)_{i=1,\dots, n+1}$, so $R_i \sim \epsilon\left(\frac{1}{G_i}\right)$. If $\widetilde f_i$ denotes the density of this random variable and $f$ denotes the density of the distribution $\epsilon(1)$, it holds that 
\[\forall x \geq 0 : \; \frac{f_i(x)}{\widetilde f_i(x)} = G_i e^{- \left(1-\frac{1}{G_i} \right) x}.
\]

We further denote by and $\bP$ and $\widetilde \bP$ their respective product. Using the notation $R^T U = \sum_{i=1}^{n+1} R_i (X_i-\mu)$ and $R^T V = \sum_{i=1}^{n+1}R_i(B-Z_{X_i})$, we now re-write the BCP using the distributions of densities $\widetilde f_i$, 
\begin{align*}
	\text{[BCP]} =& \bE_{\bP}\left[\ind\left(R^T U \geq 0, R^T V \geq 0\right)\right] \\
	= & \bE_{\widetilde \bP}\left[\ind\left(R^T U \geq 0, R^T V \geq 0\right) \prod_{i=1}^{n+1} \frac{f_i(R_i)}{\widetilde f_i(R_i)}\right] \\
	= & \prod_{i=1}^{n+1} G_i \times \bE_{\widetilde \bP}\left[\ind\left(R^T U \geq 0, R^T V \geq 0\right) e^{-\sum_{i=1}^{n+1} \left(1-\frac{1}{G_i}\right)R_i}\right] \\
	= & e^{- \sum_{i=1}^{n+1} \log\left(\frac{1}{G_i}\right)} \bE_{\widetilde \bP}\left[\ind\left(R^T U \geq 0, R^T V \geq 0\right) e^{-\sum_{i=1}^{n+1} \left(1-\frac{1}{G_i}\right)R_i}\right] \;.\\
\end{align*}
We want to make the term $\max_{\alpha, \beta}\Lambda_{\alpha, \beta}(F_n, \mu, B)+\delta$ come from the exponential term in front of the expectation, and to show that the expectation can be lower bounded by a constant. To do that, we consider 
two arbitrary constants $\eta_1>0, \eta_2>0$, and for $i\leq n+1$ we define
\[
G_i
=
\begin{cases}
\frac{1}{1-\alpha(X_i-\mu-\eta_1)- \beta (B-\eta_2 - Z_{X_i})}, 
&i\in[n],\\
\frac{1}{1-\alpha(X_i-\mu-\eta_1)- \beta (B+\gamma-\eta_2 - Z_{X_i})}, &i=n+1,
\end{cases}
\] 
for some constants $(\alpha, \beta)$.
We still center the $h$-moment in $\mu$, as we want the objective inside the expectation to be strictly easier to reach with our change of distribution. However, we keep the notation $\Lambda_{\alpha, \beta}$ with a slight abuse for simplicity. 

We then choose $X_{n+1}$ in the BCP as
\begin{equation}\label{eq::X_n+1} X_{n+1} = \text{argmin}_{x\in \R} \{1-\alpha_n^\star(x-\mu-\eta_1)- \beta_n^\star (B-\eta_2 - h(|x-\mu|))\} \;. \end{equation}
Then, by definition of $(\alpha_n^\star, \beta_n^\star)$ we obtain that
\begin{align*}
1-\alpha_n^\star (X_i-\mu-\eta_1)&-\beta_n^\star(B-\eta_2-h(|X_i-\mu|))\\ 
&\geq 1-\alpha_n^\star (X_{n+1}-\mu-\eta_1)-\beta_n^\star(B-\eta_2-h(|X_{n+1}-\mu|)) \quad \text{(by \eqref{eq::X_n+1})} \\
&\geq \beta_n^\star \gamma + \underbrace{1-\alpha_n^\star (X_{n+1}-\mu-\eta_1)-\beta_n^\star(B+\gamma-\eta_2-Z_{n+1})}_{\geq 0} \\
&\geq \overline \beta \gamma \coloneqq  \left( \min\limits_{F_n\in \cF: \bE_{F_n}[X]\leq \mu} \beta_n^\star\right) \gamma \;,
\end{align*}
where $\overline \beta$ is well-defined and satisfies $\overline \beta>0$ by compactness of $\cF$.
With the same argument, we can also deduce the existence of a (universal) upper bound $X^+$ of the potential additional support $X_{n+1}$ (exploiting the continuity of mapping $F \mapsto \alpha_F^\star$, $F \mapsto \beta_F^\star$ in $F$, and that the additional support satisfies $\alpha_F^\star =
h'(X-\mu)\beta_F^\star$). Now, given that $X_{n+1}$ cannot be arbitrarily large, $W(F_n, F_{n+1})=\cO(n^{-1})$, where $W$ denotes the Wasserstein metric. Furthermore, using the continuity of the mapping $F \mapsto 1-\alpha_{F}^\star (x-\mu-\eta_1)-\beta_{F}^\star(B-\eta_2-h(|x-\mu|))$ for any $x\in \R$, we finally conclude that there exists a rank $n_0\in \N$ such that for any $n\geq n_0$ it holds that
\[\forall i \in [n]\;, \; G_i \leq \frac{2}{\gamma \overline \beta}\;.\] 
This property ensures that all the variables $G_1,\dots, G_{n}$ have a reasonable expectation, that does not depend on $n$ but on $\gamma$: only $G_{n+1}$ can potentially be large (at most $n+1$).

We now get back to lower bounding the BCP, writing that 
\[ \text{[BCP]} \geq   e^{-n \Lambda_{\alpha, \beta}(F_n, \mu+\eta_1, B-\eta_2)} \times \bE_{\widetilde \bP}\left[\ind\left(R^T U \geq 0, R^T V \geq 0\right) e^{- \alpha R^T \widetilde U - \beta R^T \widetilde V}\right] \;,
\]
with $\widetilde U = (X_1-\mu-\eta_1, \dots, X_{n+1}-\mu-\eta_1)$ and $\widetilde V = (B-\eta_2-Z_{X_1}, \dots, B+\gamma-\eta_2-Z_{X_{n+1}}$, and after remarking that for all $i\leq n$ (and with $B+\gamma$ $i=n+1$),
\[1- \frac{1}{G_i}= \alpha (X_i-\mu-\eta_1) + \beta (B-\eta_2-Z_{X_i}).\]

Then, we handle the exponential term inside the expectation by first using from \citet{agrawal2020optimal} (their proof can be simply adapted to get the centered condition instead of the uncentered one) that $\alpha$ and $\beta$ are both upper bounded respectively by some constants $\alpha_M $ and $\beta_M$ depending only on $\mu+\eta_1$ and $B-\eta_2$. Then, for any $\delta >0$ we write that 
\[	\text{[BCP]} \geq  e^{-n \left(\Lambda_{\alpha, \beta}(F_n, \mu+\eta_1, B-\eta_2)+\frac{\delta}{2}\right)} \times P_{\alpha, \beta} \;,
\]
where we defined 
\[P_{\alpha, \beta} = \widetilde \bP\left(R^T U \geq 0, R^T V \geq 0, R^T \widetilde U \leq n\frac{\delta}{4\alpha_M}, R^T \widetilde V \leq n\frac{\delta}{4\beta_M}\right) \;.
\]
As in the bounded case, we control separately the term $R_{n+1}$ since it is the only term with an expectation that can depend on $n$. Considering a constant $c>0$, we define the event 

\[E_{n+1} \coloneqq \left\{ R_{n+1} \in \left[G_{n+1},\; (1+c)G_{n+1} \right]\right\} \;. \]

Using the cdf of exponential variables, this event has exactly a probability of 
\[\bP_{\alpha,\beta}(R_{n+1} \in \left[G_{n+1},\; (1+c)G_{n+1} \right]) = e^{-\frac{G_{n+1}}{G_{n+1}}}-e^{-\frac{G_{n+1}(1+c)}{G_{n+1}}} = e^{-1}(1-e^{-c})\;, \]
so we obtain that
\begin{align*}
	P_{\alpha, \beta} &\geq \widetilde \bP\left(R^T U \geq 0, R^T V\geq 0 R^T \widetilde U \leq (n+1)\frac{\delta}{4\alpha_M}, R^T \widetilde V \geq \frac{\delta}{4\beta_M},  E_{n+1}\right) \\
	& \geq \bP_{\alpha, \beta}(E_{n+1}) \bP\left(R^T U \geq 0, R^TV \geq 0 , R^T \widetilde U \leq (n+1)\frac{\delta}{4\alpha_M}, R^T \widetilde V \leq (n+1)\frac{\delta}{4\alpha_M}| E_{n+1}\right)\\	
	& = p_c \left(1- \underbrace{\bP_{\alpha,\beta}(R^T U \leq 0, R^T \widetilde U \geq n\frac{\delta}{4\alpha_M}|E_{n+1})}_{p_1}-\underbrace{\bP_{\alpha,\beta}(R^T V \leq 0, R^T \widetilde V \geq n\frac{\delta}{4\beta_M}|E_{n+1})}_{p_2}\right) \;,
\end{align*}
with $p_c=e^{-1}(1-e^{-c})$. The rest of the proof consists in showing that for $n$ large enough the terms $p_1$ and $p_2$ are small. We now consider that $(\alpha, \beta)= \aargmax \Lambda_{\alpha, \beta}(F_{n+1}, \mu+\eta_1, B-\eta_2)$. We consider the first event in $p_1$. Using that $\bE_{\alpha, \beta}[R^TU] = (n+1) \eta$

\begin{align*}
	\left\{R^T U \leq 0 \right\}& \subset \left\{\sum_{i=1}^n R_i (X_i - \mu) + G_{n+1} (X_{n+1}- \mu) \leq 0\right\} \\
	& = \left\{\sum_{i=1}^n R_i (X_i - \mu) - \bE\left[\sum_{i=1}^n R_i (X_i-\mu)\right] \leq -(n+1)\eta_1 \right\} \;,
\end{align*}
and so under $E_{n+1}$ the first event implies a deviation of $\sum_{i=1}^n R_i (X_i - \mu)$ from its mean. Then, to handle the second event we use that 
\begin{align*}
	&\left\{R^T \widetilde U \geq (n+1)\frac{\delta}{4\alpha_M} \right\}\subset \left\{\sum_{i=1}^n R_i (X_i - \mu) + (1+c)G_{n+1} (X_{n+1}- \mu) \geq (n+1)\frac{\delta}{4\alpha_M}\right\} \\
	& = \left\{\sum_{i=1}^n R_i (X_i - \mu) - \bE\left[\sum_{i=1}^n R_i (X_i-\mu)\right] \leq (n+1)\frac{\delta}{4\alpha_M} - c G_{n+1}(X^+-\mu) \right\} \;.
\end{align*}

Then, we use that $G_{n+1} \leq n+1$ to obtain a first condition on $c$. We arbitrarily consider $c\leq \frac{\delta}{8\alpha_M (X^+-\mu)}$, to obtain that 

\begin{align*}
	\left\{R^T \widetilde U \geq (n+1)\frac{\delta}{4\alpha_M} \right\}\subset\left\{\sum_{i=1}^n R_i (X_i - \mu) - \bE\left[\sum_{i=1}^n R_i (X_i-\mu)\right] \leq (n+1)\frac{\delta}{8\alpha_M}\right\} \;.
\end{align*}

Hence, using the notation $\eta_1' = \min\left\{\eta, \frac{\delta}{4\alpha_M} \right\}$ we get

\begin{equation}\label{eq;;ub_p1} p_1 \leq \bP_{\alpha,\beta}\left(\left|\sum_{i=1}^n R_i (X_i-\mu) - \sum_{i=1}^n G_i (X_i-\mu) \right| \geq (n+1) \eta_1' \right)\;. \end{equation}

We then perform the exact same steps with $p_2$. This time, it will depend on the sign of $B-h(X_{n+1}-\mu)+\gamma$. If $B-h(X_{n+1}-\mu)+\gamma >0$, we obtain similarly that 

\begin{align*}
	&\left\{R^T V \leq 0 \right\} \subset \left\{\sum_{i=1}^n R_i (B-h(|X_i - \mu|) - \bE\left[\sum_{i=1}^n R_i (X_i-\mu)\right] \leq -(n+1)\eta_2 \right\} \; \text{, and}\\
	& \left\{R^T \widetilde V \leq \frac{\delta}{4\beta_M} \right\} \subset \left\{\sum_{i=1}^n R_i (B-Z_{X_i}) - \bE\left[\sum_{i=1}^n R_i (B-Z_{X_i})\right] \geq (n+1) \left(\frac{\delta}{4\beta_M} - c (B+\gamma)\right) \right\} \;.
\end{align*}

So, this time we can require for instance $c \leq \frac{\delta}{8 \beta_M (B+\gamma)}$. On the other hand, if $B-h(X_{n+1}-\mu)+\gamma\leq 0$ this time we have  
\begin{align*}
	&\left\{R^T V \leq 0 \right\} \subset \left\{\sum_{i=1}^n R_i (B-h(|X_i - \mu|) - \bE\left[\sum_{i=1}^n R_i (X_i-\mu)\right] \leq -(n+1)\eta_{2, c} \right\} \;,\\
	& \left\{R^T \widetilde V \leq \frac{\delta}{4\beta_M} \right\}\\
 &\subset \left\{\sum_{i=1}^n R_i (B-Z_{X_i}) - \bE\left[\sum_{i=1}^n R_i (B-h(|X_i - \mu|))\right] \geq (n+1) \frac{\delta}{4\beta_M} \right\} \;,
\end{align*}
with $\eta_{2, c} = \eta_2 -c (h(X^+-\mu)-B-\gamma)$, and so this time we can ask that $c\leq \frac{\eta_2}{2(h(X^+-\mu)-B-\gamma)}$. Hence, if all requirements on $c$ are satisfied we can define $\eta_2' = \min \left\{\frac{\eta_2}{2(h(X^+-\mu)-B-\gamma)} ,\frac{\delta}{8 \beta_M (B+\gamma)}, \eta_2 \right\}$, and we proved that 

\begin{equation}\label{eq;;ub_p2} p_2 \leq \bP_{\alpha,\beta}\left(\left|\sum_{i=1}^n R_i (B-h(X_i-\mu)) - \sum_{i=1}^n G_i (B-h(|X_i-\mu|) \right| \geq (n+1) \eta_2' \right)\;. \end{equation}
Then, just as in the proof for the bounded case we can use the Tchebychev inequality to further upper bound $p_1$ and $p_2$ from \eqref{eq;;ub_p1} and \eqref{eq;;ub_p2}, and obtain that
\begin{align*}
	&p_1  \leq \frac{1}{(n+1)^2 \eta_1'^2} \times \sum_{i=1}^n G_i^2 (X_i-\mu)^2\;, \\
\text{and} \quad &p_2  \leq \frac{1}{(n+1)^2 \eta_2'^2} \times \sum_{i=1}^n G_i^2 (B-h(|X_i-\mu|))^2\;.
\end{align*}

This time the quadratic terms depending on $X_i$ cannot be discarded as easily as in the bounded case. However, we can use a similar trick to the one we used to prove the upper bound in Lemma~\ref{lem::bounds_bcp}, remarking that if $X_i$ is too large then $G_i$ is necessarily very small. Indeed, using that $\beta\geq \beta_m$ and that when $X_i$ is large then $G_i^2 \sim \frac{1}{\beta (B-h(|X_i-\mu|))^2}$, we obtain that there exists a constant $C>0$ such that for $|X_i-\mu| \geq C$ it holds that 

\[G_i^2 \leq \frac{2}{\beta_m (B-h(|X_i-\mu|))^2}\;. \]

Using this result and the notation $A_C = \max_{|x-\mu| \geq C} \frac{(x-\mu)^2}{\beta_m (B-h(|x-\mu|))^2}$, we first obtain that 
\begin{align*}
	p_1 & \leq \frac{1}{(n+1)^2 \eta_1'^2} \left(C^2 \sum_{i=1}^n G_i^2  + 2 n A_C   \right)  \\
	&\leq \frac{1}{(n+1)^2 \eta_1'^2} \left(C^2 \max_{i=1,\dots, n} G_i \times \sum_{i=1}^n G_i  + 2 n A_C\right) \\
	&\leq \frac{1}{(n+1)^2 \eta_1'^2} \left(C^2 (n+1) \max_{i=1,\dots, n} G_i + 2 nA_C\right) \\
	&\leq \frac{1}{(n+1) \eta_1'^2} \left(\frac{2C^2}{\gamma \beta_m} + 2 A_C\right) = \cO(n^{-1})
\;.
\end{align*}

Similarly, we obtain that

\begin{align*}
	p_2 & \leq \frac{1}{(n+1)^2 \eta_2'^2} \left((B-h(C))^2 \sum_{i=1}^n G_i^2  +  \frac{2 n}{\beta_m}   \right)  \\
	&\leq \frac{1}{(n+1) \eta_2'^2} \left(\frac{2(B-h(C))^2}{\gamma \beta_m} + \frac{2}{\beta_m}\right) = \cO(n^{-1})
	\;.
\end{align*}

As this bound is independent of $F_n$, we obtain that for any $\eta_1>0, \eta_2>0$, there exists a rank $n_{\eta_1,\eta_2}$ such that $p_1 \leq \frac{1}{4}$ and $p_2\leq \frac{1}{4}$ for $n\geq n_\eta$. In that case, it holds that 

\[ \text{[BCP]} \geq  \frac{e^{-1}(1-e^{-c})}{2}  e^{-(n+1) \left(\Lambda_{\alpha_n^\star, \beta_n^\star}(F_{n+1} , \mu+\eta_1, B-\eta_2)+\frac{\delta}{2}\right)}  \;,
\]
where this bound holds if $c\leq \min \left\{\frac{\delta}{8\alpha_M (X^+-\mu)}, \frac{\delta}{8 \beta_M (B+\gamma)}, \frac{\eta_2}{2(h(X^+-\mu)-B-\gamma)} \right\}$, so $c$ is of order $\cO(\min\{\delta, \eta_2 \})$. We now tune $\eta_1$ and $\eta_2$. For simplicity we consider 
\[\eta_1=\eta_2 = \max \left\{\eta >0: \; \forall F_n \in \cF\;, \; \max_{\alpha, \beta} \Lambda_{\alpha,\beta}(F_n, \mu+\eta, B-\eta)\leq \max_{\alpha, \beta} \Lambda_{\alpha,\beta}(F_n, \mu, B) + \frac{\delta}{2} \right\}\;. \]

Thanks to the compactness of $\cF$, this quantity exists, is strictly positive, and only depends on $\delta$. This allows us to conclude that for any $\gamma>0$ and $\delta>0$, there exists two constants $n_\cF$ and $c_\delta$ such that for all $n\geq n_\cF$ it holds that  

\[ \text{[BCP]} \geq  c_\delta  e^{-n \max_{(\alpha, \beta) \in \cR_2^\gamma} \left(\Lambda_{\alpha,\beta}(F_n, \mu, B) +\delta\right)}  \;,
\]
which concludes the part of the proof corresponding to $n\geq n_\cF$.

We finish the proof by showing that the BCP can also be lower bounded for $n \leq n_\cF$. However, this bound does not need to be tight since it will not impact the validity of the conditions that we want to prove. We provide a simple lower bound by fixing the value of the additional support to $X_{n+1} = \frac{\mu + h^{-1}(B)}{2}$. We show that for an appropriate constant $c$ depending only on $B$ and $\mu$, if the empirical distribution belongs to $\cF$ then it holds for any $n\geq 1$ that 

\[\text{[BCP]} \geq  \bP\left(w_{n+1} \geq 1-\frac{1}{cn}\right)\geq e^{-n\log(n c)} \;. \]

Indeed, if $F_n \in \cF$ then for all $i$, $\mu-h^{-1}(nB) \leq X_i \leq \mu+h^{-1}(nB)$ under the event $\{w_{n+1} \geq 1-\frac{1}{cn}\}$ we obtain

\begin{align*}
	\sum_{i=1}^{n+1} w_i X_i& \geq \left(1-\frac{1}{cn}\right) \left(\frac{\mu+h^{-1}(B)}{2}\right) - \frac{h^{-1}(nB)-\mu}{cn}\\
	& = \mu + \frac{h^{-1}(B)-\mu}{2} - \frac{1}{cn}\left(\frac{\mu+h^{-1}(B)}{2}+h^{-1}(nB)-\mu\right)\;, \\
\end{align*}
and 
\begin{align*}
	\sum_{i=1}^{n+1}& w_i h(|X_i-\mu|) \leq Z_{X_{n+1}} + \frac{1}{cn}\sum_{i=1}^n h(|X_i-\mu|) \leq h\left(\frac{h^{-1}(B)-\mu}{2}\right) + \frac{B}{c} \\ 
\end{align*}

Then, the tuning on $c$ is made to ensure that the bounds are satisfying for all values of $n$. We can verify that this is true if 
\[ c \geq C_{B,\mu} \coloneqq \max\left\{\frac{3h^{-1}(B)-\mu}{h^{-1}(B)-\mu}  , \frac{B}{B-h\left(\frac{ h^{-1}(B)-\mu}{2}\right)} \right\} \;. \]

\begin{remark}[Adaptation for the un-centered condition]\label{remark_adaptation} The same proof scheme leads to a similar result for the un-centered condition, simply with a different constant $C_{B,\mu}$. Indeed, we can fix the same value for $X_{n+1}$, and write the same results using instead that for all $i$, $-h^{-1}(nB) \leq X_i \leq h^{-1}(nB)$. The constant $C_{B, \mu}$ would become 
	\[C_{B,\mu} \coloneqq \max\left\{\frac{3h^{-1}(B)+\mu}{h^{-1}(B)-\mu}  , \frac{B}{B-h\left(\frac{\mu + h^{-1}(B)}{2}\right)} \right\} \] 
\end{remark}

\subsection{Validity of the simplified check}\label{app::proof_check_npts}

In this section we prove that for any parameter $\gamma>0$, observations $X_1,\dots, X_{n+1}$ and weights $w_1,\dots, w_{n+1}$ the following condition
\begin{equation*} \exists x \geq \mu^\star: \; \sum_{i=1}^{n} w_i X_i + w_{n+1} x \geq \mu^\star\;,\; \text{and } \sum_{i=1}^{n} w_i h(|X_i-\mu^\star|) + w_{n+1} (h(|x-\mu^\star|)-\gamma) \leq B\;,\end{equation*}
is equivalent to 
\begin{equation*}
h\left(\left(\frac{1}{w_{n+1}}\sum_{i=1}^{n} w_i (\mu^\star - X_i)\right)^+\right) \leq B + \gamma + \frac{1}{w_{n+1}} \sum_{i=1}^{n} w_i (B-h(|X_i - \mu^\star|)) \;,
\end{equation*}
where $(x)^+=\max\{x, 0\}$.

Assume that $x$ is a solution to the first (mean) constraint. Then, using that $\sum_{i=1}^{n+1}w_i=1$ and that $\{w_{n+1}=0\}$ has probability zero it must hold that 
\begin{align*}
    \sum_{i=1}^{n} w_i(X_i-\mu^\star) + w_{n+1} (x-\mu^\star) \geq 0 \Rightarrow x \geq \mu^\star + \frac{1}{w_{n+1}} \sum_{i=1}^{n} w_i (\mu^\star - X_i) \;.
\end{align*}
Then, doing the same with the second constraint it must also hold that 

\[h(x-\mu^\star) \leq B+\gamma + \frac{1}{w_{n+1}} \sum_{i=1}^n w_i (B-h(|X_i-\mu^\star|)) \;. \]

We then consider the two cases, depending on the sign of $\sum_{i=1}^n w_i (\mu^\star-X_i)$. If it is positive, as $h$ is a bijection on the positive line we directly get that 

\[h\left(\frac{1}{w_{n+1}}\sum_{i=1}^{n} w_i (\mu^\star - X_i)\right) \leq h(x-\mu^\star) \leq B + \gamma + \frac{1}{w_{n+1}} \sum_{i=1}^{n} w_i (B-h(|X_i - \mu^\star|)) \;. \]

In the case that $\sum_{i=1}^n w_i (\mu^\star-X_i)\leq 0$, it is only necessary to check if the condition holds for $x=0$, that is if 
\[0 = h(0) \leq B + \gamma + \frac{1}{w_{n+1}} \sum_{i=1}^{n} w_i (B-h(|X_i - \mu^\star|)) \;. \]

Hence, the formulation of ~\eqref{eq::check_practice} is just a way to summarize the two cases and \eqref{eq::check_npts} $\Rightarrow $~\eqref{eq::check_practice}. Furthermore, if \eqref{eq::check_practice} holds then \eqref{eq::check_practice} holds with either $x=\mu^\star$ or $x= \mu^\star + \frac{1}{w_{n+1}}\sum_{i=1}^{n} w_i (\mu^\star - X_i)$, which concludes the proof.

\subsection{Proof of Theorem~\ref{th::npts}}\label{app::proof_th_npts}

We already proved in Lemmas~\ref{lemm::cond_h} and~\ref{lemm::cond_h_cent} that (A1)--(A4) hold for families defined by an $h$-moment condition. Adapting these results for $D_\pi(F, \mu) = \Lambda_\gamma^\star(F, \mu)\ind(F\in \cF)$ would be direct (we proved them for $\Lambda_0^\star(F, \mu)$ with this notation), so we omit the proof for simplicity. The rest of the proof consists in proving that Assumption~\ref{ass::relaxed_sp} holds.

\paragraph{Upper bound} Lemma~\ref{lem::bounds_bcp} states that there exists a mapping $C$ such that for any $\eta>0$ and any $n\in \N$
	\begin{equation*}\text{[BCP]}\leq C(F_n, \eta) \times e^{-n \Lambda_{\eta, \gamma}^\star(F_n, \mu^\star)}\;,\end{equation*} 
where $C$ is continuous w.r.t. the Wasserstein metric in $F_n$, is continuous in $\eta$, and scales in $\eta^{-2}$. For any cdf $F \in \{F_2, \dots, F_K\}$ (sufficient to prove the result) and $\epsilon >0$, we define a neighborhood of $F$ as $\cB = \{F_n \in \cF: \;, W(F_n, F) \leq \epsilon \}$. If $F_n$ denotes the empirical distribution of the observations $X_1,\dots, X_n$ drawn i.i.d. from $F$, using Theorem 2 in \cite{fournier_wasserstein} and that $\bE_F[|X|^{2+\xi}]<+\infty$ for some $\xi>0$ (by definition of $\cF$) we obtain that 
\[\sum_{n=1}^{+\infty} \bP(W(F_n, F) \geq \epsilon) < +\infty \;. \]

Using this result, by continuity we can first choose $\epsilon$ small enough so that for instance $C(F_n, \eta) \leq 2 C(F, \eta)$ for any $\eta$ small enough. Then, we choose $\eta^{-1}=\cO(n^a)$ for some $a>0$, so that (1) $C(F_n, \eta)$ is a polynomial in $n$, and (2) $e^{-n \Lambda_{\eta, \gamma}^\star(F_n, \mu^\star)} \leq 2e^{-n \Lambda_{\gamma}^\star(F_n, \mu^\star)}$. With these choices of parameters, we satisfy all the conditions of Assumption~\ref{ass::relaxed_sp} regarding the upper bound on $p_k^\pi(t)$.

\paragraph{Lower bound} The formulation of the lower bound in Lemma~\ref{lem::bounds_bcp} already fits Assumption~\ref{ass::relaxed_sp} regarding the lower bound part. 
Indeed, the lower bound in Lemma~\ref{lem::bounds_bcp} becomes that of Assumption~\ref{ass::relaxed_sp} if we set 
$(n_{\eta}, c_n,p_0)$ in Assumption~\ref{ass::relaxed_sp} to
$(n_{\delta,\gamma}, c_{\delta}^{-1}, e^{-n_{\delta, \gamma} \log(n_{\delta, \gamma} C_{B,\mu})})$.
This concludes the proof.

\bibliography{biblio}

\end{document}